\documentclass{article} 

\usepackage[nonatbib, final]{style_files/neurips_2025} 

\usepackage[
    style=numeric,
    sorting=none,       
]{biblatex}
\title{Bib}
\newtoggle{nips}
\toggletrue{nips}

\let\regularcite\cite 
\renewcommand\cite[1]{
    \iftoggle{nips}{\parencite{#1}}{\regularcite{#1}} 
}

\newcommand\citeA[1]{
    \iftoggle{nips}{\textcite{#1}}{\cite{#1}}
}

\makeatletter
 
\usepackage{imports}

\counterwithout{theorem}{section}

\DTMlangsetup[en-GB]{ord=raise,monthyearsep={,\space}}
\DTMsavedate{date}{2021-07-21}

\begin{document}

\title{Online Mixture of Experts: No-Regret Learning for Optimal Collective Decision-Making}
\date{} 

\author{
    Larkin Liu\thanks{Corresponding author.} \\
    Technische Universität München \\
    Riebaki AI \\
    \texttt{larkin.liu@tum.de} \\
    \And
    Jalal Etesami \\
    Technische Universität München \\
    \texttt{j.etesami@tum.de}
}

\maketitle
\vspace{-0.8cm}

\begin{abstract}

We explore the use of expert-guided bandit learning, which we refer to as online mixture-of-experts (OMoE). In this setting, given a context, a candidate committee of experts must determine how to aggregate their outputs to achieve optimal results in terms of aggregate accuracy. We propose two algorithms to address this problem. The first algorithm combines aggregate voting with UCB-driven successive elimination, efficiently pruning suboptimal exploration actions. The second algorithm employs an online weighted-majority-voting mechanism, leveraging the respective voting power of each expert proportional to their predictive power. We derive theoretical guarantees for the regret properties in the bandit setting under ideal circumstances, and empirical results are provided accordingly. As a modern study on applications, these methods are applied to the online fine-tuning of a set of expert large language models (LLMs), where after each response, the generative LLM dynamically reweighs its set of experts and/or selects the optimal committee of experts to generate the most accurate response. Our results introduce new methodologies and no-regret guarantees for combining multiple experts to improve on the performance of the an aggregate model overall.

\end{abstract}


\iftoggle{nips}{\vspace{-1.5em}}

\section{Introduction}

\iftoggle{nips}{\vspace{-0.5em}}

The mixture-of-experts (MoE) model is a powerful concept in applied machine learning and social choice, leveraging the collective decision-making capabilities of a group of experts to yield improved predictions or more computationally efficient models. The core idea is that within a collection of $N$ experts, $\setE$, there exists at possible a subgroup, $\setE^* \subseteq \setE$, whose combined predictions are well-optimized on some well-defined metric (e.g. collective accuracy, model efficiency etc.). Classically, MoE models involved the design of a gating mechanism (a.k.a router) that aims to optimally transform the output of many experts into a single more optimal output. Traditionally, this involved training an offline model from labelled data to optimize some loss function via supervised learning  \cite{richardson:2003learning, nowlan:1990_moe_evaluation, si:2023-MoRE}. Concerning neural architecture design for offline learning, the MoE framework has been wide integrated to several state-of-the-art LLMs (e.g. DeepSeek-MoE, Mixtral) for token prediction tasks \cite{anthony:2024_blackmamba, shazeer:2017_sparse_moe_llm, dai:2024_deepseekmoe, jiang:2024mixtral, fedus:2022_sparse_expert_llm_review}. More recently, pertinent to applications on AI alignment, the online version of MoE has garnered significant interest. This is particularly useful when agents interact with an environment in a repeated manner, receiving feedback in an online fashion, and needing to optimize in real-time.

\iftoggle{nips}{\vspace{-0.5em}}

\paragraph{Online Mixture of Experts:}  Historically, online MoE was framed as a \textit{prediction with expert advice} problem, which has been extensively studied in the  learning literature \cite{cesa:2006_pred_learn_games, gao:2022_parameter_moe, huang:2024_mh_moe, cai2024:survey_moe}. The core principle involves aggregating predictions from multiple experts, where each expert is assigned a weight proportional to its historical predictive performance. Previous online MoE learning algorithms, such as \tT{EXP4}, are theoretically guaranteed to converge to the best single expert in hindsight \cite{littlestone:1994_weighted_maj_problem, lattimore:2020_bandit_book, cesa:2006_pred_learn_games}. While this property ensures asymptotic optimality relative to the best individual expert, it introduces a concerning limitation. Specifically, such algorithms fail to account for scenarios where an aggregate outcome of the experts' predictions—such as a majority voting, plurality consensus, or any form of preference aggregation—could yield superior performance compared to any single expert. In social choice, this phenomenon is often referred to as the \textit{wisdom-of-the-crowd}. If there exists potential improvement from collective decision-making, traditional online MoE methods, such as \tT{EXP4}, would not capture this. This gap highlights a key question, \textit{could we design an algorithm that provides convergence to the best collective decision-making committee rather than to the best single expert?} In our setting, we aim to learn both the competencies of the experts, and a sufficiently optimal strategy to combine the experts aggregate outcome in a real-time online environment. Evidently this also has important ramifications to generative LLMs, allowing for adaptively improved responses to tasks across various domains, when multiple LLMs are available.

\textbf{Key Contributions:} We make the following contributions in our work,

\begin{tcolorbox}[
colback=BIOblue,
		colframe=black,
		arc=4pt,
		boxsep=0.3pt,
	]%
    \textbf{Online Expert Aggregation:} 
    We introduce a family of methods that dynamically aggregate expert outputs in real-time through a combination of majority voting and online learning. This provably optimizes the collective aggregate competency of the expert set.
\end{tcolorbox}
\begin{tcolorbox}[
colback=BIOorange,
		colframe=black,
		arc=4pt,
		boxsep=0.3pt,
	]
    \textbf{LLM Application:} We empirically validate the effectiveness of our method across realistic problem domains via online majority voting mechanisms to aggregate responses from multiple LLMs serving as experts, yielding improved answer quality.
\end{tcolorbox}
\begin{tcolorbox}[
colback=BIOgreen,
		colframe=black,
		arc=4pt,
		boxsep=0.3pt,
	]
	\textbf{No-Regret Guarantees:} We provide theoretical guarantees with respect to the sublinear finite-time regret of the learning algorithm relative to the optimal solution under perfect information, bridging online learning methods to social choice theory.
     \looseness=-1
\end{tcolorbox}

\iftoggle{nips}{\vspace{-0.9em}}

\section{Problem Setting} \label{sec:problem_setting}

\begin{figure}
    \centering
    \includegraphics[width=0.98\linewidth]{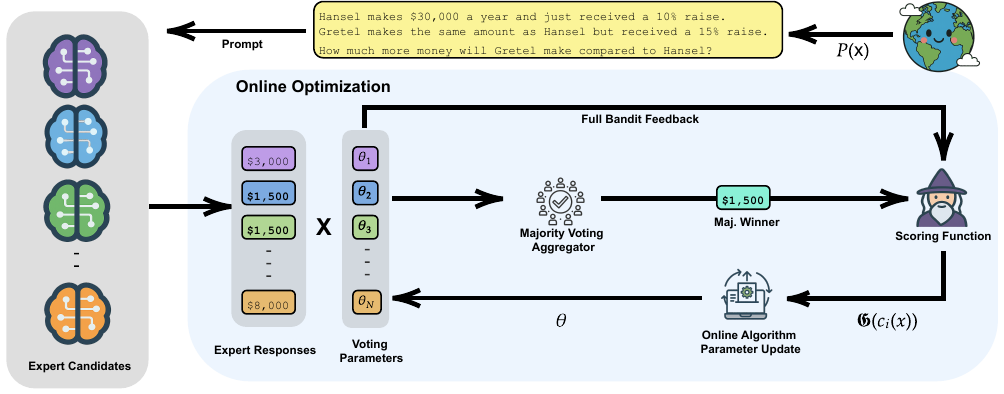}
    \caption[Online Mixture of Experts Workflow]{\textbf{Online Mixture of Experts via Majority Voting.} At each timestep, a context $\contexT$ is sampled from a distribution $P(\contexT)$. A prompt derived from $\contexT$ is used to query a set of expert candidates $\setExpertCand$. The responses are aggregated via a majority voting mechanism parameterized by $\theta$, yielding a single output. All expert responses and the aggregated output are evaluated by a scoring function (e.g., an oracle or human). The voting parameters $\theta$ are updated online to maximize expected score, using full-bandit feedback over a finite time horizon.  }
    \label{fig:enter-label}
    \vspace{-1.5em}
\end{figure}


\iftoggle{nips}{\vspace{-0.5em}}

To formally describe the problem setting, consider a learning model where the output is generated by combining the results of multiple weak learners through an aggregation function. These weak learners, referred to as \textit{experts}, produce individual outputs, and the aggregation function determines how these outputs are combined. Let $\setExpert$ denote the set of expert with the size $N$. Provided a context $\contexT\in\mathcal{X}$, expert $i\in \mathcal{E}$ makes a prediction $\bm{c}_i(\contexT)\in \mathcal{C}_i$ which represents an abstract categorical ordering or a singleton. Since $\{\mathcal{C}_i\}_{i\in\mathcal{E}}$ might be different sets (e.g., $\setCateg_i = \{\tT{cat}, \tT{dog}, \tT{mouse}\}$ and $\setCateg_j = \{\tT{mouse}, \tT{eagle}, \tT{fly}\}$), we introduce a \emph{standardizer} $\standardizeFunc(\cdot)$ that maps an arbitrary collection of prediction spaces $\{\setCateg_i\}_{i \in \setExpert}$ to a unified output space $\mathcal{C}$. Formally, the standardizer satisfies,
\[
\bigcap_{i \in \setExpert} \setCateg_i \subseteq \mathcal{C}:=\standardizeFunc(\{\setCateg_i\}_{i \in \setExpert}) \subseteq \bigcup_{i \in \setExpert} \setCateg_i,
\]
We refer to the size of the standardized output as $\dimStandard := |\standardizeFunc(\{\setCateg_i\}_{i \in \setExpert})|$. 

\textbf{Aggregation function:} It is  $\aggFun_\theta:\bigotimes_{i\in\mathcal{E}}\mathcal{C}_i\mapsto \mathcal{C}$ and characterized by a set of parameters $\theta=\{\theta_i\}_{i\in\mathcal{E}}$. It takes as input the experts' predictions $\{\bm{c}_i(\contexT)\}_{i\in\mathcal{E}}$ and produces a new prediction $\catAgg \in \mathcal{C}$. 

\textbf{Scoring function:} The \textit{goodness} of the aggregation function's output is valued by the scoring function, $\scrFun :\mathcal{C} \mapsto \mathbb{R}$ which assigns a real-value score in $\mathbb{R}$ to it's prediction. Suppose a series of contexts, $\contexT$, are randomly drawn from a distribution $P$. This contextual information is passed into each expert, and the predictions are subsequently aggregated parametrically. As indicated in Eq. \eqref{eq:context_label_scoring}, the prediction itself may be context-dependent. The scoring function evaluates the accuracy or quality of the prediction. The performance of a given aggregation function $\aggFun_\theta$ is measured by the average scoring function,
\begin{align}
    \mathbb{E}_{\contexT\sim P}[\scrFun \circ \aggFun_\theta] = \int_{\contexT} \scrFun \circ \aggFun_\theta\left(\{\bm{c}_i(\contexT)\}_{i\in\mathcal{E}}\right) \, \, dP(\contexT), \label{eq:context_label_scoring}
\end{align}
where $\scrFun \circ \aggFun_\theta$ denotes the composition of the scoring and aggregation functions. We aim is to maximize the average scoring function, formulated as,
$
    \theta^*\in\argmax_{\theta \in \bm{\theta}} \,\, \mathbb{E}_{\contexT\sim P}[\scrFun \circ \aggFun_\theta].
$ 
The properties of the utility function, such as commutativity or monotonicity, can depend on the structure of the aggregation outcomes. Our goal is to design an algorithm that learns the optimal parameters for the aggregation algorithm in an online setting.

\begin{remark} \label{rem:catalogue_expert}
    \textup{\textbf{Catalogue Standardizer:}} The standard dimension represents unified responses across experts (e.g., a fixed set of answers for LLMs). We distinguish two approaches: (1) \textbf{catalogue-based}, the model produces a catalogue of items where each expert selects from, and (2) \textbf{model-based}, where experts (e.g., transformers, MLPs) project arbitrary outputs to a standardized logit or token space $\mathbb{R}^D$. We applied the \textbf{catalogue} based standardizer for our experiments.
\end{remark}

\iftoggle{nips}{\vspace{-0.3em}}

The score of individual expert, expert's \textit{competency}, can be obtained from the scoring function by computing the expectation over all possible contexts distributed according to $P(\contexT)$ as follows,
\begin{align}
    p_i:= \int_{\contexT} \scrFun\left(\bm{c}_i(\contexT)\right) \, \, dP(\contexT).
\end{align}
In more straightforward terms, the competency of an expert refers to its individual ability to maximize the scoring-aggregation function over a singleton committee consisting of itself - i.e. the ability of the expert to individually predict the correct answer among candidate solutions. In voting theory, this is an important piece of information measuring potential effectiveness of aggregating experts. The average scoring function is the expectation over  the composition of the aggregation first $\aggFun(\cdot)$, followed by the scoring function $\scrFun(\cdot)$, according to distribution $P(\contexT)$ over all contexts, Eq. \eqref{eq:context_label_scoring}. Which we denote as as the scoring-aggregation function, or aggregation function for brevity.

\begin{assumption} \label{ass:eps-increm}
    \textbf{$\minPEps$-Incrementality (Heterogeneous Expert Competencies):} Every expert has a different $p_i$, offset by at least $\minPEps$ increment, that is $\inf_{i\neq j}\left( p_i - p_j \right) \geq \minPEps, \,\, \forall i, j \in 1 \dots N$. 
\end{assumption}

\vspace{-.2cm}
\paragraph{Extensions of Condorcet's Jury Theorem:} The celebrated Condorcet's Jury Theorem \cite{condorcet:1785_jury} states that with competent experts ($p_i > 0.5$), majority voting accuracy approaches 1 as $N \to \infty$. \citeA{list:2001_plurality_epistemic} extends this to plurality voting, showing asymptotic accuracy can hold even with $p_i < 0.5$ under strict reliability assumptions, i.e. that the probability of any expert selecting the correct class is greater than selecting any incorrect class. However, when these assumptions fail, incompetent experts can degrade aggregate accuracy, highlighting the need for robust expert aggregation in our setting. Our work applies epistemological social choice theory \cite{brandt:2016_handbook, list:2001_plurality_epistemic} to MoE models, focusing on maximizing predictive accuracy with a modest, finite set of experts ($\sim$20), where aggregate accuracy guarantees may or may not hold \cite{list:2001_plurality_epistemic, everaere:2010_belief_merging_truth_tracking, karge:2022_belief_fusion}. Unlike traditional social choice settings, we prioritize maximizing output accuracy (and accordingly minimizing regret) under imperfect information. Key challenges include dynamically inferring heterogeneous expert competencies and aggregating unreliable multi-class predictions under sample efficiency constraints.

\subsection{Egalitarian Majority Voting} \label{sec:egal_committee}

Consider a special setting in which experts' scored prediction are binary $\tT{Im}(\scrFun) =\{0,1\}$, where 1 indicates that the expert is deemed correct and 0 if incorrect. 
Suppose that the score-aggregation function is the majority voting with a random tie-breaking. 
In this setting, the composition of the scoring and aggregation function is expressed in Eq. \eqref{eq:score_agg_fun}. In words, it is 1 whenever majority of the experts give one and zero otherwise. 
\begin{small}
\begin{align}\label{eq:score_agg_fun}
    & \scrFun \circ \aggFun(\{\bm{c}_i(\contexT)\}_{i\in\mathcal{E}})=\mathbb{I}\left(\sum_{i\in\mathcal{E}} \scrFun (\bm{c}_i(\contexT))> \frac{N}{2}\right) + \mathbb{I}\left(\sum_{i\in\mathcal{E}} \scrFun (\bm{c}_i(\contexT))= \frac{N}{2}\right)\cdot Y,
\end{align}
\end{small}
where $\mathbb{I}(\cdot)$ is the indicator function and $Y\sim Be(1/2)$ is a Bernoulli random variable with parameter $1/2$. 
Let $\mathcal{X}_i:=\{\bm{x} = 1 | \bm{c}_i(\contexT) \}$ to be the set of contexts for which expert $i$ gives a correct prediction and $\mathcal{X}_i^c$ its complement. 
Note that the measure of $\mathcal{X}_i$ is $p_i$. Using this definition, the average scoring function can be rewritten as follows,
\begin{small}
\begin{align*}
    \mathbb{E}[\scrFun \circ \aggFun]= \proB \left(\bigcup_{\substack{S\subseteq\mathcal{E} \\ |S|> N/2  }}\bigcap_{i\in S} \mathcal{X}_i \bigcap_{j\in S^c} \mathcal{X}^c_j \right) + \frac{1}{2} \proB \left(\bigcup_{\substack{S\subseteq\mathcal{E} \\ |S|=N/2  }}\bigcap_{i\in S} \mathcal{X}_i \bigcap_{j\in S^c} \mathcal{X}^c_j \right).
\end{align*}
\end{small}

\begin{remark}
    This setting can be extended to non-binary prediction set, i.e., $\tT{Im}(\scrFun)=\{1,2,...,K\}$ for some $K$ and the output of the aggregation function is $i\in\tT{Im}(\scrFun)$ whenever a majority of experts vote for $i$. 
    In the event of a tie among the experts, the aggregation function outputs one of the tied items uniformly at random. That is the Bernoulli variable in \eqref{eq:score_agg_fun} will be replaced with a uniform random variable over the set of items that received the highest (tied) number of votes.       
\end{remark}

\textbf{Independent Experts:} Suppose that the experts are independent, i.e., $P(\bigcap_{i\in S} \mathcal{X}_i)=\prod_{i\in S}P(\mathcal{X}_i)$ for all $S\subseteq\mathcal{E}$. In this case, the average scoring function will be,
\begin{small}
\begin{align}\label{eq:equal_weights}
    \mathbb{E}[\scrFun \circ \aggFun]=\sum_{\substack{S\subseteq\mathcal{E} \\ |S|>  N/2  }}\prod_{i\in S}p_i\prod_{j\in S^c}(1-p_j) + \frac{1}{2}\sum_{\substack{S\subseteq\mathcal{E} \\ |S|=N/2  }}\prod_{i\in S}p_i\prod_{j\in S^c}(1-p_j).
\end{align}
\end{small}
Evaluating the above averaging function is combinatorially complex as its complexity grows exponentially with the number of experts $N$. Combinatorial multi-armed bandit (CMAB) have been applied for this purpose in \cite{merlis:2019_batch, chen:2013_combinatorial, wang:2017_improving}. Such methods assume the learning algorithm has access to an offline ($\alpha \beta$)-approximation oracle oracle where, with probability $\beta$, it outputs a solution whose value is at least $\alpha$ proportional of the maximum reward. Specifically, when the reward function is monotone and submodular, such an oracle can be given where $\alpha=1-1/e$. In our setting, the learner does not have access to such an ($\alpha \beta$)-oracle. For instance, adding experts an existing committee does not result in any generalizable behaviour in our majority voting setting. 
Throughout this paper, we assume that the experts are independent. However, as it is discussed in Section \ref{sec:online}, our proposed algorithms assemble their committees solely using the marginal experts' competencies and subsequently, the theoretical guarantees remain valid even if the experts are correlated. 

\textbf{Optimal Egalitarian Committee (OEC):} The optimal egalitarian committee (EC), also known as the binary weighted participation problem, refers to the formation of a committee in which each expert receives exactly one vote. Examples of such committees include corporate boards, academic panels, and juries. In machine learning, similar structures appear in ensemble methods such as Random Forests, boosting algorithms, deep learning ensembles, and Mixture-of-Experts (MoE) models.

While online learning algorithms like \tT{EXP4} typically converge to the best expert in hindsight, prior work has shown that it is often possible to achieve better performance by identifying an optimal egalitarian committee—that is, a subset of experts whose aggregated votes outperform the best expert. In this setting, the final prediction is determined by majority vote, and the resulting combined estimate can exceed the accuracy of any individual expert. We denote this committee as $\mathcal{E}^* \subseteq \setE$. 

When considering the multi-class classification problem, for a majority vote to pass, more than some minimum amount, referred to as a \textit{quota}, of the experts should have voted for the correct answer, each occurring with probability $p_i$.
For majority voting, we set the quota to be half of the number of experts (Optionally, the quota could also be reduced to allow for the experts to arrive at a clear option as a winner from a choice of options \cite{brandt:2016_handbook}). We denote the average scoring function of a candidate committee set $\setExpertCand\subseteq \mathcal{E}$ by $\pMaj(\setExpertCand)$ and is given by,
\begin{align}\label{eq:combo_maj_vote} 
    \pMaj(\setExpertCand) :=\sum_{\substack{S\subseteq\setExpertCand \\ |S|\geq |\setExpertCand|/2}}\prod_{i\in S}p_i\prod_{j\in S^c}(1-p_j)+\frac{1}{2}\sum_{\substack{S\subseteq\setExpertCand \\ |S|=|\setExpertCand|/2  }}\prod_{i\in S}p_i\prod_{j\in S^c}(1-p_j).  
\end{align}
Accordingly, the optimal committee is $\setExpert^*\in\arg\max_{\setExpertCand\subseteq \setExpert}\pMaj(\setExpertCand)$.

\textbf{Successive Expert Elimination Function:} To form th OEC, we can intuitively successively include each expert sorted by their respective $p_i$, to a candidate expert set $\setExpertCand$, given that such an inclusion would improve the overall accuracy of the committee. For any subset $\setExpert' \subseteq  \setExpertCand^c$, we define the advantage function as follows,
\begin{align}
    \advFunc(\setExpertCand, \setExpert') := \pMaj(\setExpertCand \cup \setExpert') - \pMaj(\setExpertCand), \quad \text{and} \quad \pMaj(\emptyset) := 0. \label{eq:succ_add_func_def}
\end{align}
Clearly, when the above function is negative, the addition of $\setExpert'$ to $\setExpertCand$ will not produce $\setExpert^*$, i.e.,
\begin{align*}
    \advFunc(\setExpertCand, \setExpert') \leq 0 \implies \setExpert^* \neq \setExpertCand \cup \setExpert'. \label{eq:expert_set_elimination}
\end{align*}

\begin{lemma} \label{lem:expert_ordinality}
    \textbf{Top-K Ordinality of Experts:} The optimal expert committee $\setExpert^*\subseteq \setExpert$ is a subset of Top-K experts based on their competencies, for $0 < K \leq N$. (Proof can be found in Appendix \ref{prf:expert_ordinality}.)
\end{lemma}

\textbf{Elimination of Experts:} To build the OEC under the perfect information, i.e., experts' competencies are known, we propose a constructive algorithm, where we start with $\setExpertCand=\setExpert$, and successively eliminate admissible subsets $\setExpert'$, until there is no further possible improvement via the $\advFunc(\cdot)$. 
Given the result of Lemma \ref{lem:expert_ordinality}, starting with all experts, $\setExpert$, the set of admissible experts must fall within the set of Top-$K$ experts for some unknown $0\!<\!K\!\leq\! N$, i.e. $\setExpert' \subseteq \tT{TopK}(\setExpert)$. 
The benefit of Lemma \ref{lem:expert_ordinality} allows us to efficiently compute the OEC, avoiding combinatorial search over all possible subsets. Under perfect information, this allows us to construct the OEC via a greedy algorithm. (See Algorithm \ref{alg:greedy_oec} in Appendix.)

\subsection{Weighted Majority Voting} \label{sec:weighted_maj_voting}



We extend from the egalitarian committee in Section \ref{sec:egal_committee} to a scenario where voters in a committee have unequal voting power. Intuitively, more accurate experts should have stronger voting power to sway the accuracy of the majority voting system. In a \textit{weighted majority voting} (WMV) system, the voting mechanism is a decision-making system where each expert has a specific level of influence, determined by an assigned weight. Previously \cite{littlestone:1994_weighted_maj_problem} investigated the problem providing only an optimal error guarantee limited to binary predictions, in the offline setting where voters can only produce one of two outputs $\{0, 1\}$. We aim to extend the setting to multiple class predictions with no-regret guarantees in the online setting.

We denote the weight of expert $i\in\mathcal{E}$ by $\theta_i\in \mathbb{R}$. 
These weights often reflect factors such as expertise, historical performance, or stake in the outcome. Each expert votes for a candidate class, and the total weight of votes for each candidate is calculated. The option with the weight that surpasses a quota $Q$ is chosen as the final decision. This approach extends simple majority voting by allowing experts to have unequal influence, making it useful in contexts like corporate governance, expert panels, and machine learning ensembles, where experts vary in reliability or expertise. Let us define $ \pMaj(\setExpert, \theta)$ as the aggregate scoring function of a committee of voters with unequal weights denoted by $\theta=(\theta_i)_{i\in\mathcal{E}}$.


\noindent
\begin{minipage}{0.55\textwidth}
  \begin{equation}
    \pMaj(\setExpert, \theta)\!:=\!\sum_{S\subseteq\setExpert} \! \feasFunc[S, \theta]{N} \left( \prod_{i \in  S} p_i  \prod_{j\in S^c}  (1-p_j) \right), \label{eq:weighted_combo_maj_vote}
  \end{equation}
\end{minipage}%
\begin{minipage}{0.45\textwidth}
  \quad
  \begin{equation}
    \text{where} \quad \feasFunc[S, \theta]{N}\! :=\! \mathbb{I}\left(\sum_{j\in S}\theta_j>Q\right), \label{eq:quota_indicator_conditions}
  \end{equation}
\end{minipage}

where $\sum_{j\in\emptyset}\theta_j:=0$. In words the above selection function decides whether the configuration $S$ of the committee $\setExpert$ meets the feasibility requirements for the vote to pass. In the special case where all weights are one, then \eqref{eq:weighted_combo_maj_vote} reduces to \eqref{eq:equal_weights} with an appropriate tie-breaking. The weighted majority voting problem with quota $Q$ can be stated as,

\vspace{-1.9em}

\begin{align}\label{eq:wmp}
    \max_{\substack{\theta\in\mathbb{R}_+^{N}, \|\theta\|_1\leq 2Q}}\pMaj(\setExpert, \theta).
\end{align}

\iftoggle{nips}{\vspace{-0.5em}}

In our experiments, for simplicity, we used $Q = N/2$, ensuring majority consensus. This guarantees a clear winner when the majority of experts predict optimally. Should a majority decision be unreachable, a random class is sampled uniformly as the winner. To help further simplify the solution to Eq. \eqref{eq:wmp}, via Lemma \ref{lem:equality_Q}, we can solve for $\theta$ at equality as opposed to constrained cap on $Q$, constituting a more convenient expression encompassing the latter.

\begin{lemma} \label{lem:equality_Q}
    \textbf{Weights Satisfied at Equality:} When solving for the optimal weights for $\theta$-weighted majority voting, the constraint \textup{$Q \leq \|\theta\|_1\leq 2Q$}, from Eq. \eqref{eq:wmp} can be replaced with  \textup{$\|\theta\|_1 = 2Q$}. (Proof is available in Appendix \ref{prf:equality_Q}.)
\end{lemma}

\iftoggle{nips}{\vspace{-0.7em}}


\subsubsection{Solving the Optimal Weighted Voting Committee} 

A key objective is optimizing the parameters, $\theta$, to maximize majority voting accuracy in Eq. \eqref{eq:weighted_combo_maj_vote}. While convex combinations of expert outputs with non-zero weights may yield higher optima, the problem involves iterative combinatorial complexity. For binary outcomes, i.e. only two classes to vote for, optimal weights follow the log-odds ratio: $\theta^*_i = \log(p_i/1-p_i), \, \, \forall \theta_i \in \mathbb{R}$
\cite{berend:2014_binary_maj_voting,nitzan:1982_binary_voting} (see Appendix \ref{sec:binary_weighted_voting} for a discussion and proof). For multi-class settings, the specifications of the problem challenges the aforementioned existing theoretical guarantees. \citeA{everaere:2010_belief_merging_truth_tracking} and \citeA{abramowitz:2022_review_reviewers} provide further insight, yet optimality guarantees still require binary outcomes or large expert pools. \citeA{karge:2022_belief_fusion} provides non-asymptotic bounds via a competency gap, but remains limited for small-scale online settings. Therefore, our problem requires novel approaches balancing accuracy maximization with sample efficiency and expert heterogeneity. First, we leverage potential simplifications in optimal weighted majority voting problem structure.

\begin{proposition} \label{prop:discrete_img}
    \textbf{Discrete Image:} 
    Many configurations of $\theta$ can lead to the same result for $\pMaj(\setExpert, \theta)$. (Proof is provided in Appendix \ref{prf:discrete_img}.).
\end{proposition}

\iftoggle{nips}{\vspace{-0.5em}}

The results of Prop. \ref{prop:discrete_img} are rather intuitive, stating that many different configurations of weights could lead to the same expected accuracy via majority voting. One can prove this constructively by considering a scenario where two or more experts do not contribute to the optimal solution. So long as their weights do not contribute to swinging a majority vote for any possible outcome (i.e. this could be an infinitesimally small weight), it could be a valid optimal solution. Next, we provide an algorithm which for any given ordering over the competencies, denoted as $\tT{Ord}(p) := (p_1, p_2 , \dots p_N)$, such that $p_1 \geq p_2 \geq \dots p_N$, computes an optimal set of weights.

\iftoggle{nips}{\vspace{-0.5em}}

\paragraph{Mixed Integer Program Formulation:} The exact formulation of the optimal weights in the voting problem can be formulated as a mixed-integer program (MIP), governed by the properties outlined in this section. 
Let \( \boldsymbol{\theta}\! =\! [\theta_1, \theta_2, \dots, \theta_N] \) be a  positive vector, where \( \theta_i\! \geq\! 0 \) for all \( i \). Let \( \mathbf{x}\! =\! [\mathbf{x}_1, \mathbf{x}_2, \dots, \mathbf{x}_{2^N}] \) be a predefined vector where $\mathbf{x}_S\!\in\!\{0,1\}^N$ and represents configuration $S\subseteq \setExpert$, i.e., $[\mathbf{x}_S]_j\!=\!1$ if and only if expert $j$ is voted correctly in configuration $S$. 
We also define $\indicMip_S$
as a binary variable corresponding to $S$, determined by the following if-else condition, 
\begin{equation*}
\indicMip_S := 
\begin{cases} 
1 & \text{if } \mathbf{x}_S \cdot \boldsymbol{\theta} > Q,\\
0 & \text{otherwise},
\end{cases}
\end{equation*}
where $Q \in \mathbb{R}^+$ is the given quota. 
The problem in \eqref{eq:wmp} can be formulated as an MIP with the following objective function,
\begin{equation}
\max_{\theta \in \mathbb{R}^{N}_+} \,\, \sum_{S \subseteq \setExpert} \indicMip_S \prod_{i \in  S} p_i  \prod_{j\in S^c}  (1-p_j). \label{eq:mip_objective_func}
\end{equation}

This MIP computes the expected value over all voting outcomes using a binary selection variable and possibility variable $\indicMip_S$. When maximizing for the MIP objective, we must formulate a set of valid constraints. 
Leveraging Lemma \ref{lem:equality_Q}, we tighten the quota condition, introducing constraint \(\sum_j \theta_j = 2Q\). Assumption \ref{ass:eps-increm} enforces $\theta_i \neq \theta_j$ as a constraint. 
Importantly, we enforce logical consistency by ensuring that complementary voting outcomes cannot simultaneously satisfy both an outcome and its complement
(e.g. if scenario $S = [0,1,1,0]$ is possible and $\indicMip_S = 1$, then its complement $S^c =[1,0,0,1]$ is not, implying  $\indicMip_{S^c} = 0$.). These constraints collectively validate voting configurations, enforce expert competency distinctions, and eliminate contradictory outcomes, while maximizing for a global solution. We provide the details for this MIP formulation in Appendix \ref{sec:mip_solution_appendix}.

\begin{lemma} \label{lem:maj_vote_dominance}
    \textbf{Majority Voting Dominance over Egalitarian Voting:} For any set of experts, the optimal majority weighted voting method will always yield a stronger or equal predictive accuracy than the plurality vote of the committee. (Proof can be found in Appendix \ref{prf:maj_vote_dominance}.) 
\end{lemma}
Lemma \ref{lem:maj_vote_dominance} posits that given the same committee of experts, utilizing the $\theta$-WMV algorithm (Alg. \ref{alg:theta_wmv}), to aggregate experts’ output will always yield a stronger result than using egalitarian voting. Therefore, if the computational resources are available the $\theta$-WMV algorithm will produce a better accuracy. However, this depends on the systems ability to serve all experts and provide a solution to the MIP expressed in Eq. \eqref{eq:mip_objective_func}. If the goal is to learn an efficient representation of experts (i.e. a sparse set of experts) then \tT{SEE} (in Alg. \ref{alg:succEx}) can be a better choice to retain efficiency.


\section{Online Learning} \label{sec:online}

\iftoggle{nips}{\vspace{-0.3em}}

In previous work, the online MoE problem was typically framed as a no-regret problem relative to the best single expert in hindsight (Hannan consistency). Recent work on online MoE learning focuses on identifying top-$K$ experts based on their competencies. A prominent gap remains in efficiently learning optimal aggregation with heterogeneous expert competencies \textit{without} combinatorial explosion. Prior work assumes either pairwise feedback (duelling) or submodularity, yet there remains a focal body of work addressing optimal MoE aggregation in an online learning setting. Key previous approaches include duelling bandits, where pairwise comparisons identify dominant experts \cite{yue:2012_duelling_bandit,dudik:2015contextual_duel}, and submodular bandits, which treat committee formation as a matroid optimization problem with greedy guarantees \cite{chen:2017_submod_bandit,hazan:2012_online_sub}. Alternatively, satisficing methods \cite{niazadeh:2021_online_greedy,feng:2025_satisficing,russo:2022_satisficing} trade optimality for efficiency, while Top-K bandits \cite{evendar:2006_succ_elim,bubeck:2013_CSAR,rejwan:2020_topk_combo_bandit} leverage competency ordering for successive acceptance and elimination of arms. Despite the ample amount of research in the online Top-K arm-selection area, a prominent research gap yet remains which bridges voting theory to online no-regret guarantees.

\iftoggle{nips}{\vspace{-0.6em}}

\paragraph{Notion of Regret:} As our proposed algorithm is a no-regret learner, we formally introduce the notion of regret in Def. \ref{def:regret}. Cumulative regret, $R_T$, is defined as the difference between the expected majority voting aggregate accuracy, $\pMaj(\setExpert, \theta^*)$, under parameters, $\theta^*$, subtracted by the expected aggregate accuracy under the current parameters $\pMaj(\setExpert, \theta^t)$.

\begin{definition} \label{def:regret} \textbf{Regret:}
    The cumulative regret of an online algorithm over a time horizon of length $T$ that selects a set of weights $\theta^t$ at round $t$ is given by 
$
    R_T := \sum_{t=1}^T \pMaj(\setExpert, \theta^*) - \pMaj(\setExpert, \theta^t).
$
\end{definition}

\iftoggle{nips}{\vspace{-0.9em}}

\subsection{Successive Expert Elimination (\tT{SEE})}

\iftoggle{nips}{\vspace{-0.6em}}

The idea of successive elimination in online learning was demonstrated in \citeA{evendar:2006_succ_elim} where the successive elimination bandit algorithm with a PAC bound on sample complexity was proposed, along with both a gap dependent and gap independent versions based on the UCB. \citeA{rejwan:2020_topk_combo_bandit} presents a successive elimination algorithm in the full-bandit-feedback setting (a minor detail is that only the sum of rewards is available, which is an additive result from the individual rewards). Yet, it does not consider the optimal Top-K expert selection (only that K is given).To distinguish, in our situation, the aggregate reward, $\pMaj$ function is not necessarily additive, and the Top-$K$ is not given beforehand and must be determined. We first propose an algorithm based on the UCB, where there is some confidence range between our estimate of $\hat{p}_i$ and the true value for each expert $p^*_i$. This confidence bound should shrink over time with increasing samples.


\iftoggle{nips}{\vspace{-0.6em}}

\tikzset{every picture/.style={line width=0.75pt}} 

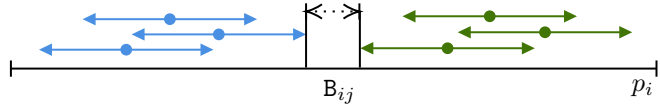
\begin{figure}[ht]
\centering
\begin{tikzpicture}[x=0.55pt,y=0.55pt,yscale=-1,xscale=1]


\draw [color={rgb, 255:red, 74; green, 144; blue, 226 }  ,draw opacity=1 ]   (69.33,40.33) -- (183.33,40.33) ;
\draw [shift={(186.33,40.33)}, rotate = 180] [fill={rgb, 255:red, 74; green, 144; blue, 226 }  ,fill opacity=1 ][line width=0.08]  [draw opacity=0] (8.93,-4.29) -- (0,0) -- (8.93,4.29) -- cycle    ;
\draw [shift={(126.33,40.33)}, rotate = 0] [color={rgb, 255:red, 74; green, 144; blue, 226 }  ,draw opacity=1 ][fill={rgb, 255:red, 74; green, 144; blue, 226 }  ,fill opacity=1 ][line width=0.75]      (0, 0) circle [x radius= 3.35, y radius= 3.35]   ;
\draw [shift={(66.33,40.33)}, rotate = 0] [fill={rgb, 255:red, 74; green, 144; blue, 226 }  ,fill opacity=1 ][line width=0.08]  [draw opacity=0] (8.93,-4.29) -- (0,0) -- (8.93,4.29) -- cycle    ;
\draw [color={rgb, 255:red, 74; green, 144; blue, 226 }  ,draw opacity=1 ]   (103,50.8) -- (217,50.8) ;
\draw [shift={(220,50.8)}, rotate = 180] [fill={rgb, 255:red, 74; green, 144; blue, 226 }  ,fill opacity=1 ][line width=0.08]  [draw opacity=0] (8.93,-4.29) -- (0,0) -- (8.93,4.29) -- cycle    ;
\draw [shift={(160,50.8)}, rotate = 0] [color={rgb, 255:red, 74; green, 144; blue, 226 }  ,draw opacity=1 ][fill={rgb, 255:red, 74; green, 144; blue, 226 }  ,fill opacity=1 ][line width=0.75]      (0, 0) circle [x radius= 3.35, y radius= 3.35]   ;
\draw [shift={(100,50.8)}, rotate = 0] [fill={rgb, 255:red, 74; green, 144; blue, 226 }  ,fill opacity=1 ][line width=0.08]  [draw opacity=0] (8.93,-4.29) -- (0,0) -- (8.93,4.29) -- cycle    ;
\draw [color={rgb, 255:red, 74; green, 144; blue, 226 }  ,draw opacity=1 ]   (39.33,61.33) -- (153.33,61.33) ;
\draw [shift={(156.33,61.33)}, rotate = 180] [fill={rgb, 255:red, 74; green, 144; blue, 226 }  ,fill opacity=1 ][line width=0.08]  [draw opacity=0] (8.93,-4.29) -- (0,0) -- (8.93,4.29) -- cycle    ;
\draw [shift={(96.33,61.33)}, rotate = 0] [color={rgb, 255:red, 74; green, 144; blue, 226 }  ,draw opacity=1 ][fill={rgb, 255:red, 74; green, 144; blue, 226 }  ,fill opacity=1 ][line width=0.75]      (0, 0) circle [x radius= 3.35, y radius= 3.35]   ;
\draw [shift={(36.33,61.33)}, rotate = 0] [fill={rgb, 255:red, 74; green, 144; blue, 226 }  ,fill opacity=1 ][line width=0.08]  [draw opacity=0] (8.93,-4.29) -- (0,0) -- (8.93,4.29) -- cycle    ;
\draw [color={rgb, 255:red, 65; green, 117; blue, 5 }  ,draw opacity=1 ]   (289.33,38.33) -- (403.33,38.33) ;
\draw [shift={(406.33,38.33)}, rotate = 180] [fill={rgb, 255:red, 65; green, 117; blue, 5 }  ,fill opacity=1 ][line width=0.08]  [draw opacity=0] (8.93,-4.29) -- (0,0) -- (8.93,4.29) -- cycle    ;
\draw [shift={(346.33,38.33)}, rotate = 0] [color={rgb, 255:red, 65; green, 117; blue, 5 }  ,draw opacity=1 ][fill={rgb, 255:red, 65; green, 117; blue, 5 }  ,fill opacity=1 ][line width=0.75]      (0, 0) circle [x radius= 3.35, y radius= 3.35]   ;
\draw [shift={(286.33,38.33)}, rotate = 0] [fill={rgb, 255:red, 65; green, 117; blue, 5 }  ,fill opacity=1 ][line width=0.08]  [draw opacity=0] (8.93,-4.29) -- (0,0) -- (8.93,4.29) -- cycle    ;
\draw [color={rgb, 255:red, 65; green, 117; blue, 5 }  ,draw opacity=1 ]   (327.33,49.33) -- (441.33,49.33) ;
\draw [shift={(444.33,49.33)}, rotate = 180] [fill={rgb, 255:red, 65; green, 117; blue, 5 }  ,fill opacity=1 ][line width=0.08]  [draw opacity=0] (8.93,-4.29) -- (0,0) -- (8.93,4.29) -- cycle    ;
\draw [shift={(384.33,49.33)}, rotate = 0] [color={rgb, 255:red, 65; green, 117; blue, 5 }  ,draw opacity=1 ][fill={rgb, 255:red, 65; green, 117; blue, 5 }  ,fill opacity=1 ][line width=0.75]      (0, 0) circle [x radius= 3.35, y radius= 3.35]   ;
\draw [shift={(324.33,49.33)}, rotate = 0] [fill={rgb, 255:red, 65; green, 117; blue, 5 }  ,fill opacity=1 ][line width=0.08]  [draw opacity=0] (8.93,-4.29) -- (0,0) -- (8.93,4.29) -- cycle    ;
\draw [color={rgb, 255:red, 65; green, 117; blue, 5 }  ,draw opacity=1 ]   (260.33,60.33) -- (374.33,60.33) ;
\draw [shift={(377.33,60.33)}, rotate = 180] [fill={rgb, 255:red, 65; green, 117; blue, 5 }  ,fill opacity=1 ][line width=0.08]  [draw opacity=0] (8.93,-4.29) -- (0,0) -- (8.93,4.29) -- cycle    ;
\draw [shift={(317.33,60.33)}, rotate = 0] [color={rgb, 255:red, 65; green, 117; blue, 5 }  ,draw opacity=1 ][fill={rgb, 255:red, 65; green, 117; blue, 5 }  ,fill opacity=1 ][line width=0.75]      (0, 0) circle [x radius= 3.35, y radius= 3.35]   ;
\draw [shift={(257.33,60.33)}, rotate = 0] [fill={rgb, 255:red, 65; green, 117; blue, 5 }  ,fill opacity=1 ][line width=0.08]  [draw opacity=0] (8.93,-4.29) -- (0,0) -- (8.93,4.29) -- cycle    ;
\draw [color={rgb, 255:red, 0; green, 0; blue, 0 }  ,draw opacity=1 ]   (220,35) -- (220,73.6) ;
\draw [color={rgb, 255:red, 0; green, 0; blue, 0 }  ,draw opacity=1 ]   (257,35.03) -- (257,73.6) ;
\draw    (17,74.33) -- (461,74.33) ;
\draw [shift={(461,74.33)}, rotate = 180] [color={rgb, 255:red, 0; green, 0; blue, 0 }  ][line width=0.75]    (0,5.59) -- (0,-5.59)   ;
\draw [shift={(17,74.33)}, rotate = 180] [color={rgb, 255:red, 0; green, 0; blue, 0 }  ][line width=0.75]    (0,5.59) -- (0,-5.59)   ;
\draw  [dash pattern={on 0.84pt off 2.51pt}]  (220,35) -- (256.67,35.03) ;
\draw [shift={(256.67,35.03)}, rotate = 180.05] [color={rgb, 255:red, 0; green, 0; blue, 0 }  ][line width=0.75]    (0,5.59) -- (0,-5.59)(10.93,-4.9) .. controls (6.95,-2.3) and (3.31,-0.67) .. (0,0) .. controls (3.31,0.67) and (6.95,2.3) .. (10.93,4.9)   ;
\draw [shift={(220,35)}, rotate = 0.05] [color={rgb, 255:red, 0; green, 0; blue, 0 }  ][line width=0.75]    (0,5.59) -- (0,-5.59)(10.93,-4.9) .. controls (6.95,-2.3) and (3.31,-0.67) .. (0,0) .. controls (3.31,0.67) and (6.95,2.3) .. (10.93,4.9)   ;

\draw (442,80) node [anchor=north west][inner sep=0.75pt]   [align=left] {$p_i$};
\draw (230,80) node [anchor=north west][inner sep=0.75pt]   [align=left] {$\breaK$};

\end{tikzpicture}
\vspace{-0.5em}

\caption{Illustration of a breakage event $\breaK$.} \label{fig:breakage} 
\end{figure}


\iftoggle{nips}{\vspace{-0.5em}}

\begin{definition}
    \textbf{Breakage Event:} A breakage event at round $t$ denoted as \textup{$\breaK^t$} between experts $i$ and $j$ denotes an event where the upper confidence limit of the empirical estimation of $j$'s competency at round $t$, $\hat{p}_j^t$,  
    is less than the lower confidence limit of $\hat{p}_i^t$, i.e., 
    \[
        \textup{$\breaK^t$}:   \hat{p}_j^t + \ucB_j^t < \hat{p}_i^t - \ucB_i^t.
    \]
\end{definition}

\paragraph{Intuition of Successive Expert Elimination (\tT{SEE}):} The key idea of \tT{SEE} is that we can use the confidence bounds to successively eliminate experts from a candidate committee until we arrive at a final set of optimal candidates, with high confidence. The algorithm begins with the complete set of candidate committee equal to the set of all experts $\setExpertCand \gets \setExpert$ and all experts are played at each round until a breakage event occurs $\breaK^t$. 
At the event of a breakage, we first order the experts whose estimated competencies are lower than expert $j$, say $j_1,j_2,...,j_k:=j$, such that $\hat{p}^t_{j_1} < \hat{p}^t_{j_2}\leq ...\leq \hat{p}^t_{j}$ (pictured as blue points in Diagram \ref{fig:breakage}). Then, we define $\tT{TopK}(\setExpertCand^{\leq j})$ as the set $\big\{\{j_k\}, \{j_k,j_{k-1}\},...,\{j_1,...,j_k\}\big\}$. 
To perform a removal test, for each element $\setExpert'$ in $\tT{TopK}(\setExpertCand^{\leq j})$ 
, we consult the advantage function to see if $\advFunc(\setExpertCand, \setExpert') \leq 0$; that is no advantage can be obtained by including $\setExpert'$ with high confidence. If so, all experts in $\setExpertCand^{\leq j}:=\{j_1,...,j_k\}$ could be subsequently eliminated from $\setExpertCand$ (i.e. removal of blue points from the previous union of blue and green w.r.t Diagram \ref{fig:breakage}).

\begin{algorithm}
\caption{Successive Expert Elimination (\tT{SEE})} \label{alg:succEx}
\begin{algorithmic}[1]
\State $\setExpertCand \gets \setExpert$ \Comment{Set the candidate expert set to the full set of experts.}
\For{$t = 1$ to $T$}
    \State Estimate $\hat{p}^t_i$ and upper confidence bound, $\ucB_i^t$, for each expert in $\setExpertCand$.   
    \If{any $\breaK^t$ event is observed} \Comment{Breakage event occurs.}
        \For{every $\breaK^t$, in descending order}
            
            
            
            \State $\setExpertCand^+ \gets \setExpertCand^{> j}$
            \If{ $\advFunc(\setExpertCand^+, \setExpert') \leq 0, \quad \forall \setExpert' \in \tT{TopK}(\setExpertCand^{\leq j})$} \Comment{Perform removal test.}
                \State $\setExpertCand \gets \setExpertCand \setminus \setExpertCand^{\leq j}.$ \Comment{Expert set truncation.}
            \EndIf
        \EndFor
    \EndIf
\EndFor

\State \Return $\setExpertCand$
\end{algorithmic}
\end{algorithm}

\begin{lemma}\label{lem:pac_sample_complexity_shrink}

\textbf{PAC Sample Complexity:} Suppose that the competency of each expert $i$ is modeled as a sub-Gaussian random variable with mean $p_i$ and variance $\sigma_i^2\leq \sigma^2$.  Under Assumption \ref{ass:eps-increm}, for any two experts $i$ and $j$, if the number of samples $t$ and shrinkage term $\ucB_i^t$ satisfy 
\[
    \ucB_i^t := \sqrt{\frac{2 \sigma_i^2 \log(4/\delta)}{t}}
    , \quad t \geq  t_0:=\frac{32\sigma^2 \log(4/\delta)}{\minPEps^2},
\]
then, with probability at least $1 - \delta$, $\delta \in (0, 1)$, the breakage condition, \textup{$\breaK^t$}, will be met. (Proof provided in Appendix \ref{prf:pac_sample_complexity_shrink})
\end{lemma}

\begin{theorem}
\label{thm:regret_bound_succex_log}
\textbf{Regret Bound for \textup{\tT{SEE}}:} The total regret of \textup{\tT{SEE}} Algorithm in \ref{alg:succEx} is bounded by:
\begin{align}
    R_T \in \mathcal{O}\left( N \log(T)/\minPEps^2 \right).
\end{align}
\end{theorem}

\textit{Sketch of Proof:} The bound follows a PAC elimination sample complexity (Lemma~\ref{lem:pac_sample_complexity_shrink}), followed by a union bound over \( \mathcal{O}(N^2)\) pairwise comparisons. Next, we set the failure probability \(\delta = 1/T\) to control catastrophic regret. The dominant term arises from aggregating \(N-1\) elimination rounds, each contributing \( \mathcal{O}\left(\log(T)/\minPEps^2\right)\) regret. (For complete proof see Appendix \ref{prf:regret_bound_succex_log}.)


\subsection{Online Weighted Majority Voting Solution (Full Bandit Feedback)}

\iftoggle{nips}{\vspace{-0.5em}}

\textbf{Full Bandit Feedback:} As discussed in previous section, the optimal weighed majority voting problem ($\theta$-WMV) can be cast as an MIP under perfect information. Consequently, this can be integrated into UCB style bandit algorithms. To implement a no-regret learning algorithm, we solve for the best- and worst-case outcomes under uncertainty in the experts' competencies via the MIP, and quantify the corresponding shrinkage in the bounded simple regret.


\begin{algorithm}[H]
\caption{$\theta$-Weighted Majority Voting with Full Bandit Feedback ($\theta$-WMV) } \label{alg:theta_wmv}
\begin{algorithmic}[1]
\Require Set of experts $\setExpert$ with $n \in \{1, \dots, N\}$.
\State $\setExpertCand \gets \setExpert$ \Comment{Set the candidate expert set to the full set of experts.}
\For{$t = 1$ to $T$}
    \State Estimate $\hat{p}^t_i$ and upper confidence bound, $\ucB^t_i$ using the history for all experts in $\setExpertCand$.
    \State Solve for $\bar{\theta}^*$ with optimistic estimates, $\hat{p}^t + \ucB^t_i$, via MIP.
    \State Elicit all responses from $\setExpert$, and combine responses using weighted majority voting with $\bar{\theta}^*$.
    \State Receive feedback from scoring function, and add it to the history.
\EndFor

\State \Return $\bar{\theta}^*$
\end{algorithmic}
\end{algorithm}

\begin{theorem} \label{thm:regret_theta_star_mip}
    \textbf{Regret Bound for $\theta$-WMV:} Under the Assumption \ref{ass:bounding_condition}, Algorithm \ref{alg:theta_wmv} achieves a regret bound of $R_T \in \mathcal{O} \left( \sqrt{N T \log(T)} \right)$. (Proof in Appendix \ref{thm:regret_theta_star_mip_proof}.)
\end{theorem}

\textit{Sketch of Proof:} We show that the MIP solution to Eq. \eqref{eq:mip_objective_func} under $\bar{\theta}^*$ misspecification is bounded by the difference in outcome scenario probabilities, which is also bounded by the uncertainty estimates of the expert's competencies, $p_i \pm \ucB$ (see Lemma \ref{lem:mip_pi_diff_bound} in Appendix). Subsequently, using standard PAC concentration bound arguments (i.e. Lemma~\ref{lem:pac_sample_complexity_shrink}), we derive the expected cumulative regret. A looser alternative bound under relaxed assumptions is also provided in Appendix \ref{sec:alternative_mip_bound}.

\paragraph{Targeted-$m$ Online Majority Expert Voting:} In MoE applications, computational constraints often limit the number of experts that can be queried at each timestep to at most $m$. We address this extension to the problem by partitioning the $N$ experts into $N/m$ groups and running a windowed phase of $t_0$ rounds per group. This ensures high-probability breakage detection while preserving the original regret bounds, scaled by a multiplicative $N/m$ factor (Lemma \ref{lem:pac_sample_complexity_shrink}). Specifically, the targeted-$m$ variant of \tT{SEE} maintains $\mathcal{O}\left(N^2 \log T/ m \minPEps^2\right)$ regret, matching the original setting's guarantees with respect to $T$ despite the query limitation.

\iftoggle{nips}{\vspace{0.3em}}

\begin{corollary} \label{cor:targeted_m_setting}
\textbf{Targeted-$m$ WMV:} An algorithm exists that allows for the regret bounds of Theorems \ref{thm:regret_bound_succex_log} and \ref{thm:regret_theta_star_mip} to hold with an  multiplicative factor of $N/m$. (Proof can be found in Appendix \ref{prf:targeted_m_setting}.)
\end{corollary}

\section{Empirical Study}

\begin{figure*}[!tb]
\centering
\subfloat[Optimal Egalitarian Voting.]{%
  \includegraphics[width=43.5mm]{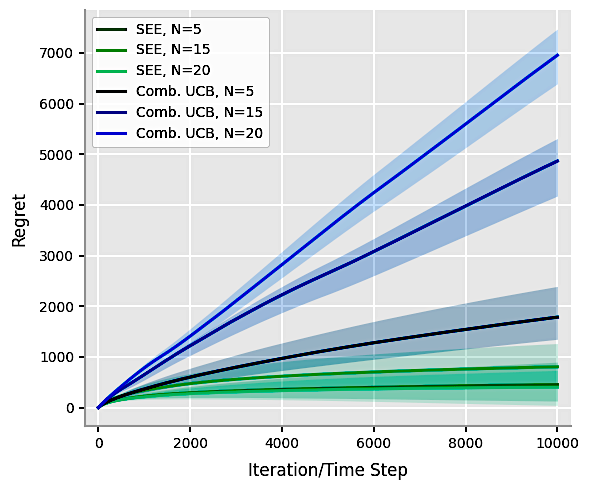}%
  \label{fig:see_bern}%
}\hspace{10pt}
\subfloat[Optimal Weighted Voting.]{%
  \includegraphics[width=43.5mm]{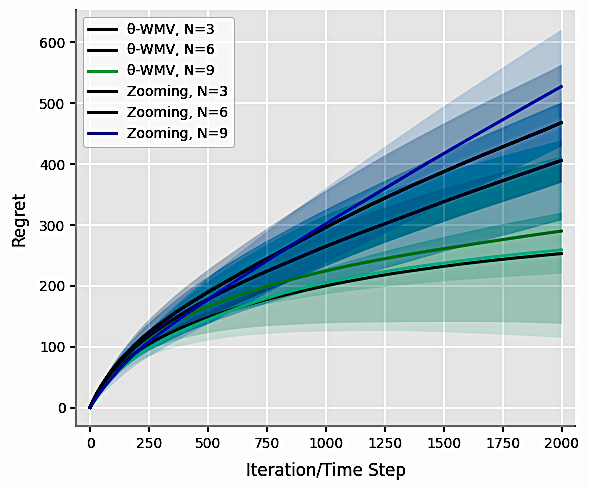}%
  \label{fig:theta_wmv_bern}%
}\hspace{10pt}
\subfloat[LLM Experts.]{%
  \includegraphics[width=43.5mm]{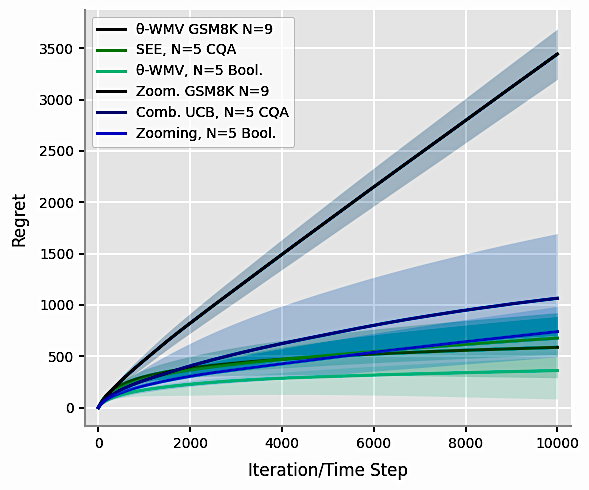}
  \label{fig:llm_experiments}
}
\caption{
\textbf{Comparison of Cumulative Regret for MoE Voting Environments.}
Fig.~\ref{fig:see_bern}: \tT{SEE} on simulated Bernoulli experts; 
Fig.~\ref{fig:theta_wmv_bern}: $\theta$-WMV on simulated Bernoulli experts; 
Fig.~\ref{fig:llm_experiments}: LLM experts across domains (e.g., GSM8k, BoolQ etc.), with tasks sampled from each domain. 
Shaded regions show $\pm \ucB$ uncertainty. Evidently, \tT{SEE} and $\theta$-WMV consistently outperform baselines in regret minimization.
} \label{fig:all_regret_plots_main}
\end{figure*}

\begin{table}[t]
    \centering
    \small
    \begin{tabularx}{\linewidth}{@{}
        >{\hsize=1.1\hsize}X  
        >{\hsize=0.3\hsize}X  
        >{\hsize=0.3\hsize}X  
        >{\hsize=0.4\hsize}X  
        >{\hsize=0.4\hsize}X  
        >{\hsize=0.4\hsize}X  
        >{\hsize=0.4\hsize}X  
        >{\hsize=0.4\hsize}X  
        >{\hsize=1.2\hsize}X@{}  
    }
        \toprule
        \textbf{Algorithm} & $N$ & $|\mathcal{E}^*|$ & $\mathbb{P}^*_{\text{maj}}$ & $\minPEps$ & $R_T$ & $\%R \downarrow$ & $\hat{P}_{\text{maj}}$ & \textbf{Domain} \\
        \midrule
        \tT{SEE} & 20 & 9 & 0.992 & 0.010 & \textbf{546.2} & 0.921 & \textbf{0.980} & Bernoulli \\
        Comb. UCB & 20 & 9 & 0.992 & 0.010 & 6932 & - & 0.370 & Bernoulli \\
        \addlinespace[0.1cm]
        $\theta$-WMV & 9 & 9 & 0.927 & 0.027 & \textbf{587.0} & 0.829 & \textbf{0.927} & GSM8K \\
        Zooming & 9 & 9 & 0.927 & 0.027 & 3443 & - & 0.604 & GSM8K \\
        \addlinespace[0.1cm]
        \tT{SEE} & 5 & 3 & 0.810 & 0.031 & \textbf{676.8} & 0.364 & \textbf{0.810} & CommonsenseQA \\
        Comb. UCB & 5 & 3 & 0.810 & 0.031 & 1065 & - & 0.664 & CommonsenseQA \\
        \addlinespace[0.1cm]
        $\theta$-WMV  & 5 & 5 & 0.735 & 0.013 & \textbf{360.9} & 0.512 & \textbf{0.735} & BoolQ \\
        Zooming & 5 & 5 & 0.735 & 0.013 & 740.3 & - & 0.694 & BoolQ \\
        \bottomrule
    \end{tabularx}
    \vspace{0.3em}
    \caption[Bandit experiment results]{Multi-armed bandit experiment results showing consistent performance across environments, with time horizon $T\leq10^4$. Performance metrics show reduced cumulative regret ($R_T$) and increased empirical accuracy ($\hat{P}_{\text{maj}}$) across all baselines. Extended empirical results can be found in Appendix \ref{sec:extended_empirical_appendix}.}
    \label{tab:transposed_results}
    \vspace{-2.6em}
\end{table}

We evaluate online learning in both model-based and data-derived settings. In the model-based setting, \( N \) Bernoulli experts with fixed success probabilities test each algorithm's ability to distinguish and optimize for varying competencies under noise. In the data-derived setting, we sample tasks from diverse domains: GSM8K~\cite{cobbe2021gsm8k} (mathematical reasoning), CommonsenseQA~\cite{talmor-etal-2019-commonsenseqa} (implicit knowledge), and BoolQ~\cite{clark2019boolq} (evidence-based binary inference), each presenting unique challenges. Experts are instantiated as LLMs with varying reasoning capabilities (see Appendix \ref{sec:llm_competency_app}).

\paragraph{Propose-then-Vote:} Datasets used in our experiments, such as CommonsenseQA \cite{talmor-etal-2019-commonsenseqa}, BoolQ \cite{clark2019boolq} etc., constitute structured labelled data, that is the catalogue of answers is provided to the LLM to choose from (note Remark \ref{rem:catalogue_expert}). But as we tackle unstructured open-ended tasks for experts, we need to standardize the vote aggregation functions with proper structure in order to elicit the appropriate expert feedback. One solution is to use the MoE to generate each individually their top responses to a prompt. By default, this is added to the admissible catalogue and the expert votes for their own response among the candidate responses. This, \textit{propose-then-vote}, mechanism essentially uses the committee to first determine the answer catalogue, and then vote on it. (For additional details and a complete set of comprehensive experiments please refer to Appendix \ref{sec:extended_empirical_appendix}.)

\iftoggle{nips}{\vspace{-0.8em}}

\paragraph{Results Summary:} We analyze \tT{SEE} (Algorithm~\ref{alg:succEx}) and $\theta$-WMV (Algorithm~\ref{alg:theta_wmv}) in terms of expected aggregate accuracy $\pMaj(\setExpert, \theta^T)$ (Eq.~\eqref{eq:weighted_combo_maj_vote}) and cumulative regret $R_T$ (Def.~\ref{def:regret}). As baselines, we compare against combinatorial UCB bandits and the zooming algorithm~\cite{kleinberg:2008_zooming}, which serve as no-regret benchmarks. Further details -- including prompting strategies (chain-of-thought), expert LLM specifications, baseline implementations, and ablation studies -- are provided in Appendix~\ref{sec:experimental_details_app}.

\iftoggle{nips}{\vspace{-0.7em}}

\section{Conclusion}

{We introduce a new no-regret learning framework for online decision-making with voting-based expert ensembles, unifying online learning and social choice theory as a key contribution. Our algorithms—\tT{SEE} for egalitarian voting and $\theta$-WMV for weighted majority voting—provide provable cumulative regret guarantees with UCB-style analysis, accompanied by definitive empirical performance. There are of course, also some limitations. First, although we consider up to a modest number of experts, the $\theta$-WMV algorithm depends on the solution to a mixed-integer program which could become overly complex when a large number of experts is introduced. Second, we did not consider the additional use of contextual and/or side information for the MoE framework to enhance regret minimization performance. Nevertheless, both learning algorithms offer practical advantages for ensemble MoE systems, particularly in LLM applications as demonstrated. This work establishes both theoretical foundations and practical tools for voting-based ensemble learning, opening new directions at the intersection of online learning and collective decision-making.}

\section*{Acknowledgements}

We would like to thank Nicholas Mattei, Ben Abramowitz and Ariel D. Procaccia, as well as the reviewers for their insightful feedback and suggestions for our work. 

\printbibliography

@inproceedings{richardson:2003learning,
  title={Learning with knowledge from multiple experts},
  author={Richardson, Matthew and Domingos, Pedro},
  booktitle={Proceedings of the 20th international conference on machine learning (ICML-03)},
  pages={624--631},
  year={2003}
}

@inproceedings{dudik:2015contextual_duel,
  title={Contextual dueling bandits},
  author={Dud{\'\i}k, Miroslav and Hofmann, Katja and Schapire, Robert E and Slivkins, Aleksandrs and Zoghi, Masrour},
  booktitle={Conference on Learning Theory},
  pages={563--587},
  year={2015},
  organization={PMLR}
}

@article{yue:2012_duelling_bandit,
  title={The k-armed dueling bandits problem},
  author={Yue, Yisong and Broder, Josef and Kleinberg, Robert and Joachims, Thorsten},
  journal={Journal of Computer and System Sciences},
  volume={78},
  number={5},
  pages={1538--1556},
  year={2012},
  publisher={Elsevier}
}

@book{cesa:2006_pred_learn_games,
  title={Prediction, learning, and games},
  author={Cesa-Bianchi, Nicolo and Lugosi, G{\'a}bor},
  year={2006},
  publisher={Cambridge university press}
}

@article{anthony:2024_blackmamba,
  title={Blackmamba: Mixture of experts for state-space models},
  author={Anthony, Quentin and Tokpanov, Yury and Glorioso, Paolo and Millidge, Beren},
  journal={arXiv preprint arXiv:2402.01771},
  year={2024}
}

@article{jiang:2024mixtral,
  title={Mixtral of experts},
  author={Jiang, Albert Q and Sablayrolles, Alexandre and Roux, Antoine and Mensch, Arthur and Savary, Blanche and Bamford, Chris and Chaplot, Devendra Singh and Casas, Diego de las and Hanna, Emma Bou and Bressand, Florian and others},
  journal={arXiv preprint arXiv:2401.04088},
  year={2024}
}

@article{dai:2024_deepseekmoe,
  title={Deepseekmoe: Towards ultimate expert specialization in mixture-of-experts language models},
  author={Dai, Damai and Deng, Chengqi and Zhao, Chenggang and Xu, RX and Gao, Huazuo and Chen, Deli and Li, Jiashi and Zeng, Wangding and Yu, Xingkai and Wu, Y and others},
  journal={arXiv preprint arXiv:2401.06066},
  year={2024}
}

@article{fedus:2022_sparse_expert_llm_review,
  title={A review of sparse expert models in deep learning},
  author={Fedus, William and Dean, Jeff and Zoph, Barret},
  journal={arXiv preprint arXiv:2209.01667},
  year={2022}
}

@article{shazeer:2017_sparse_moe_llm,
  title={Outrageously large neural networks: The sparsely-gated mixture-of-experts layer},
  author={Shazeer, Noam and Mirhoseini, Azalia and Maziarz, Krzysztof and Davis, Andy and Le, Quoc and Hinton, Geoffrey and Dean, Jeff},
  journal={arXiv preprint arXiv:1701.06538},
  year={2017}
}

@book{lattimore:2020_bandit_book,
  title={Bandit algorithms},
  author={Lattimore, Tor and Szepesv{\'a}ri, Csaba},
  year={2020},
  publisher={Cambridge University Press}
}

@article{nowlan:1990_moe_evaluation,
  title={Evaluation of adaptive mixtures of competing experts},
  author={Nowlan, Steven and Hinton, Geoffrey E},
  journal={Advances in neural information processing systems},
  volume={3},
  year={1990}
}

@inproceedings{niazadeh:2021_online_greedy,
  title={Online learning via offline greedy algorithms: Applications in market design and optimization},
  author={Niazadeh, Rad and Golrezaei, Negin and Wang, Joshua R and Susan, Fransisca and Badanidiyuru, Ashwinkumar},
  booktitle={Proceedings of the 22nd ACM Conference on Economics and Computation},
  pages={737--738},
  year={2021}
}

@article{hazan:2012_online_sub,
  title={Online Submodular Minimization.},
  author={Hazan, Elad and Kale, Satyen},
  journal={Journal of Machine Learning Research},
  volume={13},
  number={10},
  year={2012}
}

@article{chen:2017_submod_bandit,
  title={Interactive submodular bandit},
  author={Chen, Lin and Krause, Andreas and Karbasi, Amin},
  journal={Advances in Neural Information Processing Systems},
  volume={30},
  year={2017}
}

@article{evendar:2006_succ_elim,
  title={Action elimination and stopping conditions for the multi-armed bandit and reinforcement learning problems.},
  author={Even-Dar, Eyal and Mannor, Shie and Mansour, Yishay and Mahadevan, Sridhar},
  journal={Journal of machine learning research},
  volume={7},
  number={6},
  year={2006}
}

@article{littlestone:1994_weighted_maj_problem,
  title={The weighted majority algorithm},
  author={Littlestone, Nick and Warmuth, Manfred K},
  journal={Information and computation},
  volume={108},
  number={2},
  pages={212--261},
  year={1994},
  publisher={Elsevier}
}

@article{huang:2024_mh_moe,
  title={Multi-Head Mixture-of-Experts},
  author={Shaohan Huang and Xun Wu and Wenhui Wang and Shuming Ma and Li Dong and Furu Wei},
  journal={arXiv preprint arXiv:2404.15045},
  year={2024},
  url={https://arxiv.org/abs/2404.15045}
}

@inproceedings{gao:2022_parameter_moe,
  title={Parameter-Efficient Mixture-of-Experts Architecture for Pre-trained Language Models},
  author={Gao, Ze-Feng and Liu, Peiyu and Zhao, Wayne Xin and Lu, Zhong-Yi and Wen, Ji-Rong},
  booktitle={Proceedings of the 29th International Conference on Computational Linguistics},
  pages={3263--3273},
  year={2022},
  address={Gyeongju, Republic of Korea},
  publisher={International Committee on Computational Linguistics},
  url={https://aclanthology.org/2022.coling-1.288/}
}

@article{cai2024:survey_moe,
  title={A Survey on Mixture of Experts},
  author={Cai, Weilin and Jiang, Juyong and Wang, Fan and Tang, Jing and Kim, Sunghun and Huang, Jiayi},
  journal={arXiv preprint arXiv:2407.06204},
  year={2024},
  url={https://arxiv.org/abs/2407.06204}
}

@book{brandt:2016_handbook,
  title={Handbook of computational social choice},
  author={Brandt, Felix and Conitzer, Vincent and Endriss, Ulle and Lang, J{\'e}r{\^o}me and Procaccia, Ariel D},
  year={2016},
  publisher={Cambridge University Press}
}

@inproceedings{kleinberg:2008_zooming,
  title={Multi-armed bandits in metric spaces},
  author={Kleinberg, Robert and Slivkins, Aleksandrs and Upfal, Eli},
  booktitle={Proceedings of the fortieth annual ACM symposium on Theory of computing},
  pages={681--690},
  year={2008}
}

@article{berend:2014_binary_maj_voting,
  title={Consistency of weighted majority votes},
  author={Berend, Daniel and Kontorovich, Aryeh},
  journal={Advances in Neural Information Processing Systems},
  volume={27},
  year={2014}
}

@article{nitzan:1982_binary_voting,
  title={Optimal decision rules in uncertain dichotomous choice situations},
  author={Nitzan, Shmuel and Paroush, Jacob},
  journal={International Economic Review},
  pages={289--297},
  year={1982},
  publisher={JSTOR}
}

@article{list:2001_plurality_epistemic,
	author = {Christian List and Robert E. Goodin},
	doi = {10.1111/1467-9760.00128},
	journal = {Journal of Political Philosophy},
	number = {3},
	pages = {277--306},
	publisher = {Blackwell},
	title = {Epistemic Democracy: Generalizing the Condorcet Jury Theorem},
	volume = {9},
	year = {2001}
}

@incollection{everaere:2010_belief_merging_truth_tracking,
  title={The epistemic view of belief merging: can we track the truth?},
  author={Everaere, Patricia and Konieczny, S{\'e}bastien and Marquis, Pierre},
  booktitle={ECAI 2010},
  pages={621--626},
  year={2010},
  publisher={IOS Press}
}

@inproceedings{karge:2022_belief_fusion,
  title={The More the Worst-Case-Merrier: A Generalized Condorcet Jury Theorem for Belief Fusion.},
  author={Karge, Jonas and Rudolph, Sebastian},
  booktitle={KR},
  year={2022}
}

@article{abramowitz:2022_review_reviewers,
  title={Who Reviews The Reviewers? A Multi-Level Jury Problem},
  author={Abramowitz, Ben and Lev, Omer and Mattei, Nicholas},
  journal={arXiv preprint arXiv:2211.08494},
  year={2022}
}

@book{condorcet:1785_jury,
  title={Essai sur l'application de l'analyse à la probabilité des décisions rendues à la pluralité des voix},
  author={Condorcet, Marie Jean Antoine Nicolas de Caritat, Marquis de},
  year={1785},
  publisher={De l'Imprimerie Royale},
  address={Paris},
  language={French}
}

@inproceedings{si:2023-MoRE,
    title = "Getting {M}o{RE} out of Mixture of Language Model Reasoning Experts",
    author = "Si, Chenglei  and
      Shi, Weijia  and
      Zhao, Chen  and
      Zettlemoyer, Luke  and
      Boyd-Graber, Jordan",
    editor = "Bouamor, Houda  and
      Pino, Juan  and
      Bali, Kalika",
    booktitle = "Findings of the Association for Computational Linguistics: EMNLP 2023",
    month = dec,
    year = "2023",
    address = "Singapore",
    publisher = "Association for Computational Linguistics",
    url = "https://aclanthology.org/2023.findings-emnlp.552/",
    doi = "10.18653/v1/2023.findings-emnlp.552",
    pages = "8234--8249"
}

@article{russo:2022_satisficing,
  title={Satisficing in Time-Sensitive Bandit Learning},
  author={Daniel Russo and Benjamin Van Roy},
  journal={Mathematics of Operations Research},
  volume={47},
  number={4},
  pages={2815--2839},
  year={2022},
  publisher={INFORMS},
  doi={10.1287/moor.2021.1229}
}

@inproceedings{feng:2025_satisficing,
  title={Satisficing Regret Minimization in Bandits},
  author={Feng, Qing and Ma, Tianyi and Zhu, Ruihao},
  booktitle={The Thirteenth International Conference on Learning Representations}
}

@inproceedings{rejwan:2020_topk_combo_bandit,
  title={Top-$ k $ combinatorial bandits with full-bandit feedback},
  author={Rejwan, Idan and Mansour, Yishay},
  booktitle={Algorithmic Learning Theory},
  pages={752--776},
  year={2020},
  organization={PMLR}
}

@inproceedings{bubeck:2013_CSAR,
  title={Multiple identifications in multi-armed bandits},
  author={Bubeck, S{\'e}ebastian and Wang, Tengyao and Viswanathan, Nitin},
  booktitle={International Conference on Machine Learning},
  pages={258--265},
  year={2013},
  organization={PMLR}
}

@article{cobbe2021gsm8k,
  title={Training Verifiers to Solve Math Word Problems},
  author={Cobbe, Karl and Kosaraju, Vineet and Bavarian, Mohammad and Chen, Mark and Jun, Heewoo and Kaiser, Lukasz and Plappert, Matthias and Tworek, Jerry and Hilton, Jacob and Nakano, Reiichiro and Hesse, Christopher and Schulman, John},
  journal={arXiv preprint arXiv:2110.14168},
  year={2021}
}

@article{wei2022chain,
  title={Chain-of-thought prompting elicits reasoning in large language models},
  author={Wei, Jason and Wang, Xuezhi and Schuurmans, Dale and Bosma, Maarten and Xia, Fei and Chi, Ed and Le, Quoc V and Zhou, Denny and others},
  journal={Advances in neural information processing systems},
  volume={35},
  pages={24824--24837},
  year={2022}
}

@inproceedings{talmor-etal-2019-commonsenseqa,
    title = "{C}ommonsense{QA}: A Question Answering Challenge Targeting Commonsense Knowledge",
    author = "Talmor, Alon  and
      Herzig, Jonathan  and
      Lourie, Nicholas  and
      Berant, Jonathan",
    booktitle = "Proceedings of the 2019 Conference of the North {A}merican Chapter of the Association for Computational Linguistics: Human Language Technologies, Volume 1 (Long and Short Papers)",
    month = jun,
    year = "2019",
    address = "Minneapolis, Minnesota",
    publisher = "Association for Computational Linguistics",
    url = "https://aclanthology.org/N19-1421",
    doi = "10.18653/v1/N19-1421",
    pages = "4149--4158",
    archivePrefix = "arXiv",
    eprint        = "1811.00937",
    primaryClass  = "cs",
}

@inproceedings{clark2019boolq,
  title =     {BoolQ: Exploring the Surprising Difficulty of Natural Yes/No Questions},
  author =    {Clark, Christopher and Lee, Kenton and Chang, Ming-Wei, and Kwiatkowski, Tom and Collins, Michael, and Toutanova, Kristina},
  booktitle = {NAACL},
  year =      {2019},
}

@inproceedings{chen:2013_combinatorial,
  title={Combinatorial multi-armed bandit: General framework and applications},
  author={Chen, Wei and Wang, Yajun and Yuan, Yang},
  booktitle={International conference on machine learning},
  pages={151--159},
  year={2013},
  organization={PMLR}
}

@article{wang:2017_improving,
  title={Improving regret bounds for combinatorial semi-bandits with probabilistically triggered arms and its applications},
  author={Wang, Qinshi and Chen, Wei},
  journal={Advances in Neural Information Processing Systems},
  volume={30},
  year={2017}
}

@inproceedings{merlis:2019_batch,
  title={Batch-size independent regret bounds for the combinatorial multi-armed bandit problem},
  author={Merlis, Nadav and Mannor, Shie},
  booktitle={Conference on Learning Theory},
  pages={2465--2489},
  year={2019},
  organization={PMLR}
}

\clearpage


\clearpage

\iftoggle{nips}{
    \section{Checklist} 
\providecommand{\answerYes}[1][]{\textcolor{blue}{[Yes] #1}}
\providecommand{\answerNo}[1][]{\textcolor{orange}{[No] #1}}
\providecommand{\answerNA}[1][]{\textcolor{gray}{[NA] #1}}
\providecommand{\answerTODO}[1][]{\textcolor{red}{\bf [TODO]}}
\providecommand{\justificationTODO}[1][]{\textcolor{red}{\bf [TODO]}}

\section*{NeurIPS Paper Checklist}

\begin{enumerate}

\item {\bf Claims}
    \item[] Question: Do the main claims made in the abstract and introduction accurately reflect the paper's contributions and scope?
    \item[] Answer: \answerYes{} 
    \item[] Justification: The main claims of the paper are clearly presented in the abstract and introduction, with supporting theory and empirical evidence.

\item {\bf Limitations}
    \item[] Question: Does the paper discuss the limitations of the work performed by the authors?
    \item[] Answer: \answerYes{} 
    \item[] Justification: The limitations of the work are mentioned in the Conclusion.

\item {\bf Theory assumptions and proofs}
    \item[] Question: For each theoretical result, does the paper provide the full set of assumptions and a complete (and correct) proof?
    \item[] Answer: \answerYes{} 
    \item[] Justification: Assumptions are clearly stated in Section \ref{sec:problem_setting}.

    \item {\bf Experimental result reproducibility}
    \item[] Question: Does the paper fully disclose all the information needed to reproduce the main experimental results of the paper to the extent that it affects the main claims and/or conclusions of the paper (regardless of whether the code and data are provided or not)?
    \item[] Answer: \answerYes{} 
    \item[] Justification: Experimental results are clearly presented in the paper, with reproduction data in Appendix \ref{sec:experimental_details_app}.

\item {\bf Open access to data and code}
    \item[] Question: Does the paper provide open access to the data and code, with sufficient instructions to faithfully reproduce the main experimental results, as described in supplemental material?
    \item[] Answer: \answerYes{} 
    \item[] Justification: The methods and procedures to reproduce the results are clearly outlined in Algorithm \ref{alg:succEx} and \ref{alg:theta_wmv}. Moreover, implementation details are outlined in Appendix \ref{sec:experimental_details_app}.

\item {\bf Experimental setting/details}
    \item[] Question: Does the paper specify all the training and test details (e.g., data splits, hyperparameters, how they were chosen, type of optimizer, etc.) necessary to understand the results?
    \item[] Answer: \answerYes{} 
    \item[] Justification: We provide full details in Appendix \ref{sec:experimental_details_app}.

\item {\bf Experiment statistical significance}
    \item[] Question: Does the paper report error bars suitably and correctly defined or other appropriate information about the statistical significance of the experiments?
    \item[] Answer: \answerYes{} 
    \item[] Justification: Uncertainty estimates are provided in all graphs presented in the paper.

\item {\bf Experiments compute resources}
    \item[] Question: For each experiment, does the paper provide sufficient information on the computer resources (type of compute workers, memory, time of execution) needed to reproduce the experiments?
    \item[] Answer: \answerYes{} 
    \item[] Justification: Experimental hardware specifications can be found in Appendix \ref{sec:experimental_details_app}.
    
\item {\bf Code of ethics}
    \item[] Question: Does the research conducted in the paper conform, in every respect, with the NeurIPS Code of Ethics \url{https://neurips.cc/public/EthicsGuidelines}?
    \item[] Answer: \answerYes{} 
    \item[] Justification: Our research complies with all ethical guidelines, in no way do we have intent or motivation to produce harmful or unethical output detrimental to humanity.

\item {\bf Broader impacts}
    \item[] Question: Does the paper discuss both potential positive societal impacts and negative societal impacts of the work performed?
    \item[] Answer: \answerYes{} 
    \item[] Justification: In the introduction of the paper we discuss several applications of the paper, including LLM alignment and achieving better understanding of efficient agent learning. 
    
\item {\bf Safeguards}
    \item[] Question: Does the paper describe safeguards that have been put in place for responsible release of data or models that have a high risk for misuse (e.g., pretrained language models, image generators, or scraped datasets)?
    \item[] Answer: \answerNA{} 
    \item[] Justification: There exists no potential for high risk misuse. 

\item {\bf Licenses for existing assets}
    \item[] Question: Are the creators or original owners of assets (e.g., code, data, models), used in the paper, properly credited and are the license and terms of use explicitly mentioned and properly respected?
    \item[] Answer: \answerYes{} 
    \item[] Justification: Code and methodology can be openly used with proper credits given to authors.

\item {\bf New assets}
    \item[] Question: Are new assets introduced in the paper well documented and is the documentation provided alongside the assets?
    \item[] Answer: \answerNA{} 
    \item[] Justification: No digital assets were used in this paper.

\item {\bf Crowdsourcing and research with human subjects}
    \item[] Question: For crowdsourcing experiments and research with human subjects, does the paper include the full text of instructions given to participants and screenshots, if applicable, as well as details about compensation (if any)? 
    \item[] Answer: \answerNA{} 
    \item[] Justification: No crowdsourcing was used in this paper.

\item {\bf Institutional review board (IRB) approvals or equivalent for research with human subjects}
    \item[] Question: Does the paper describe potential risks incurred by study participants, whether such risks were disclosed to the subjects, and whether Institutional Review Board (IRB) approvals (or an equivalent approval/review based on the requirements of your country or institution) were obtained?
    \item[] Answer: \answerNA{} 
    \item[] Justification: No study participants were used in this work.

\item {\bf Declaration of LLM usage}
    \item[] Question: Does the paper describe the usage of LLMs if it is an important, original, or non-standard component of the core methods in this research? Note that if the LLM is used only for writing, editing, or formatting purposes and does not impact the core methodology, scientific rigorousness, or originality of the research, declaration is not required.
    \item[] Answer: \answerNA{} 
    \item[] Justification: LLM's were not used as any important, original, or non-standard component of the core methods in this research.

\end{enumerate}

    \clearpage    
}

\appendix

\section{Proofs for Core Theory}

\subsection{Proof for Lemma \ref{lem:expert_ordinality}} \label{prf:expert_ordinality}

\textbf{Top-K Ordinality of Experts:} The optimal expert committee $\setExpert^*\subseteq \setExpert$ is a subset of Top-K experts based on their competencies, for $0 < K \leq N$.

\begin{proof}
Suppose there is an expert $i$ in the optimal committee $\setExpert^*$ and another expert $j$ with a higher competency, i.e., $p_i<p_j$ who is  not included in the optimal expert committee. 
In this case, replacing experts $i$ and $j$ results in a new committee with higher average score, leading to a contradiction. We prove it for when the size of the optimal committee is odd. Similar proof can be written for when it is even. 
Let $|\setExpert^*|=2s+1$ for some $s$, $i\in \setExpert^*$ but $j\not\in \setExpert^*$ and $p_i<p_j$. We form a new committee $\setExpert_0=\setExpert^*\cup\{j\}\setminus\{i\}$. 
    Let $S_i(k)$ be the event that precisely $k$ experts in $\setExpert^*$ are correct including expert $i$. 
    We also define $S^c_i(k)$ to be the event that $i$ is incorrect but $k-1$ other experts voted correctly in $\setExpert^*$. 
    Similarly, we define $S^c_j(k)$ and $S^c_j(k)$ in committee $\setExpert_0$. 
    It is straightforward to see that the probability of either $S^c_i(k)$ or $S_i(k)$ occurs in committee $\setExpert^*$ is equal to the probability of either $S^c_j(k)$ or $S_j(k)$ occurs in $\setExpert_0$ for all $k\in\{s+2,...,2s+1\}$. 
    This implies the sum of these events are the same for both committees in \eqref{eq:combo_maj_vote}. 
    Note that when $k=s+1$, both events $S^c_i(k)$ and $S^c_j(k)$ will not appear in the average scoring in \eqref{eq:combo_maj_vote} as in both of these events $s$ experts are correct and $s+1$ are incorrect. Moreover, the probability that event $S_j(s+1)$ occurs is higher than the probability of the event $S_i(s+1)$ since in the former expert $j$ is correct and in the latter expert $i$ and $p_j>p_i$.
\end{proof}

\subsection{Proof for Lemma \ref{lem:maj_vote_dominance}} \label{prf:maj_vote_dominance}

\textbf{Majority Voting Dominance over Egalitarian Voting:} For any set of experts, the optimal majority weighted voting method will always yield a stronger or equal predictive accuracy then the plurality vote of the committee. 

\begin{proof}
    Suppose $\setExpert^*$ denotes the OEC where $|\setExpert^*|=M$. In this case, we define $\theta$ such that $\theta_i=1$ if $i\in\setExpert^*$ and zero, otherwise.  
    Next, we show that $\pMaj(\setExpert, \theta)\geq \pMaj(\setExpert^*)$, where the quota is $Q$ rather than $|\setExpert^*|/2$ as in \eqref{eq:combo_maj_vote}. Let $S_1=S\cap \setExpert^*$, $S_2=S\cap (\setExpert^*)^c$, then,
    
    \begin{align*}
            \pMaj(\setExpert, \theta)&=\!\sum_{S\subseteq\setExpert} \! \feasFunc[S, \theta]{N} \left( \prod_{i \in  S} p_i  \prod_{j\in S^c}  (1-p_j) \right)\\
            &=\sum_{S_1\subseteq\setExpert^*}\sum_{S_2\subseteq(\setExpert^*)^c} \! \feasFunc[S_1\cup S_2, \theta]{N} \left( \prod_{i \in  S_1} p_i  \prod_{j\in \setExpert^*\setminus S_1}  (1-p_j) \right)\left( \prod_{i \in  S_2} p_i  \prod_{j\in (\setExpert^*)^c\setminus S_2}  (1-p_j) \right)
    \end{align*}
    Note that,
    \begin{align*}
        \feasFunc[S, \theta]{N}\!=\! \mathbb{I}\left(\sum_{j\in S}\theta_j>Q\right)=1 \Leftrightarrow |S\cap\setExpert^*|=|S_1|>Q,
    \end{align*}
    This implies that the nonzero terms in the above summation are those with $|S_1|>Q$. Therefore, we can simplify the summation as
        \begin{align*}
            \pMaj(\setExpert, \theta)&=\sum_{\substack{S_1\subseteq\setExpert^*\\ |S_1|>Q}}\sum_{S_2\subseteq(\setExpert^*)^c}  \left( \prod_{i \in  S_1} p_i  \prod_{j\in \setExpert^*\setminus S_1}  (1-p_j) \right)\left( \prod_{i \in  S_2} p_i  \prod_{j\in (\setExpert^*)^c\setminus S_2}  (1-p_j) \right)\\
            &=\sum_{\substack{S_1\subseteq\setExpert^*\\ |S_1|>Q}} \left( \prod_{i \in  S_1} p_i  \prod_{j\in \setExpert^*\setminus S_1}  (1-p_j) \right)\sum_{S_2\subseteq(\setExpert^*)^c} \left( \prod_{i \in  S_2} p_i  \prod_{j\in (\setExpert^*)^c\setminus S_2}  (1-p_j) \right)
        \end{align*}
        \begin{align*}
            =\sum_{\substack{S_1\subseteq\setExpert^*\\ |S_1|>Q}} \left( \prod_{i \in  S_1} p_i  \prod_{j\in \setExpert^*\setminus S_1}  (1-p_j) \right)= \pMaj(\setExpert^*).
    \end{align*}
\end{proof}

\subsection{Proof for Proposition \ref{prop:discrete_img}} \label{prf:discrete_img}

\textbf{Discrete Image:} 
    Many configurations of $\theta$ can lead to the same result for $\pMaj(\setExpert, \theta)$. 
\begin{proof}
    $\pMaj(\setExpert, \theta)$ is a function with finite number of discontinuities (jumps), which is bounded by $2^N$, the number of possible committees. 

    Let $\theta=[\theta_1,...,\theta_N]$ be a set of weights and $\mathcal{S}=\{S_1,..., S_k\}$ be all the set of configurations for which $\sum_{j\in S_i}\theta_j>Q$ for all $i\in[K]$. Let $S_1$ be the configuration which has the lowest sum of weights, i.e., $S_1=\arg\min_i \sum_{j\in S_i}\theta_j$. 
    Then, define $\theta^{\epsilon}$ to be a new set of weights in which some of the weights in $S_1$ are disturbed by $\epsilon$ such that $\sum_{j\in S_1}\theta_j=\sum_{j\in S_1}\theta^\epsilon_j$ and the remaining weights are left unchanged.
    After this distribution, some weights in $S_1$ are reduced by $\epsilon$, some increased, and some are left unchanged. 
    Let $S_0=\arg\max_{S\not\in \mathcal{S}}\sum_{j\in S}\theta_j$. Then, set $\epsilon \leq \min\{\sum_{j\in S_1}\theta_j -Q , Q- \sum_{j\in S_0}\theta_j\} /N$. 
    Clearly, by this choice of $\epsilon$, all the configurations in $\mathcal{S}$ will pass the aggregation criterion and no other configurations will be added to the $\mathcal{S}$, that is $\pMaj(\setExpert, \theta)=\pMaj(\setExpert, \theta^\epsilon)$.
\end{proof}

\subsection{Remark \ref{rem:wedge}: Wedge Property} \label{sec:wedge}

\begin{remark} \label{rem:wedge}
    \textbf{Wedge Property:} Provided expert competencies, $p_1, p_2 \dots, p_N$, with the ordering $p_1 \geq p_2 \dots \geq p_N$, there exists an optimal solution with parameters $\theta^*$ that maximizes $\pMaj(\setExpert, \theta)$ and are constrained by the ordering,  $\theta_1  \geq \dots, \geq \theta_N$. 
\end{remark} 

The intuition behind Remark \ref{rem:wedge} hinges on the idea that more competent voters would contribute more in terms of voting power in the expert committee compared to less competent experts. Voting power can be measured by various metrics such as the Shapley-Shubik Index or Banzhaf Power Index, which measures the amount of contribution that each voter contributes to swinging the collective decision. Having more weight, or allocated votes, to the expert reflects this property. Therefore we can expect a monotone increasing wedge property on the expert weight $\theta_i$ proportional to the ordering on $p_i$. Enabling this as a constraint for any optimization program will tighten the search space for optimal solutions, making it more computationally efficient.


\subsection{Proof for Lemma \ref{lem:equality_Q}} \label{prf:equality_Q}

\textbf{Weights Satisfied at Equality:} When solving for the optimal $\theta^*$ for $\theta$-weighted majority voting, we can replace the constraint \textup{$Q \leq \sum_{i=1}^N \theta_i$}, from Eq. \eqref{eq:quota_indicator_conditions}, with a binding equality \textup{$\sum_{i=1}^N \theta_i = 2Q$}, to obtain a valid solution. (Proof is available in Appendix \ref{prf:equality_Q}.)

\begin{proof}


    We can begin by considering $2Q$, or two times the quota, as the total number of votes in an electorate. We would like to allocate the total number of allowable votes among the voters to ensure total representation, allocating any less could end up with suboptimalities. If the quota, $Q$, is too high in comparison to the total number of votes allotted, then it could be the case that some participants who would have contributed optimally to the voting committee, are unable to participate due to quota limitations. In order to ensure majority voting, more than half of the total votes must pass for any certain choice - by definition. Therefore, between $Q$ and $2Q$, the objective $\pMaj(\theta)$ is strictly non-decreasing, as more useful  voters could potentially be included in the voting pool without sacrificing exclusion of others. Since the objective function is strictly non-decreasing w.r.t. in put $\theta$, then for any mixed-integer program, the constraint $Q \leq \sum_{i=1}^N \leq 2Q$ can be replaced by a binding constraint $\sum_{i=1}^N \theta_i = 2Q$.
    
\end{proof}

\subsection{Proof for Lemma \ref{lem:pac_sample_complexity_shrink}} \label{prf:pac_sample_complexity_shrink}

\textbf{PAC Sample Complexity:} For any two distributions for experts $i$ and $j$, if the number of samples $t$ and shrinkage term $\ucB_i(t)$, satisfies,
\[
    \ucB_i^t = \sqrt{\frac{2 \sigma_i^2 \log(4/\delta)}{t}}
    , \quad t_{i,j} \geq  \frac{16 (\sigma_i^2 + \sigma_j^2) \log(4/\delta)}{\minPEps^2}.
\]
then, with probability at least $1 - \delta$, $\delta \in (0, 1)$, the breakage condition, $\breaK^t$, will be met.

\begin{proof}
Without loss of generality, let us consider $i=1$ and $j=2$. Let there exist two sub-Gaussian distributions parameterized with means $p_1, p_2$ and variances $\sigma_1^2, \sigma_2^2$, respectively. Given $\delta \in (0,1)$, we seek the minimal number of samples $t$ required to distinguish between the two distributions with probability at least $1 - \delta$. Due to Hoeffding's inequality, with respect to the number of samples $t$, the empirical means $\hat{p}_1^t$ and $\hat{p}_2^t$ must satisfy, 
\[
\mathbb{P}\left(|\hat{p}_i^t - p_i| \geq \ucB_i^t\right) \leq 2 \exp\left(-\frac{t \,(\ucB_i^t)^2}{2 \sigma_i^2}\right), \quad i \in \{1,2\}.
\]
Let us  set the confidence bound shrinkage parameter $\ucB_i^t$ as,
\[
\ucB_i^t := \sqrt{\frac{2 \sigma_i^2 \log(4/\delta)}{t}}, \quad i \in \{1,2\}.
\]
The condition for breakage, $\breaK^t$, is indicated by non-overlapping confidence intervals, can be expressed as,
\[
|\hat{p}_1^t - \hat{p}_2^t| > \epsilon_1(t) + \epsilon_2(t).
\]
Therefore, with probability at least $1 - \delta$, the intervals are separated when
\[
\Delta_{12} \geq 2\sqrt{\frac{4 (\sigma_1^2 + \sigma_2^2) \log(4/\delta)}{t}}\geq 2\left(\epsilon_1(t)+\epsilon_2(t)\right).
\]
Solving for $t_{i,j}$ yields the sample complexity:
\[
t_{i,j}\geq \frac{32 \sigma^2 \log(4/\delta)}{\minPEps^2} \geq \frac{16 (\sigma_i^2 + \sigma_j^2) \log(4/\delta)}{\minPEps^2}.
\]
where $\Delta_{ij}:=|p_i-p_j|$, $\minPEps=\min_{i\neq j}\Delta_{ij}$, and $\sigma_i\geq\sigma$ for all $i$.
\end{proof}

\subsection{Proof for Theorem \ref{thm:regret_bound_succex_log}} \label{prf:regret_bound_succex_log}

\textbf{Regret Bound for \tT{SEE} Algorithm:} The \tT{SEE} algorithm achieves a total regret $R_T$ bounded by,
\begin{align*}
    R_T \in \mathcal{O}\left( \frac{N}{\minPEps^2} \log T \right).
\end{align*}
\vspace{-.7cm}
\iftoggle{nips}{}{
    \textcolor{red}{
    We have missed something here. So the issue is that even if all the  breakages occur correctly after $t$ iterations, then when the SEE checks for elimination, it might eliminate a wrong set of experts and this may produce constant regret. 
    For example, let say we have only three experts with $p_1>p_2>p_3$ such that 
    $$
    A := p_1(1-p_2)(1-p_3) - (1-p_1)p_2p_3.
    $$
    Then it is easy to check when $A>0$, then $\setExpert^*=\{1\}$ otherwise $\setExpert^*=\{1,2,3\}$. 
    $$
    \hat{A} := \hat{p}_1(1-\hat{p}_2)(1-\hat{p}_3) - (1-\hat{p}_1)\hat{p}_2\hat{p}_3.
    $$
    Now suppose, after $t$ rounds, we have good estimations of the competencies. but what decides whether SEE suffers regret or not is $\hat{A}$. If both $A$ and $\hat{A}$ have the same sign, then no regret otherwise, there will be a small regret depending on $|\pMaj(\{1\})-\pMaj(\{1,2,3\})|$.
    }
    \textcolor{red}{
    Another thing is about the current version of SEE. It starts with $\setExpertCand=\setExpert$ which means $\advFunc(\setExpertCand, \setExpert')=0$ always because $\setExpertCand\cup \setExpert'=\setExpertCand$. This means that at the first round, line 7 will be executed. This should not be the case. So the condition in 6 of SEE should be $<0$ and also in the main text line 141, it should be $    \advFunc(\setExpertCand, \setExpert') < 0 \implies \setExpert^* \neq \setExpertCand \cup \setExpert'. $
    }
    \ja{
    Since, we eventually require that for all experts their estimated competencies are accurate enough such that all breakages (if any) occur with high confidence, then we have to apply the union bound over all experts. This leads to 
    \begin{align*}
    \mathbb{P}\Big( \left|  \hat{p}^t_i - p_i \right| \leq \ucB_i^t , \forall i\in\{1,...,N\}\Big) \geq 1 - 2\sum_{i=1}^N \exp\left( -\frac{t (\ucB_i^t)^2}{2\sigma^2} \right).
    \end{align*}
    This changes the regret bound by an addition $N$ inside the logarithm. 
    More precisely, by setting
    $\ucB_i^t\leq \minPEps/4$ for all $i$, we get
    \begin{align*}
    \mathbb{P}\Big( \left|  \hat{p}^t_i - p_i \right| \leq \minPEps/4 , \forall i\in\{1,...,N\}\Big) \geq 1 - 2N \exp\left( -\frac{t \minPEps^2}{32\sigma^2} \right).
    \end{align*}
    On the other hand, for any given candidate set $\setExpertCand$ (assuming $|\setExpertCand|$ is odd, similar line of reasoning also works for the even case), when $|p_i-\hat{p}^t_i|\leq \epsilon$ for all experts, we could write
    \begin{align*}
        &|\pMaj(\setExpertCand) -\widehat{\pMaj}^t(\setExpertCand)|\\
        &=\left|\sum_{\substack{S\subseteq\setExpertCand \\ |S|\geq |\setExpertCand|/2}}\prod_{i\in S}p_i\prod_{j\in S^c}(1-p_j)-\sum_{\substack{S\subseteq\setExpertCand \\ |S|\geq |\setExpertCand|/2}}\prod_{i\in S}\hat{p}_i\prod_{j\in S^c}(1-\hat{p}_j)\right|\\
        &\leq 2^{|\setExpertCand|}\max_{S\subseteq\setExpert}\left|\prod_{i\in S}p_i\prod_{j\in S^c}(1-p_j)-\prod_{i\in S}\hat{p}_i\prod_{j\in S^c}(1-\hat{p}_j)\right|\leq 2^{|\setExpertCand|}\sum_{i=1}^N|p_i-\hat{p}_i|\leq 2^N N\epsilon.
    \end{align*}
    Requiring the above probability to be at least $1-\delta=1-1/T$ and using the above inequalities give us
    \begin{align*}
       \mathbb{P}\left(|\pMaj(\setExpertCand) -\widehat{\pMaj}^t(\setExpertCand)|\leq \epsilon_0(t)\right) >1-\frac{1}{T},
    \end{align*}
    where 
    \begin{align*}
        \epsilon_0(t)=\sqrt{\frac{4^NN^2\sigma^2}{t}\log(2NT)}
    \end{align*}
    Recall that \tT{SEE} performs an elimination when 
    \begin{align*}
        \advFunc(\setExpertCand, \setExpert') < 0, \quad \forall \setExpert' \in \tT{TopK}(\setExpertCand^{\leq j})
    \end{align*}
    There are at most $N$ sets in $\tT{TopK}(\setExpertCand^{\leq j})$ and in order for the algorithm to perform the elimination correctly, all these sets should 
    Defining $t_0:=$, and keeping in mind that $\pMaj(\setExpert, \theta^*) - \pMaj(\setExpert, \theta_t)\leq 1$, result in 
    \begin{align*}
        R_T\leq \sum_{t=1}^{t_0}1 + \frac{1}{T} \sum_{t=t_0+1}^T 1+ (1-\frac{1}{T})\sum_{t=t_0+1}^T 0\in\mathcal{O}\left(\frac{1}{\minPEps^2}\log\big(NT\big)\right).
    \end{align*}
    Note that 
    \begin{align*}
        \max_{S\subseteq\setExpert}\left|\prod_{i\in S}p_i\prod_{j\in S^c}(1-p_j)-\prod_{i\in S}\hat{p}_i\prod_{j\in S^c}(1-\hat{p}_j)\right|\leq\sum_{i=1}^N|p_i-\hat{p}_i|.
    \end{align*}
    we need to bound
    \begin{align*}
        |\pMaj(\setExpertCand\cup\{k\}) -\pMaj(\setExpertCand)|=\left|\frac{1}{2}-p_k\right|\cdot\sum_{\substack{S\subseteq\setExpertCand \\ |S|= |\setExpertCand|/2}}\prod_{i\in S}p_i\prod_{j\in S^c}(1-p_j)
    \end{align*}
    }
}
\begin{proof}
\textbf{Pairwise Comparison Regret:} During execution, there exist $\binom{N}{2}$ possible pairwise comparisons leading to expert elimination, but the algorithm would encounter at most $N-1$ elimination events. For each elimination round, we can expect the most samples to be drawn when $\Delta_{ij}$ between experts, $i$ and $j$, is minimized, therefore, the maximum sample complexity to achieve full breakage among all experts is upper-bounded by the sum over $t_{i,i+1}$ from $i$ goes from 1 to $N-1$. Furthermore, we also impose conditions to limit the sub-optimal inclusion of experts under conservative advantage function estimates (Def. \ref{def:conservative_exclusion}), so long as the $\ucB$ is sufficiently sampled such that $2 \, \ucB < \minGapMaj$ (illustrated by Condition \ref{cond:consistent_inclusion}). As a consequence Lemma \ref{lem:pac_sample_complexity_shrink} will still hold by the same derivation, but we take the smaller of two gaps, where we define $\minGapSEE$ as, 
\begin{align*}
    \minGapSEE := \min \{ \minPEps, \minGapMaj \}.
\end{align*}
where $\minPEps = \min_{i\neq j}\Delta_{ij}$ and $\minGapMaj$ can be computationally derived per Def. \ref{def:minGapMaj}. Leveraging the sample complexity result from Lemma \ref{lem:pac_sample_complexity_shrink}, the regret decomposes as,

{\small
\begin{align}
    R_T \leq \sum_{i=1}^{N-1}\left(\underbrace{\frac{16(\sigma_i^2 + \sigma_{i+1}^2)\log(4/\delta)}{\minGapSEE^2}}_{\text{Exploration Cost}} + \underbrace{\delta T\bar{\Delta}}_{\text{Failure Cost}}\right) \label{eq:see_pac_regret_expr_delta}.
\end{align}
}
Here, we apply the fact that the optimality gap is upper bounded, $\pMaj(\setExpert, \theta^*) - \pMaj(\setExpert, \theta_t)\leq 1$.

In order to translate the sample complexity result from Lemma \ref{lem:pac_sample_complexity_shrink} into a regret bound, we must realize that we expect to play suboptimally with probability $1-\delta$. Next, $\delta$ represents the probability where a catastrophic breakage even could happen (i.e. the algorithm eliminated expert $i$ from the committee, even though they were supposed to be part of the optimal expert committee). If catastrophic breakage were to occur, we could obtain a upper bounded worst-case regret of $T\bar{\Delta}$, where $\bar{\Delta}\leq1$ (the maximum expert competency difference) - so essentially a regret of $T$. Therefore, we select $\delta = 1/T$ to offset this issue. By selecting $\delta = 1/T$, the failure cost is effectively an upper bounded by a constant (in this case 1), and the exploration cost scales logarithmically with the number of samples $T$. We can further express Eq. \eqref{eq:see_pac_regret_expr_delta} as,
\begin{align*}
    R_T &\leq \sum_{i=1}^{N-1} \left( \frac{32\sigma^2}{\minGapSEE^2} \log(4T) + 1 \right)\in \bigO \left( \frac{N\sigma^2}{ \minGapSEE^2} \log(T) \right).
\end{align*}
\end{proof}

\clearpage

\subsection{Mixed Integer Program Formulation for Optimal Weights} \label{sec:mip_solution_appendix}





\paragraph{Problem Formulation:} Let \( \boldsymbol{\theta} = [\theta_1, \theta_2, \dots, \theta_N] \) be a continuous positive vector, where \( \theta_i \geq 0 \) for all \( i = 1, 2, \dots, 2^N \). Let \( \mathbf{x} = [x_1, x_2, \dots, x_{2^N}] \) be a given fixed vector. Define \( \mathbf{z} = [\indicMip_1, \indicMip_2, \dots \indicMip_{2^N}] \) as a binary vector corresponding to each scenario $S$, as determined by the following if-else condition,

\[
\indicMip_S = 
\begin{cases} 
1 & \text{if } \mathbf{x}_i \cdot \boldsymbol{\theta} > Q \\
0 & \text{otherwise}.
\end{cases}
\]
where $Q \in \mathbb{R}^+$ is a given constant, and $\indicMip_S$ is the indicator of the specific scenario. Then, the objective in \eqref{eq:mip_objective_func} can be written  as the dot product \( \mathbf{z}^\top \mathbf{p} \), where \( \mathbf{p}=[p_1, \dots, p_{|\setS|}]\) and $p_S$,  with a slight abuse of notation, is defined by,
\[ 
    p_S := \prod_{i \in  S} p_i  \prod_{j\in S^c}  (1-p_j), \quad S\in \setS,
\]

where $\setS$ is the size of all scenarios of size $2^N$.

\paragraph{Mixed Integer Linear Program Formulation:} The problem can be formulated as a mixed-integer linear program (MIP) as follows:
\begin{align}    
    \max_{\theta \in \mathbb{R}^{N}} \,\, \sum_{S \in \setS} \indicMip_S \prod_{i \in  S} p_i  \prod_{j\in S^c}  (1-p_j) \label{eq:mip_indicator_condition}
\end{align}
subject to: 
\begin{align}
    &\sum_{j=1}^N x_{j} \theta_j - Q \geq \epsilon_{\tT{MIP}} \indicMip_S, && \forall S \in \setS \label{eq:constraint1} \\
    &\sum_{j=1}^N x_{j} \theta_j - Q \leq M \indicMip_S, && \forall S \in \setS \label{eq:constraint2} \\
    &\sum_{j=1}^N \theta_j = 2Q, \label{eq:constraint3} \\
    &\theta_j - \theta_k \geq \minPEps, && \forall j, k \in \{1, 2, \dots, N\}, j < k, \label{eq:constraint4} \\
    &\indicMip_S + \indicMip_{S^c} \leq 1, && \forall S \in \setS, \text{ where } \mathbf{x}_k = \mathbf{1} - \mathbf{x}_i,  \label{eq:constraint6} \\
    &\indicMip_S \in \{0, 1\}, && \forall S \in \setS . \label{eq:constraint8}
\end{align}
As standard notation for indicator variable $\indicMip_s$, we define $M$ as a very large number, and $\epsilon_{\tT{MIP}}$ as a very small number, such that the indicator condition in Eq. \eqref{eq:mip_indicator_condition} will hold as specified.

\paragraph{Description of Constraints}

This MIP is formulated with the objective variable as the dot product between a binary selection variable which dictates the possible voting outcomes, and a \textit{possibility variable} $z_i$ which dictates the possibility of this outcome among the constraints . This produces a valid sum over all possibilities and their individual expected values, over which the summation over indicates the expectation over the configuration provided a set of weights $\theta$.

\begin{itemize}
    \item \textbf{Correct Vote Indicators:} The two constraints, \eqref{eq:constraint1} and \eqref{eq:constraint2}, ensure that \(z_i = 1\) if \(\sum_{j=1}^N x_{ij} \theta_j \geq Q + \epsilon\), or \(z_i = 0\) otherwise.
    \item \textbf{Electorate Approval Condition:} the two constraints  \eqref{eq:constraint1} and \eqref{eq:constraint8}, the variable \(z_i\) is binary and is determined by the condition \(\sum_{j=1}^N x_{ij} \theta_j \geq Q + \epsilon\).
    \item \textbf{Equality at Quota:} The constraint \eqref{eq:constraint3} ensures the total sum of \(\theta_j\) is fixed to $2Q$. Lemma \ref{lem:equality_Q} provides constraint tightening which allows to solve address the quota inequality with a more convenient constraint.
    \item \textbf{$\minPEps$-Incrementality (Heterogeneous Expert Competencies) and $\theta$-Wedge Property:} \eqref{eq:constraint4} enforces the Assumption \ref{ass:eps-increm}, and the $\theta$-Wedge property in Remark \ref{rem:wedge}, imposing a minimum difference between all pairs of \(\theta_i, \theta_j\).
    \item \textbf{Proper Outcome Complements:} The constraint \eqref{eq:constraint6} ensures that complementary binary vectors \(\mathbf{x}_i\) and \(\mathbf{x}_k\) cannot both have \(z_i = 1\) and \(z_k = 1\). (e.g. a specific outcome where expert 1,3,4 votes correctly, cannot have any of 1,3, or 4 also voting incorrectly)
\end{itemize}



\subsection{Assumption on Summation of Competencies}

 Let us denote in shorthand, $\mathbf{z}( \mathbf{p} ) := \arg\max_{\mathbf{z}} \mathbf{z}^\top \mathbf{p}$ such that $\mathbf{z}( \mathbf{p} )$ is the value of $\mathbf{z}$ that maximizes the dot product of any given $\mathbf{p}$. Let $\bar{\mathbf{p}}$ denote the optimistic estimate estimates and $\underline{\mathbf{p}}$ represent the pessimistic estimates of scenario probabilities based on $\ucB$ uncertainty for expert competencies.

    \begin{assumption} \label{ass:bounding_condition}
        For scenario probabilities $\mathbf{p}$ and scenario indicator vector $\mathbf{z} \in \{0, 1\}^{|\setS|}$, the inequality holds such that,
        \begin{align}
            \sum_{i=1}^N \bar{p}_i - \mathbf{z}^\top \bar{\mathbf{p}} \geq \sum_{i=1}^N \underline{p}_i - \mathbf{z}^\top\underline{\mathbf{p}}, \label{eq:bounding_condition}
        \end{align}
        where, $\bar{p}_i$ and  $\underline{p}_i$ refer to the maximum and minimum probabilities subject to $\ucB$ bounds over competency estimates, i.e. $p_i \pm \ucB, \, \forall p_i \in p_{1}, \dots, p_N$.
    \end{assumption}

    To provide additional discourse on Assumption \ref{ass:bounding_condition}, we assume that the variation in $\mathbf{z}(\bar{\mathbf{p}})^\top \bar{\mathbf{p}} - \mathbf{z}(\underline{\mathbf{p}})^\top\underline{\mathbf{p}}$ is sufficiently small compared to the fluctuation in the shrinking $\ucB$. In the worst case scenario, for each $2^{N-1}$ scenario probabilities, we could add a deviation term of $N \times \ucB$ for each scenario (although this is a highly conservative estimate).
    
    \textbf{Bonferroni Expansion:} The \textit{Bonferroni inequalities} provide a sequence of upper and lower bounds on the probability of the union of \(N\) events, refining the standard inclusion-exclusion principle. These inequalities alternate between over- and under-estimates as more terms are included, ultimately converging to the exact probability when all \(N\) terms are considered.

    For events \( A_1, A_2, \dots, A_N \), the exact probability of their union is given by:
    \[
    P\left( \bigcup_{i=1}^N A_i \right) = \sum_{k=1}^N (-1)^{k+1} \tT{S}_k,
    \]
    where \(\tT{S}_k\) represents the sum of the probabilities of all \(k\)-wise intersections:
    \[
    \tT{S}_k = \sum_{1 \leq i_1 < i_2 < \dots < i_k \leq N} P\left( A_{i_1} \cap A_{i_2} \cap \dots \cap A_{i_k} \right).
    \]

    \textbf{General \(k\)-th Order Bonferroni Inequality:}  The Bonferroni inequalities state that truncating the inclusion-exclusion series at the \(k\)-th term yields an upper or lower bound depending on whether \(k\) is odd or even. If \(k\) is odd, the partial sum is an upper bound,
    \[
    P\left( \bigcup_{i=1}^N A_i \right) \leq \sum_{j=1}^k (-1)^{j+1} \tT{S}_j.
    \]
    If \(k\) is even, the partial sum is a lower bound,
    \[
    P\left( \bigcup_{i=1}^N A_i \right) \geq \sum_{j=1}^k (-1)^{j+1} \tT{S}_j.
    \]
    Let $(A_1, A_2, A_3, \dots, A_N)$ be $N$ independent events with success probabilities $(P(A_1), P(A_2), P(A_3), \dots, P(A_N))$,

    \small
    \begin{align*}
        P\left( \bigcup_{i=1}^N A_i \right) &= P(A_1) + P(A_2) + P(A_3) \dots + P(A_N) 
        - P(A_1 \cap A_2) - P(A_1 \cap A_3) - P(A_2 \cap A_3) \dots \\
        &+ P(A_1 \cap A_2 \cap A_3) \dots 
        - \dots 
        \pm P(A_1 \cap A_2 \cap A_3 \dots \cap A_N) 
    \end{align*}
    \normalsize

    Given expert competencies, we could express,
    \begin{align*}
        P\left( \bigcup_{i=1}^N A_i \right) 
        = \underbrace{p_1 + p_2 + p_3 \dots + p_N}_{\tT{S}_1}
        - \underbrace{p_1p_2 - p_1p_3 - p_2p_3 \dots}_{\tT{S}_2} 
        + \underbrace{p_1 p_2 p_3 \dots}_{\tT{S}_3} 
        - \dots 
        \pm \underbrace{\prod^N_{i=1} p_i}_{\tT{S}_N}
    \end{align*}
    If we were to vary $p_i \pm \ucB$, we would incorporate this modification into all of the terms in the Bonferroni expansion. Given the alternating $\pm$ structure of the Bonferroni expansion, Ass. \ref{ass:bounding_condition} posits that the difference between, any two MIP solutions under $\bar{p}_i$ versus $\underline{p}_i$ would entail a difference of,
    \begin{align*}
        \mathbf{z}^\top \mathbf{\bar{p}} - \mathbf{z}^\top \mathbf{\underline{p}}  
        &= \sum_{i=1}^N \indicMip_s \bar{p}_i \,\,
            - \indicMip_s \bar{p}_1 \bar{p}_2 - \indicMip_s \bar{p}_1 \bar{p}_3 - \indicMip_s \bar{p}_2 \bar{p}_3 \dots 
            + \indicMip_s \bar{p}_1 \bar{p}_2 \bar{p}_3 \dots 
            - \dots
            \pm \indicMip_s \prod^N_{i=1}  \bar{p}_i
        \\
        &- 
        \sum_{i=1}^N \indicMip_s \underline{p}_i \,\,
        \underbrace{
            + \indicMip_s \underline{p}_1 \underline{p}_2 + \indicMip_s \underline{p}_1 \underline{p}_3 + \indicMip_s \underline{p}_2 \underline{p}_3 \dots 
            - \indicMip_s \underline{p}_1 \underline{p}_2 \underline{p}_3 \dots 
            + \dots
            \pm \indicMip_s \prod^N_{i=1} \underline{p}_i
        }_{\text{Higher order terms.}}
    \end{align*}

    Where $\bar{p}_i = p_i + \ucB$, and $\underline{p}_i = p_i - \ucB$, and $\indicMip_s$ is the indicator variable for whether the specific scenario is possible, and equivalent in both the upper and lower bounding formulations. Ass. \ref{ass:bounding_condition} posits that the higher order terms of the Bonferroni expansion serve to cancel each other out, such that as $T \to \infty$, and $\ucB \to 0$, 
    \begin{align}
        \mathbf{z}^\top \mathbf{\bar{p}} - \mathbf{z}^\top \mathbf{\underline{p}} 
        \leq \sum_{i=1}^N \bar{p}_i - \sum_{i=1}^N \underline{p}_i.
    \end{align}
    This is true in most cases where the values of the expert competencies are spread out over $[0,1]$ and not highly concentrated around a single number, which is supported by Ass. \ref{ass:eps-increm}. It also tends to be true as we have a modest size of experts, as in our setting. It is possible to check, whether Ass. \ref{ass:bounding_condition} holds provided $\bar{p}_i$ and $\underline{p}_i$ as $\ucB$-based estimates input into the MIP. Our empirical analysis has demonstrated that this assumption holds in all tested configurations of expert competencies.


\subsection{Bound on Max-Min Probabilities for Weighted Majority Voting}

\begin{lemma} \label{lem:mip_pi_diff_bound}
    Under the Assumption \ref{ass:bounding_condition}, the regret of any learning algorithm (Def. \ref{def:regret}) which provides bounds on the individual agent competencies $p_i$ in a majority voting setting can be bounded by,
    \begin{align}
        \reg_t\leq \sum^N_{i=1} \bar{p}_i - \sum^N_{i=1} \underline{p}_i.
    \end{align}   
\end{lemma}





\begin{proof}

    Under full-bandit feedback, as we draw more samples the estimate of each individual expert's competencies improve, and the confidence bound shrinks, and with high confidence the estimate of each competency is bounded by, $p_i \pm \ucB$. Recall that the probability of each scenario $S \in \setS$ denoted by $p_S$ is given by,
    \[ 
        p_s = \prod_{i \in  S} p_i  \prod_{j\in S^c}  (1-p_j), \quad S\subseteq \setS,
    \]
    

    where $\setS$ contains $2^N$ elements. Let $\mathbf{z}^*$ be the binary vector that maximizes \(\mathbf{z}^\top \mathbf{p}\) for a given \(\mathbf{p}\). We further subject $\mathbf{z}$ to a constraint,
    
    \begin{align}
        \sum^{|\setS|}_s \indicMip_S \leq \frac{1}{2} |\setS|, \quad \forall \indicMip_S \in \mathbf{z} \label{eq:sc_half_constraint}
    \end{align}

    The constraint in Eq \eqref{eq:sc_half_constraint} represents the restriction on $\indicMip_S$, such that if $\indicMip_S = 1$ then its complement $\indicMip_{S^c} = 0$. Which, giving winning condition for tie-free weighted majority voting $\sum^N_{i=1} \theta_i > Q$, and $Q \geq N/2$, implies that the summation of $\indicMip_S$ over all scenarios cannot be greater than half of all scenarios $0.5|\setS|$, because its scenario complement, $S^c$, must be not be possible. This constraint is imposed because $\indicMip_S \in \{0, 1\}^{|\setS|}$ represents the possibility of the scenario in the set of all possible scenarios, and any possible scenario always excludes its complement, $S^c$, therefore $\indicMip_{S^c} = 0$, and the sum over $\setS$ is at most one half of $|\setS|$. Let $\mathbf{p}$ represent the probability vector of $2^N$ possible scenarios in $\setS$. Therefore, the regret is expressed as,
    
    \begin{align}
        \reg_t = \max_{\mathbf{z}} \mathbf{z}^\top \mathbf{p}^* - \mathbf{z}(\mathbf{p})^\top \mathbf{p}^* 
    \end{align}

    \textbf{Establishing the Inequality:} We establish the key inequality to bound our regret,
    \begin{align}
       \mathbf{z}(\mathbf{p})^\top \underline{\mathbf{p}} 
       \leq \mathbf{z}(\mathbf{p})^\top \mathbf{p}^* 
       \leq \max_{\mathbf{z}} \mathbf{z}^\top \mathbf{p}^* 
       \leq \mathbf{z}(\bar{\mathbf{p}})^\top \bar{\mathbf{p}}.
    \end{align}

    We consider the Bonferroni inequalities, via first order Bonferroni inequality,
    \begin{align}
        \max_{\mathbf{z}} \mathbf{z}^\top \mathbf{p}
        \leq P\left( \bigcup_{i=1}^{N} A_i \right) 
        \leq \sum^N_{i=1} p_i \label{eq:bonferroni_event_bound}
    \end{align}
    Where $A_i$ are defined as probability of success of expert $i$, for $N$ experts. The probability of the union of $A_i$, that is at least one experts predicts correctly, is always greater than or equal to the probability that at least $N/2$ experts predict correctly (majority condition). Therefore, for any vector $\mathbf{z}$ adhering to constraints from Eq. \eqref{eq:sc_half_constraint}, we can use the union bound over all expert competencies to upper bound the majority correct vote probability, leading to upper bounds, 
    \begin{align}
        \mathbf{z}^\top \bar{\mathbf{p}} 
        \leq \sum^N_{i=1} \bar{p}_i, 
        \qquad  \ \mathbf{z}^\top \underline{\mathbf{p}} 
        \leq \sum^N_{i=1} \underline{p}_i.
    \end{align}
    From the bound on Eq. \eqref{eq:bonferroni_event_bound}, combined with constraints on the sensitivity of the competency estimates as limited by Assumption \ref{ass:bounding_condition} we can obtain bounds on the simple regret of the algorithm under uncertainty of expert competencies,
    \begin{align}
        \reg_t\leq \sum^N_{i=1} \bar{p}_i - \sum^N_{i=1} \underline{p}_i.
    \end{align}    
    
    


\end{proof}

\subsection{Proof of Theorem \ref{thm:regret_theta_star_mip}} \label{thm:regret_theta_star_mip_proof}

\textbf{Regret Bound for $\theta$-WMV Algorithm:} An algorithm based on $\theta$-MIP weight assignment, such as in Algorithm \ref{alg:theta_wmv}, achieves a regret bound of $R_T \in \bigO \left( \sqrt{N T \log(T)} \right)$.


\begin{proof}
Provided the result from Lemma \ref{lem:mip_pi_diff_bound}, we simply  account for the uncertainty difference between the two estimated sums for $\sum_{n=1}^N p_i$ and $\sum_{n=1}^N \bar{p}_i$. 
    Let $S_\Sigma = \sum_{n=1}^N p_i$, denotes the true summed competencies, thus any deviation from $S_\Sigma$ via estimating $p_i$ will give the PAC bound of $\sum_{i=1}^N (\overline{p}_i - p_i)$. Let us express the uncertain quantity as $\ucB$, where, 

    \[
        \ucB_\Sigma = \left| \sum_{i=1}^N \hat{p}_i - S_\Sigma \right|.
    \]
    
    \paragraph{PAC Bounds for Summation of Binomial Trials:} 
    The $\ucB_\Sigma$ represents the radius that the estimate $\sum_{i=1}^N \hat{p}_i$ should fall within with probability $1-\delta$. Consider \( N \) independent binomial experiments, where the \( i \)-th experiment consists of \( n_i \) trials with success probability \( p_i \). Let \( \hat{p}_i = \frac{k_i}{n_i} \) denote the empirical probability of success, where \( k_i \) is the number of successes observed in the \( i \)-th experiment. We are interested in deriving a \( 1-\delta \) PAC bound for the sum of the true probabilities \( S_\Sigma = \sum_{i=1}^N p_i^* \). 
    
    \paragraph{Hoeffding's Inequality for Unequal  $n_i$:}
    Using Hoeffding's inequality, the deviation of the empirical sum \( \sum_{i=1}^N \hat{p}_i \) from the true sum \( S_\Sigma \) can be bounded. Let \( \{X_{i,j}\}_{j=1}^{n_i} \) be independent Bernoulli random variables with parameter \( p_i \in [0,1] \), for \( i = 1, \ldots, N \). That is, for each fixed \( i \),
    
    \[
    X_{i,j} \sim \text{Bernoulli}(p_i), \quad \text{i.i.d. for } j = 1, \ldots, n_i,
    \]
    and the samples are independent across different \(i\). Our goal is to establish a concentration inequality for the sum of these empirical means, $S_\Sigma$, around its expectation, $\mu$. Where,
    \begin{align}
        \hat{p}_i := \frac{1}{n_i} \sum_{j=1}^{n_i} X_{i,j}, \qquad S_\Sigma := \sum_{i=1}^N \hat{p}_i, \qquad \mu := \sum_{i=1}^N p_i.
    \end{align}
    Each \(\hat{p}_i\) is an average of \(n_i\) independent Bernoulli random variables bounded in \([0,1]\). Hence, \(\hat{p}_i\) itself lies in the interval \([0,1]\). We can write,
    \[
    \sum_{i=1}^N \hat{p}_i = \sum_{i=1}^N \frac{1}{n_i} \sum_{j=1}^{n_i} X_{i,j}.
    \]
    Consider the collection of all \(M := \sum_{i=1}^N n_i\) independent random variables \(\{X_{i,j}\}\). The sum \(S\) is a weighted sum of these variables, where each \(X_{i,j}\) is multiplied by \(1/n_i\). 

    \textbf{Hoeffding-Type Concentration Bound:} For any \(\epsilon > 0\), it holds that (see Lemma \ref{lem:scaling_hoeffding_1_over_n}), 
    \[
    \mathbb{P}\left( \left| \sum_{i=1}^N \hat{p}_i - \sum_{i=1}^N p_i \right| \geq \epsilon \right) \leq 2 \exp \left( - \frac{2 \epsilon^2}{\sum_{i=1}^N \frac{1}{n_i}} \right).
    \]
    To achieve a \( 1-\delta \) PAC bound, we set the right-hand side to \( \delta \) and solve for \( \epsilon \):
    \begin{align}
        \delta = 2 \exp\left( -\frac{2 \epsilon^2}{\sum_{i=1}^N \frac{1}{n_i}} \right), \qquad
        \epsilon = \sqrt{\frac{\ln(2/\delta)}{2} \cdot \sum_{i=1}^N \frac{1}{n_i}}.
    \end{align}
    Thus, with probability at least \( 1-\delta \):
    \begin{align}
        \left| \sum_{i=1}^N \hat{p}_i - S_\Sigma \right| \leq \sqrt{\frac{\ln(2/\delta)}{2} \cdot \sum_{i=1}^N \frac{1}{n_i}}.
    \end{align}
    The term \( \sum_{i=1}^N \frac{1}{n_i} \) captures the variability introduced by the unequal number of trials \( n_i \). Experiments with smaller \( n_i \) contribute more to the bound due to their higher uncertainty, while experiments with larger \( n_i \) contribute less due to their lower uncertainty.
    
    Suppose in the complete full bandit feedback, we always play every single arm,, therefore at each time interval, we would play each expert exactly $N$ times, therefore $n_i = t$, for all experts at round $t$. We can then use an integral to bound a summation as,
    \[
    \sum_{t=1}^{T} \sqrt{\frac{\ln(2/\delta)}{2} \cdot \sum_{i=1}^N \frac{1}{n_i}} = \sum_{t=1}^{T}\sqrt{\frac{\ln(2/\delta)}{2} \cdot \sum_{i=1}^N \frac{1}{t}} \leq  \sqrt{N\frac{\ln(2/\delta)}{2}} \int_{1}^{T} \sqrt{\frac{1}{t}} \, dt,
    \]
    
    To solve for the definite integral,
    \begin{align}
        \int_{1}^{T} \frac{1}{\sqrt{t}} \, dt &= \int_{1}^{T} t^{-1/2} \, dt \\
         \left[ 2t^{1/2} \right]_{1}^{T} &= 2T^{1/2} - 2(1)^{1/2} =  2\sqrt{T} - 2  
    \end{align}
    Putting this back into the original inequality,
    \begin{align}
        \sqrt{\frac{\ln(2/\delta)}{2} \cdot \sum_{i=1}^N \frac{1}{n_i}} \leq  \sqrt{\frac{\ln(2/\delta)}{2}} 2\sqrt{N} \left( \sqrt{T} - 1 \right) \label{eq:pi_ucb_uncertainty}
    \end{align}
    Eq. \eqref{eq:pi_ucb_uncertainty} bounds the high probability estimation uncertainty of the algorithm, where at each time step the regret can be bounded in the high probability regime. Therefore, we can express the expected regret from 1 to $T$ as,
    \begin{align}
        R_T \leq  \sqrt{2N\ln(2/\delta) } \left( \sqrt{T} - 1 \right)  + T \bar{\Delta} \delta \label{eq:pi_ucb_uncertainty}
    \end{align}
    We can choose $\delta = \frac{1}{T}$, allowing it to vanish. Therefore, the expected regret is bounded by,

    \begin{align}
        R_T \in \bigO \left( \sqrt{N T \log(T)} \right) \label{eq:pi_ucb_uncertainty}
    \end{align}
\end{proof}

\textbf{Worst-Case Bound:} It is possible to check, whether Ass. \ref{ass:bounding_condition} holds provided $\bar{p}_i$ and $\underline{p}_i$, and should this condition not hold, the algorithm will still retain its no-regret properties, however, it will incur a looser regret bound, illustrated in Appendix \ref{sec:alternative_mip_bound}.

\iftoggle{nips}{}
{
    \hlinE
    
    \paragraph{Tighter Bounds Using Bernstein's Inequality}
    For tighter bounds, especially when \( p_i^* \) is close to 0 or 1, Bernstein's inequality can be used:
    \begin{align}
    \mathbb{P}\left( \left| \sum_{i=1}^N \hat{p}_i - S \right| \geq \epsilon \right) \leq 2 \exp\left( -\frac{\epsilon^2}{2 \sum_{i=1}^N \frac{p_i^* (1 - p_i^*)}{n_i} + \frac{2 \epsilon}{3} \sum_{i=1}^N \frac{1}{n_i}} \right).
    \end{align}
    Solving for \( \epsilon \) to achieve a \( 1-\delta \) PAC bound yields:
    \begin{align}
    \epsilon \leq \sqrt{2 \ln(2/\delta) \cdot \sum_{i=1}^N \frac{p_i^* (1 - p_i^*)}{n_i}} + \frac{\ln(2/\delta)}{3} \cdot \sum_{i=1}^N \frac{1}{n_i}.
    \end{align}
    This bound is tighter than Hoeffding's inequality when the variances \( p_i^* (1 - p_i^*) \) are small.

    The derived PAC bounds explicitly account for the imbalance in the number of trials \( n_i \) through the term \( \sum_{i=1}^N \frac{1}{n_i} \). For tighter bounds, Bernstein's inequality can be used to incorporate the variance of each \( \hat{p}_i \). These results are particularly useful in settings where the number of trials varies significantly across experiments. 
    
    \hlinE
}

\subsection{Proof for Corollary \ref{cor:targeted_m_setting}} \label{prf:targeted_m_setting}

\textbf{Targeted-$m$ Setting:} An algorithm exists that allows for the regret bounds of Theorems \ref{thm:regret_bound_succex_log} and \ref{thm:regret_theta_star_mip} to hold with an additional multiplicative $N/m$ factor (e.g., \tT{SEE}'s regret becomes $\mathcal{O}\left( \frac{N^2}{m\minPEps^2} \log T \right)$).

\begin{proof}
    The algorithms presented in the previous sections were developed under the assumption that all experts could be queried at each round. However, this assumption may be restrictive in certain applications where only 
    $m$ experts are available at any  round. We refer to this setting as targeted-$m$, and in the following, we propose an extension of the previous algorithm to address this constraint
    
    A simple modification of the previous algorithm can be applied to the targeted-$m$ setting. Specifically, we partition the $N$ experts into $(N/m)$ groups of size $m$ each. Then, we run a burn-in period during which all experts within each group are queried for $t_0$ rounds, where $t_0$ is defined in Lemma \ref{lem:pac_sample_complexity_shrink}. 
    This results in the total $t_0\cdot(N/m)$ number of queries. 
    The purpose of the burn-in period is to ensure that any breakage event, if it occurs, does so with high probability. 
    Afterwards, we apply Algorithms \ref{alg:succEx} or \ref{alg:theta_wmv} for assembling the OEC or finding the optimal weights in the weighted majority voting problem, respectively. 
    Therefore, the regret bounds of Theorems \ref{thm:regret_bound_succex_log} and \ref{thm:regret_theta_star_mip} remain valid with an additional factor of $N/m$, e.g., the regret of \tT{SEE} in the Targeted-$m$ setting will be $\mathcal{O}( \frac{N^2}{m\minPEps^2} \log T )$.
\end{proof}

\subsection{Binary Outcome Expert Weighting} \label{sec:binary_weighted_voting}

\begin{lemma}
    \textbf{Optimal Majority Voting Expert Weights for Binary Classification:} From \citeA{nitzan:1982_binary_voting}, in a weighted majority voting scenario, as described in Sec. \ref{sec:weighted_maj_voting}, when the space of outcomes is binary (binary classification), the optimal weighting that maximizes the combined predictive accuracy for all experts combined is,
    \begin{align}
        \beta_i = \log \frac{p_i}{1-p_i}.
    \end{align}
    Where $p_i$ is the competency of expert $i$ (the probability that the expert predicts correctly.)
\end{lemma}

\begin{proof}

    We begin with a few key assumptions. First as stated, that the space of predictive outcomes is binary, the expert predicts true, they receive a value of $1$, and false a value of $0$. Let $y$ denote a set of outcomes (i.e. $y = [1, 0, 1, 1, 0]$ indicates that experts 1, 3, 5 selected the correct answer out of 5 experts). Due to the binary aspect of the outcomes, if an expert consistently predicts the incorrect outcome, we can always negate their prediction (predict the opposite) to obtain a equally consistently correct outcome. Further, let us define an aggregation function $f(y): y \mapsto \{ 0, 1 \}$ that takes any outcome and aggregates all individual outcomes to form a single outcome in $\{ 0, 1 \}$.
    
    \textbf{Pairing with Compliments:} Suppose we impose an arbitrary pairing over outcomes. Let $A(y)$ denote the selection over $y$ where we impose as the positive outcomes voters. Let $B(y)$ denote its compliment. Thus for any outcome $y = A(y) \cup B(y)$. This implies that if all outcomes are 1 in $A(y)$ the outcome of $f(y)$ will result in an overall positive outcome $y=1, \, \forall y \in A(y) \implies f(y) = 1$. The compliment, $B(y)$ is by definition mutually exclusive from $A(y)$ therefore it is not possible to include both in the selection voting criteria (either $A(y)$ or $B(y)$ for any outcome $y$.) We also can see that there is a deterministic relations that maps any selection $A(y)$ to its compliment $B(y)$ (i.e. if $A(y) = [1, 0, 1, 1, 0] \implies B(y) = [0, 1, 0, 0, 1]$.) Therefore, the set of all outcomes $\setY$ is paired in such a way. Let us denote this as $(A_k(y), B_k(y))$ for each pair.
    
    \textbf{Maximizing $\expeC[f(y)]$:} In order to maximize the probability that $f(y) = 1$, for each pair $(A_k(y), B_k(y))$ we must determine a selection rule that selects the greater of $P(A_k(y))$ or $P(B_k(y))$. For notation purposes, we impose that,
    \begin{align}
        P(A_k(y)) \geq P(B_k(y)), \quad \forall k \in \seT{K} \implies \expeC[f(y)] = \expeC[f^*(y)] \label{eq:selec_greater_than_compliment}
    \end{align} 
    Should a selection rule be found such that Eq. \eqref{eq:selec_greater_than_compliment} holds, then the value $\expeC[f(y)]$ is maximized. This simply holds because all outcomes are assigned to pairs, and events in each are mutually exclusive (or condition), and events between pairs are not (and condition). We can express this probability as, 
    \begin{align}
        P(A(y)) = \prod_{i \in A} p_i \prod_{j \in B} (1-p_j) \geq  \prod_{i \in A} (1-p_i) \prod_{j \in B} p_j = P(B(y)). \label{eq:compliment_prob_ineq}
    \end{align}
    This forms a sufficient condition, so long as the selection rule fulfills Eq. \eqref{eq:compliment_prob_ineq}.

    \textbf{Selection Rule:} It just so happens that our selection rule $f(y)$ selects $f(y) = 1$, then the weighted sum of outcomes in $A(y)$ must exceed that of $B(y)$ by definition. Thus,
    \begin{align}
        \sum_{i \in A} \beta_i y_i \geq \sum_{j \in B} \beta_j \texttt{Flip}(y_j) \label{eq:selection_rule_binary_weights}
    \end{align} 
    Where $\texttt{Flip}$ represents the \textit{bitwise not operator} (for example $[1,1,0] \to [0,0,1]$). This forms a necessary condition for the selection rule to output $f(y) = 1$. 
    
    \textbf{Log-Odds Weights:} Given the sufficient and necessary conditions to produce $\expeC[f^*(y)]$ outlined in Eq. \eqref{eq:compliment_prob_ineq} and Eq. \eqref{eq:selection_rule_binary_weights}. We can readjust Eq. \eqref{eq:compliment_prob_ineq} such that,
    \begin{align}
        \prod_{i \in A}\frac{ p_i}{(1-p_i)} \geq  \prod_{j \in B}  \frac{ p_j}{(1-p_j)}
    \end{align}

    Taking the logarithm, we obtain,
    \begin{align}
        \sum_{i \in A} \log \frac{p_i}{(1-p_i)} \geq  \sum_{j \in B}  \log \frac{ p_j}{(1-p_j)}
    \end{align}
    The optimal weights therefore are equivalent to,
    \begin{align}
        \beta_i = \log \frac{p_i}{(1-p_i)} \label{eq:opt_log_odds_weights}
    \end{align}
    Likewise for $j$ as well, mutatis mutandis. To summarize, the sufficient and necessary conditions from Eq. \eqref{eq:compliment_prob_ineq} and Eq. \eqref{eq:selection_rule_binary_weights} enforce a single unique solution for the optimality condition of $\expeC[f(y)]$ expressed in \eqref{eq:opt_log_odds_weights}.
    
    
    
    
    
    
    
    
\end{proof}

\subsection{Relation Between the Optimal Binary Classification and Weighted Majority Voting}\label{sec:disc}

The setting of \citeA{nitzan:1982_binary_voting} considers first, arbitrary neutral decision rules, whereas we restrict our analysis to a specific class of majority voting systems operating over sub-committees. Consequently, the binary voting framework of \citeA{nitzan:1982_binary_voting}   encompasses a strictly larger set of decision rules, which allows the approach to achieve superior performance only in the binary classification setting.

For example, consider a simple setting in which there are only two experts with competencies $\{p_1>p_2>0.5\}$. According to \citeA{nitzan:1982_binary_voting}, the optimal neutral decision rule is given by $f(x_1,x_2)= \operatorname{sign}(w_1 x_1 + w_2 x_2)$, where $w_i=\log(p_i/(1-p_i))>0$ and $x_i\in\{1,-1\}$ is the expert $i$'s decision. 
The success probability of this decision rule will be,

\begin{align*}
    P\left(f(x_1,x_2)=1\right)&= P\left(\operatorname{sign}(w_1 x_1 + w_2 x_2) = 1\right) = P\left(x_1=1, x_2=1\right) + P\left(x_1=1, x_2=-1\right)\\
    &=p_1p_2+p_1(1-p_2) = p_1.
\end{align*}

This performance can be achieved using the WMV scheme by considering  $\theta_1=1$ and $\theta_2=0$ and quota $Q = 0.5$  in \eqref{eq:weighted_combo_maj_vote}. In this case, \eqref{eq:weighted_combo_maj_vote} will be
\begin{align*}
    \pMaj(\setExpert, \theta) = p_1p_2 + p_1(1-p_2) = p_1.
\end{align*}

Now, suppose $\{p_1>0.5>p_2\}$ such that $w_1<|w_2|$. For example,  $p_1=0.6$ and $p_2=0.1$. In this case, $f(x_1,x_2)$ will always flip the decision of expert 2 and ignore expert 1 which leads to the success probability of $1-p_2=0.9$. But this performance cannot be achieved by the WMV. Because, if the optimal committee is only expert 1, then the success probability is $0.6$. If the optimal committee contains both of them, then the success is $0.6 \cdot 0.1+0.6 \cdot 0.9+0.4\cdot0.1=0.64$.

\subsection{Non-Contradiction of Expert Inclusion}


\begin{definition} \label{def:inclusion_sig}
    \textbf{Inclusion Signal:} The inclusion signal of the advantage function $\advFunc(\setExpert_1, \setExpert_2)$ is defined as,
    \begin{align*}
        \mathbbm{1}(\advFunc(\setExpert_1, \setExpert_2) > 0).
    \end{align*}
\end{definition}

Given $\setExpert_1$ and $\setExpert_2$ are disjoint subsets of $\setExpert$, $\mathbbm{1}(\advFunc(\setExpert_1, \setExpert_2) > 0)$ serves an indicator representing whether the advantage function is positive. This indicates if it is advantageous to retain the expert set $\setExpert_2$ in the original $\setExpert = \setExpert_1 \cup \setExpert_2$ or truncate it to $\setExpert_1$. 

\begin{definition} \label{def:conservative_exclusion}
    \textbf{Conservative Advantage Function:} Given two mutually exclusive subsets $\setExpert_1 \subseteq \setExpert$ and  $\setExpert_2 \subseteq \setExpert$, where $\setExpert_1 \cap \setExpert_1 = \emptyset$, for optimistic and pessimistic estimates expert competencies of $p_i$, denoted as $\bar{p}_i$ and $\underline{p}_i$ respectively, the conservative advantage function is defined as
    \begin{align*}
        \advFunc(\underline{\setExpert}_{1}, \overline{\setExpert}_{2}),
    \end{align*}
    such that,
    \begin{align*}
        \overline{\setExpert} = \{\bar{p}_1,...,\bar{p}_N\}, \qquad \underline{\setExpert} = \{\underline{p}_1,...,\underline{p}_N\}.
    \end{align*}
\end{definition}

That is, $\overline{\setExpert}$ represents the egalitarian voting committee substituting the optimistic estimates of the expert competencies $\bar{p}_i$, and $\underline{\setExpert}$ represents the  egalitarian voting committee substituting the pessimistic estimates of expert competencies, $\underline{p}_i$.


\begin{remark}
    \textbf{Guarantee on Consistent Exclusion:} Given a bisection $\biseC$, that divides two an ordering over a set of experts, $\setExpert$, into two ordered sets, $\setExpert_1$ and $\setExpert_2$, where $p_i > p_j$ for all $i \in \setExpert_1$ and $j \in \setExpert_2$, it follows that,
    \begin{align*}
        \advFunc(\underline{\setExpert}_1, \overline{\setExpert}_2) \leq 0 \implies \advFunc(\setExpert_1, \setExpert_2) \leq 0.
    \end{align*}
\end{remark}

\cm{We begin by asserting that any committee with dominant competencies for all experts (i.e. $N$ experts, all experts have higher competencies in one committee than another) will result in $\pMaj(\setExpert) \leq \pMaj(\bar{\setExpert})$, and it follows, $\pMaj(\setExpert_1, \setExpert_2) \leq \pMaj(\setExpert_1, \bar{\setExpert}_2)$. When we are provided a condition $\advFunc(\setExpert_1, \overline{\setExpert}_2) \leq 0$ this implies $\pMaj(\setExpert_1, \bar{\setExpert}_2) \leq \pMaj(\setExpert_1)$, by the definition from Eq. \eqref{eq:succ_add_func_def}. Therefore, $\pMaj(\setExpert_1, \setExpert_2) \leq \pMaj(\setExpert_1, \bar{\setExpert}_2) \leq \pMaj(\setExpert_1) \implies \pMaj(\setExpert_1, \setExpert_2) \leq \pMaj(\setExpert_1) \implies \advFunc(\setExpert_1, \setExpert_2) \leq 0$. We can conclude that, $\advFunc(\setExpert_1, \bar{\setExpert_2}) \leq 0 \implies \advFunc(\setExpert_1, \setExpert_2) \leq 0$.} Notice that during breakage $\breaK$, the bisection property $\biseC$ naturally divides $\setExpert$ into two sets at the barrier between $i$ and $j$ (see Diagram \ref{fig:breakage}), and we can be free to select $\underline{\setExpert}_1$ from within $\setExpert_1$ without contradicting the bisection property (i.e. $p_i > p_j$ for all $i \in \setExpert_1$ and $j \in \setExpert_2$). Therefore, it holds that,
\begin{align*}
    \advFunc(\underline{\setExpert}_1, \overline{\setExpert}_2) \leq 0 \implies \advFunc(\setExpert_1, \setExpert_2) \leq 0.
\end{align*}
In other words, if a conservative advantage function is applied, when this conservative advantage function, $\advFunc(\underline{\setExpert}_1, \overline{\setExpert}_2)$ signals to exclude $\setExpert_2$ from $\setExpert_1 \cup \setExpert_2$, then with high probability $1-\delta$, this exclusion of $\setExpert_2$ from the optimal committee is valid. 
\begin{condition} \label{cond:consistent_inclusion}
    \textbf{Consistent Inclusion:} Provided an egalitarian voting committee with experts $\setExpert$, and their respective competencies $p_1 \dots p_N$, where two mutually exclusive subsets $\setExpert_1, \setExpert_2$ are formed from the breakage event, $\breaK$, such that $p_i > p_j, \,\, \forall i \in \setExpert_1, \forall j \in \setExpert_2$ (see Diagram \ref{fig:breakage}), and sufficient samples are drawn, such that \textup{$2 \ucB < \minGapMaj$}, it follows that the inclusion signal (see Def. \ref{def:inclusion_sig}) of the conservative advantage function, \cm{$\advFunc(\underline{\setExpert}_1, \overline{\setExpert}_2)$, must be consistent with the inclusion signal computed under the true competencies,  $\advFunc(\setExpert_1, \setExpert_2)$, i.e.,}
    
    \begin{align*}
        \advFunc(\underline{\setExpert}_1, \overline{\setExpert}_2) > 0 \iff \advFunc(\setExpert_1, \setExpert_2) > 0.
    \end{align*}

    
\end{condition}

Note that the expert elimination condition checks the condition $\advFunc(\underline{\setExpert}_1, \overline{\setExpert}_2) \leq 0$ for any candidate set, where $p_i > p_j, \,\, \forall i \in \underline{\setExpert}_1, \forall j \in \bar{\setExpert}_2$ with high probability, $1-\delta$. It is therefore not possible to falsely eliminate $\setExpert_2$ from $\setExpert$ with high confidence, as we are stipulating the lowest estimates of competence for $\setExpert_1$ (group with higher competence) and highest estimates of competence for $\setExpert_2$ (group with lower competence) - if no advantage exists here, then we can confidently eliminate $\setExpert_2$ from $\setExpert$. In the reverse situation without additional constraints, the inclusion signal under the conservative estimate of $\advFunc(\underline{\setExpert}_1, \overline{\setExpert}_2)$ may be inconsistent with the inclusion signal under true competencies $\advFunc(\setExpert_1, \setExpert_2)$ (i.e. the signs of the two signals are different). Therefore, we next define $\minGapMaj$ as a property of the expert configuration such that consistency is achieved, per Condition \ref{cond:consistent_inclusion}.

\begin{definition} \label{def:minGapMaj}
    \textbf{\textup{$\minGapMaj$-Consistent Gap:}} We denote a consistency gap, $\minGapMaj \in [0, 1]$, as the maximum value such that Condition \ref{cond:consistent_inclusion} holds. 
\end{definition}

For any configuration of experts, there must exist a $\minGapMaj$ such that Condition \ref{cond:consistent_inclusion} holds, as there are no constraints on how small $\minGapMaj$ can be. When we draw enough samples such that $2 \ucB < \minGapMaj$, then we will obtain a consistent elimination function, when our confidence of expert competencies is less than the gap needed for always computing inclusion-consistent advantage functions with high confidence. Given the properties of the problem, i.e. expert competency arrangements $\{ p_i \}$, an exact computation can be performed to determine $\minGapMaj$, and the value of $\minGapMaj$ is independent of the number of the number of experts $N$ or number of samples $T$. \cm{Let $\{\biseC\}$ denote the set of bisections of $\setExpert$ (i.e. division of the set of experts into two sets post breakage, where $p_i > p_j$ for all $i \in \setExpert_1$ and $j \in \setExpert_2$). The range of admissible $\minGapMaj$ values can be expressed as,}
\begin{align}
    \mathcal{G}(\setExpert) := \left\{ \ucB \in \mathbb{R}^+ |\operatorname{Sign} \left( \advFunc(\underline{\setExpert}_1, \bar{\setExpert}_2) \right) = \operatorname{Sign} \left( \advFunc(\setExpert_1, \setExpert_2) \right) \right\}, \quad \forall (\setExpert_1, \setExpert_2) \in \{\biseC\}
\end{align}
We express $\minGapMaj$ as,
\begin{align}
    \minGapMaj = \argmax_{\ucB \in \mathcal{G}(\setExpert)} \, \epsilon.
\end{align}



    \subsection{Alternative Bound for Theorem \ref{thm:regret_theta_star_mip}} \label{sec:alternative_mip_bound}

    We provide an alternative upper bound for Theorem \ref{thm:regret_theta_star_mip} which upper bounds the regret of $\theta$-WMV algorithm. As preliminaries, we introduce the definition for $\mathcal{S}(\theta)$ and Lemma \ref{lemma:app_0} which stipulate that any two sets in $\mathcal{S}(\theta)$ cannot be disjoint.
    
    \begin{definition}
    For a given set of weights, we denote by $\mathcal{S}(\theta)$ the collection $\{S_1,...,S_k\}$ of all configurations of experts such that $\sum_{k\in S_i}\theta_k>Q$ for all $i$ and $S_i\not\subseteq S_j$ for all $i\neq j$.
    \end{definition}

    \begin{lemma}\label{lemma:app_0}
    Any two sets in $\mathcal{S}(\theta)$ have non-empty intersection. 
    \end{lemma}
    \begin{proof}
        Otherwise, we have at least $i$ and $j$ such that $S_i\cap S_j=\emptyset$. Then, $\sum_{k\in S_i\cup S_j}\theta_k>2Q$ while due to constraint $\|\theta\|\leq 2Q$. This is a contradiction. 
    \end{proof}
    
    \begin{theorem} \label{thm:regret_theta_star_mip_v2}
        \textbf{Regret Bound for $\theta$-WMV:} An algorithm based on $\theta$-MIP weight assignment, such as in Algorithm \ref{alg:theta_wmv}, achieves a regret bound of $R_T \in \bigO\left(2^{N}\sqrt{NT\log\big(NT\big)}\right)$. 
    \end{theorem}
    

    \begin{proof}
        Let $\hat{p}^t_i$ be the empirical estimate of $p_i$ after collecting $t$ i.i.d. samples, then due to Hoeffding's inequality, we have,
        \begin{align*}
        \mathbb{P}\left( \left|  \hat{p}^t_i - p_i \right| \geq \epsilon \right) \leq 2 \exp\left( -\frac{t \epsilon^2}{2\sigma_i^2} \right)\leq 2 \exp\left( -\frac{t \epsilon^2}{2\sigma^2} \right).
        \end{align*}
        where  $\sigma=\max_i \sigma_i$. 
        Recall that at each round, we collect samples from all experts. Thus, after $t$ rounds, we have collected precisely $t$ samples from each expert. 
        Using Union bound, we have
        \begin{align}\label{eq:p_1}
            \mathbb{P}\Big( \left|  \hat{p}^t_i - p_i \right| \leq \epsilon , \forall i\in\{1,...,N\}\Big) \geq 1 - 2N \exp\left( -\frac{t \epsilon^2}{2\sigma^2} \right).
        \end{align}
        Using the inclusion-exclusion principle, and the definition of $\mathcal{S}(\theta)=\{S_1,...,S_k\}$, we have
        \begin{align*}
            \pMaj(\setExpert, \theta)= \sum_{i_1=1}^k\prod_{j\in S_{i_1}}p_{j} - \sum_{i_1\neq i_2=1}^k\prod_{j\in S_{i_1}\cup S_{i_2}}p_{j}+...
        \end{align*}
        When $|p_i-\hat{p}^t_i|\leq \epsilon$ for all experts, i.e.,   $i\in[N]$, for any given $\theta$, we obtain,
        \begin{align*}
        &\left|\pMaj(\setExpert, \theta) - \widehat{\pMaj}^t(\setExpert, \theta)\right|=\\
        &\left|\Big(\sum_{i_1=1}^k\prod_{j\in S_{i_1}}p_{j} - \sum_{i_1\neq i_2=1}^k\prod_{j\in S_{i_1}\cup S_{i_2}}p_{j}+...\Big)-\Big(\sum_{i_1=1}^k\prod_{j\in S_{i_1}}\hat{p}^t_{j} - \sum_{i_1\neq i_2=1}^k\prod_{j\in S_{i_1}\cup S_{i_2}}\hat{p}^t_{j}+...\Big)\right|\\
        &\leq \Big(\sum_{i=1}^k\sum_{j\in S_i}p_{j}\Big)\epsilon\leq  \binom{N-1}{\lceil \frac{N-1}{2}\rceil} \Big(\sum_{j\in S_1\cup...\cup S_k} p_j\Big)\epsilon,
        \end{align*}

        where $\lceil a \rceil$ denotes the ceiling of real number $a$.  As shown in Remark \ref{rem:app} for a simple scenario, the first inequality is tight. 
        The last inequality is due to Lemma \ref{lemma:app_0} which implies that the maximum size of $\mathcal{S}(\theta)$ is $\binom{N}{\lceil\frac{N+1}{2}\rceil]}$ and consequently, we have at most $\binom{N-1}{\lceil\frac{N-1}{2}\rceil]}$ repetitions of each $p_j$ in the above double summation. 
        Knowing that $\sum_{j\in S_1\cup...\cup S_k} p_j\leq N$, leads to the bound $\binom{N-1}{\lceil\frac{N-1}{2}\rceil}N\epsilon$. 
        Applying the Stirling's approximation, we can achieve the following,
        \begin{align*}
        &\left|\pMaj(\setExpert, \theta) - \widehat{\pMaj}^t(\setExpert, \theta)\right|\leq \frac{2^N}{\sqrt{N}}N\epsilon=2^N\sqrt{N}\epsilon.
        \end{align*}

        
        Applying the concentration result in Eq. \eqref{eq:p_1}, we obtain,
        \begin{align*}
           \mathbb{P}\left(\left|\pMaj(\setExpert, \theta) - \widehat{\pMaj}^t(\setExpert, \theta)\right|\leq 2^N\sqrt{N}\epsilon\right)\geq \mathbb{P}\Big( \left|  \hat{p}^t_i - p_i \right| \leq \epsilon , \forall i\in[N]\Big) \geq 1 - 2N \exp\left( -\frac{t \epsilon^2}{2\sigma^2} \right). 
        \end{align*}
        Now, requiring the probability to be at least $1-\delta$, we obtain,
        \begin{align}\label{eq:con2}
            \mathbb{P}\left(\left|\pMaj(\setExpert, \theta) - \widehat{\pMaj}^t(\setExpert, \theta)\right|\leq \epsilon_0\right)  \geq 1 - \delta. 
        \end{align}
        where,
        \begin{align*}
            \epsilon_0(t):=\sqrt{\frac{2N4^{N}\sigma^2}{t}\log\big(\frac{2N}{\delta}\big)}.
        \end{align*}
        
        Let us consider two possible cases at round $t$; (i) $\pMaj(\setExpert, \theta^*)-\pMaj(\setExpert, \theta^t)> 2\epsilon_0(t)$ and (ii) $\pMaj(\setExpert, \theta^*)-\pMaj(\setExpert, \theta^t)\leq 2\epsilon_0(t)$. 
        Next, 
        we show that case $i$ won't happen with high probability. To this end, using the concentration in \eqref{eq:con2}, we imply that the following events hold with probability at least $1-2\delta$.
        \begin{align*}
            &\left|\pMaj(\setExpert, \theta^*) - \widehat{\pMaj}^t(\setExpert, \theta^*)\right|\leq \epsilon_0(t),\\
            &\left|\pMaj(\setExpert, \theta^t) - \widehat{\pMaj}^t(\setExpert, \theta^t)\right|\leq \epsilon_0(t).
        \end{align*}
        Combining the above inequalities with the assumption of case (i), i.e., $1\geq\pMaj(\setExpert, \theta^*)-\pMaj(\setExpert, \theta^t)> 2\epsilon_0(t)$, implies $\widehat{\pMaj}^t(\setExpert, \theta^t)< \widehat{\pMaj}^t(\setExpert, \theta^*)$. This is a contradiction with $\theta^t\in\arg\max_\theta\widehat{\pMaj}^t(\setExpert, \theta)$.
        Thus, with probability $1-2\delta$, case (ii) occurs. By letting $\delta=1/(2T)$, we could bound the regret as follows,
        \begin{align*}
            R_T&= \sum_{t=1}^T\Big( \pMaj(\setExpert, \theta^*)-\pMaj(\setExpert, \theta^t)\Big)\leq\sum_{t=1}^{t_0} 1+\sum_{t=t_0+1}^T \Big(\pMaj(\setExpert, \theta^*)-\pMaj(\setExpert, \theta^t)\Big)\\
            &\leq t_0+ \frac{1}{T}(T-t_0)+(1-\frac{1}{T})\sum_{t=t_0+1}^T 2\epsilon_0(t) \\
            &\leq t_0+1 + (1-\frac{1}{T})\sqrt{8N4^{N}\sigma^2\log\big(4TN\big)}\int_{t_0}^T\frac{1}{\sqrt{t}}dt\in\bigO\left(2^{N}\sqrt{NT\log\big(NT\big)}\right).
        \end{align*}
        
        where $t_0:=2N4^{N}\sigma^2\log\big(4TN\big)$. Note that during the initial phase, when $t\leq t_0$, we bound the instantaneous regrets by their worst scenario that is 1 as the utility function is a probability.
    \end{proof}

    \begin{remark}\label{rem:app}
    Let consider $N=3$ experts with competencies $p_1,p_2,p_3$ and weights $\theta=(1,1,1)$. In this case, $\mathcal{S}(\theta)=\{\{1,2\}, \{1,3\}, \{2,3\}\}$. Moreover, let $\hat{p}_i=p_i+\epsilon$ for all $i$, where $\epsilon\leq\min_i p_i$ then 
    \begin{align*}
        &|\pMaj(\setExpert, \theta)-\widehat{\pMaj}(\setExpert, \theta)|=|p_1p_2-\hat{p}_1\hat{p}_2+p_1p_3-\hat{p}_1\hat{p}_3+p_2p_3-\hat{p}_2\hat{p}_3-2(p_1p_2p_3-\hat{p}_1\hat{p}_2\hat{p}_3)|\\
        &=|(p_1+p_2)\epsilon+(p_1+p_3)\epsilon+(p_2+p_3)\epsilon+3\epsilon^2-2(p_1p_2+p_1p_3+p_2p_3)\epsilon-2(p_1+p_2+p_3)\epsilon^2-2\epsilon^3|\\
        &=|2\big(p_1+p_2+p_3-(p_1p_2+p_1p_3+p_2p_3)\big)\epsilon-\big(2(p_1+p_2+p_3)-3\big)\epsilon^2-2\epsilon^3|.
    \end{align*}
    The introduced bound is $\binom{N-1}{\lceil\frac{N-1}{2}\rceil]} \Big(\sum_{j\in S_1\cup S_2\cup S_3} p_j\Big)\epsilon= 2(p_1+p_2+p_3)\epsilon$. 
    Now, lets compare the exact value and the bound for when $p=(.02,.03,.04)$. In this case, 
    \begin{align*}
        \binom{N-1}{\lceil\frac{N-1}{2}\rceil]} \Big(\sum_{j\in S_1\cup S_2\cup S_3} p_j\Big)\epsilon-|\pMaj(\setExpert, \theta)-\widehat{\pMaj}(\setExpert, \theta)|=0.18\epsilon-|0.1748\epsilon+2.82\epsilon^2-2\epsilon^3|
    \end{align*}
    Note that the above expression can be arbitrary small as $p_i$s tends to zero.
    \end{remark}

\iftoggle{nips}{}{
    
    \hlinE
    
        \lar{we can also address the scenario with $2^{N-1}$ scenarios. Where for each scenario,}
    
        \begin{align}
            \bar{p}_s - \underline{p}_s 
            = \prod_{i \in  S} \bar{p}_i  \prod_{j\in S^c}  (1-\bar{p}_j) 
            - \prod_{i \in  S} \underline{p}_i  \prod_{j\in S^c}  (1-\underline{p}_j) 
            \leq 2 N \, \ucB \label{eq:worst_case_pi_bound_product}
        \end{align}
    
        The result of Eq. \eqref{eq:worst_case_pi_bound_product} can be obtained by performing a setting $\bar{p}_i = p_i + \ucB$ and $\underline{p}_i = p_i - \ucB$, and setting any higher order $\ucB$ terms to a multiple of $\ucB$. Therefore the maximum loss would be over $2^{N-1}$ selection events, where due to specification, the absolute worst case outcome becomes,
    
        \begin{align}
            \mathbf{z}^\top \mathbf{\bar{p}} - \mathbf{z}^\top \mathbf{\underline{p}} 
            \leq N 2^N \, \ucB
        \end{align}
    
        \lar{Although here we consider a highly conservative upper-bound, and in most scenarios, when expert competencies are evenly spread out, i.e. $\minPEps$ from Ass. \ref{ass:eps-increm} is sufficiently wide, or when $N$ is large, then Ass. \ref{ass:bounding_condition} will hold.}   
    }
    
\iftoggle{nips}{}{
    \ja{Next lemmas are useful to establish a tighter bound, i.e., $\mathcal{O}\left(e^{c(\minPEps)N}N\log\big(NT\big)\right)$ which for now, we should not include in this version.  }
    
    \begin{lemma}\label{lemma:a_0}
        Under Assumption \ref{ass:eps-increm}, function $\pMaj(\setExpert, \theta)$ is piecewise continuous with jumps (magnitude of discontinuity) bounded by $\exp\big(-N\cdot c(\minPEps)\big)$, where $N$ is number of experts and $c(\minPEps)\geq0$ depending on $\minPEps$.
    \end{lemma}
    \begin{proof}
        According to the MIP formulation,  we have $\pMaj(\setExpert, \theta)=\mathbb{Z}^\top \mathbf{P}$, where $\mathbf{P}=[\prod_{i \in  S} p_i  \prod_{j\in S^c}  (1-p_j): S\subseteq\setExpert]$. Thus, the minimum jump in   $\pMaj(\setExpert, \theta)$ is at least $\min_{S\subseteq\setExpert}\max\{0,\prod_{i \in  S} p_i  \prod_{j\in S^c}  (1-p_j)\}$.
        Without loss of generality, suppose that the competencies of the experts are ordered as $p_1<p_2<...<p_N$. 
        Moreover, let $k$ to be the first expert for which $p_k\geq0.5$, i.e., $p_1<p_2<...<p_{k-1}<0.5\leq p_k<...<p_N$. 
        Hence, under the Assumption \ref{ass:eps-increm} and assuming that $N\minPEps\leq 1$, we can write
        \begin{align*}
            \min_{S\subseteq \setExpert}\max\{0,\prod_{i \in  S} p_i  \prod_{j\in S^c}  (1-p_j)\}\geq \prod_{i=1}^{k-1} p_i  \prod_{j=k}^{N}  (1-p_j)\geq (k-1)!(N-k+1)!\minPEps^N
        \end{align*}
        Note that Binomial coefficients $\binom{N}{k-1}$ is maximized at $k-1=[N/2]$, which means that  $(k-1)!(N-k+1)!\geq([N/2]!)^2$. Substituting this result in the above inequality yields
        \begin{align*}
            &\min_{S\subseteq \setExpert}\max\{0,\prod_{i \in  S} p_i  \prod_{j\in S^c}  (1-p_j)\} \geq ([N/2]!)^2\minPEps^N\\
            &\qquad\geq \Big(\big(\frac{[N/2]}{e}\big)^{[N/2]}\Big)^2\minPEps^N\approx \Big(\frac{N}{2e}\minPEps\Big)^N =\exp\left(-N\log\left(\frac{2e}{N\minPEps}\right)\right).
        \end{align*}
        Note that $N\minPEps\leq 1$ and $\log(x)\leq x-1$ for $x\geq1$. This implies 
        \begin{align*}
            &\min_{S\subseteq \setExpert}\max\{0,\prod_{i \in  S} p_i  \prod_{j\in S^c}  (1-p_j)\} \geq \exp\left(-N\left(\frac{2e}{N\minPEps}-1\right)\right)\geq \exp(-Nc(\minPEps)).
        \end{align*}
        The lats inequality is by assuming that $1/s\leq N\minPEps\leq 1$ for some constant $s$. In this case, we have $c(\minPEps)=2se -1$. 
    \end{proof}

    \begin{lemma}\label{lemma:a_1}
        If $|\pMaj(\setExpert, \theta)-\widehat{\pMaj}^t(\setExpert, \theta)|<0.5\exp\big(-Nc(\minPEps)\big)$, for all $\theta$, then $\theta^*:=\arg\max_\theta \pMaj(\setExpert, \theta)=\arg\max_\theta \widehat{\pMaj}^t(\setExpert, \theta):=\theta^t$.
    \end{lemma}
    \begin{proof}
        If $\theta^*\neq \theta^t$, then using the result of Lemma \ref{lemma:a_0}, we should have $\pMaj(\setExpert, \theta^*)- \pMaj(\setExpert, \theta_t)\geq \exp\big(-Nc(\minPEps)\big)$. 
    On the other hand, from the assumption of this Lemma, we have
    \begin{align*}
        & |\pMaj(\setExpert, \theta^*)-\widehat{\pMaj}^t(\setExpert, \theta^*)|<0.5\exp\big(-Nc(\minPEps)\big),\\
        &|\pMaj(\setExpert, \theta^t)-\widehat{\pMaj}^t(\setExpert, \theta^t)|<0.5\exp\big(-Nc(\minPEps)\big).
    \end{align*}
    Therefore, we have
    \begin{align*}
        \widehat{\pMaj}^t(\setExpert, \theta^t)\leq \widehat{\pMaj}^t(\setExpert, \theta^*).  
    \end{align*}
    Recall that $\theta^t\in\arg\max_\theta \widehat{\pMaj}^t(\setExpert, \theta)$. 
    The above inequality contradicts with $\theta^*\neq \theta^t$, as it implies that $\theta^*\in\arg\max_\theta \widehat{\pMaj}^t(\setExpert, \theta)$.
    \end{proof}

}

\subsection{Scaling Variance of $1/n_i$ for Bernoulli Variables}

\begin{lemma} \label{lem:scaling_hoeffding_1_over_n}
    For estimates sums of Bernoulli variables $X_i \in \{0,1\}$, where $p_i$ is the probability of success for variable $X_i$, $\hat{p}_i = \frac{k_i}{n_i}$ is the empirical estimate, and $n_i$ is the empirical count of successes, it holds that,
    
    \[
    \mathbb{P}\left( \left| \sum_{i=1}^N \hat{p}_i - \sum_{i=1}^N p_i \right| \geq \epsilon \right) \leq 2 \exp \left( - \frac{2 \epsilon^2}{\sum_{i=1}^N \frac{1}{n_i}} \right).
    \]
\end{lemma}


\begin{proof}

\textbf{Scaling $1/n_i$:} Since each \(X_{i,j} \in [0,1]\), the range of each scaled variable \(\frac{X_{i,j}}{n_i}\) is contained in \( \{0, 1/n_i \}\). Each original Bernoulli random variable \(X_{i,j}\) takes values in \(\{0,1\}\), thus,
    \[
    X_{i,j} \in [0,1].
    \]
    When forming the empirical mean,
    \[
    \hat{p}_i = \frac{1}{n_i} \sum_{j=1}^{n_i} X_{i,j},
    \]
    each summand inside the sum is scaled by \(1/n_i\). Therefore,
    \[
    \frac{X_{i,j}}{n_i} \in \left\{0, \frac{1}{n_i}\right\} \subseteq [0, 1/n_i].
    \]
    This implies the range of each scaled variable is reduced from length \(1\) to length \(1/n_i\). This scaling is allows us to applying Hoeffding's inequality, as the concentration bound depends on the sum of the squared lengths of these ranges,
    
    \[
    \sum_{i=1}^N \sum_{j=1}^{n_i} \left(\frac{1}{n_i} - 0\right)^2 = \sum_{i=1}^N n_i \cdot \frac{1}{n_i^2} = \sum_{i=1}^N \frac{1}{n_i}.
    \]
    Hence, the variance proxy in the concentration bound is governed by \(\sum_i \frac{1}{n_i}\).
    
    Applying Hoeffding's inequality for sums of independent bounded variables, for any \(\epsilon > 0\):
    \[
    \mathbb{P}\left( \left| S_\Sigma - \mu \right| \geq \epsilon \right) 
    \leq 2 \exp \left( - \frac{2 \epsilon^2}{\sum_{i=1}^N n_i \left( \frac{1}{n_i} - 0 \right)^2} \right).
    \]
    
    Simplifying the denominator,
    \[
    \sum_{i=1}^N n_i \left( \frac{1}{n_i} \right)^2 = \sum_{i=1}^N \frac{1}{n_i}.
    \]
    
    Hence,
    \[
    \mathbb{P}\left( \left| \sum_{i=1}^N \hat{p}_i - \sum_{i=1}^N p_i \right| \geq \epsilon \right) \leq 2 \exp \left( - \frac{2 \epsilon^2}{\sum_{i=1}^N \frac{1}{n_i}} \right).
    \]

\end{proof}

\clearpage

\section{Algorithms}

\subsection{Greedy Algorithm for Constructing Optimal Egalitarian Committee} 

\begin{algorithm} 
\small
\caption{Greedy Algorithm - Construct Optimal Committee with Accuracy $\pMaj$}
\label{alg:greedy_oec}
\begin{algorithmic}[1]
\Require Set of experts $\setExpert$ and $\setExpertCand = \emptyset$
\Ensure Optimal $\pMaj$ and $\setOec$
\State Sort all experts by their accuracy measure $p_i$, resulting in $\tT{sorted}(\{p_i\})$.
\State $\setExpertCand \gets e_1$.
\For{$j = 2$ to $N$}
    \If{ $\advFunc(\setExpertCand, \setExpert') \geq 0, \quad \forall \setExpert' \in \tT{TopK}(\setExpertCand^{\leq j})$} \Comment{Check the advantage function.}
        \State $\setExpertCand \gets \setExpert \cup \setExpert'$. 
        \State $\pMaj \gets \pMaj(\setExpertCand)$
    \EndIf
\EndFor
\State $\setOec \gets \setExpertCand$
\State \Return $\pMaj, \setOec$
\end{algorithmic}
\end{algorithm}


\subsection{Zooming Algorithm}

\begin{algorithm}[H]
\small
\caption{Zooming Algorithm for Lipschitz Bandits}
\label{alg:zooming}
\begin{algorithmic}[1]
\State \textbf{Input:} Metric space \( (X, d) \), Lipschitz constant \( L > 0 \), time horizon \( T \)
\State \textbf{Initialization:} Set active arms \( \text{Active} = \emptyset \)
\For{\( t = 1, 2, \dots, T \)}
    \State \textbf{Update Confidence Intervals:}
    \For{\( x \in \text{Active} \)}
        \State Compute empirical mean reward:
        \[
        \hat{\mu}_t(x) = \frac{1}{n_t(x)} \sum_{s=1}^{n_t(x)} r_s(x)
        \]
        \State Compute confidence radius:
        \[
        r_t(x) = \sqrt{\frac{2 \log T}{n_t(x)}}
        \]
    \EndFor
    \State \textbf{Activate New Arms:}
    \For{\( x \in X \setminus \text{Active} \)}
        \If{\( \min_{y \in \text{Active}} d(x, y) > r_t(y) \)}
            \State Add \( x \) to \( \text{Active} \)
        \EndIf
    \EndFor
    \State \textbf{Select Arm to Play:}
    \[
    x_t = \arg\max_{x \in \text{Active}} \left( \hat{\mu}_t(x) + r_t(x) \right)
    \]
    \State \textbf{Play Arm \( x_t \):}
    \State Observe reward \( r_t \) and update \( n_t(x_t) \)
\EndFor
\end{algorithmic}
\end{algorithm}

The zooming algorithm from \citeA{kleinberg:2008_zooming} adaptively explores and exploits strategies in a metric space by maintaining confidence intervals around the estimated rewards of active arms, serving as an appropriate benchmark bandit algorithm for when the reward function lacks convexity. It dynamically activates new arms that are sufficiently far from existing ones (based on their confidence radii) and prioritizes arms with high upper confidence bounds. In non-convex spaces, the algorithm leverages the Lipschitz condition to generalize observations across nearby arms, ensuring efficient exploration while adapting to the structure of the problem instance. 


\clearpage

\section{Experiment Details} \label{sec:experimental_details_app}

\subsection{LLM Offline Performance Benchmarking (i.e. Expert Competencies)} \label{sec:llm_competency_app}

\begin{figure}[h]
    \centering
    \includegraphics[width=0.8\linewidth]{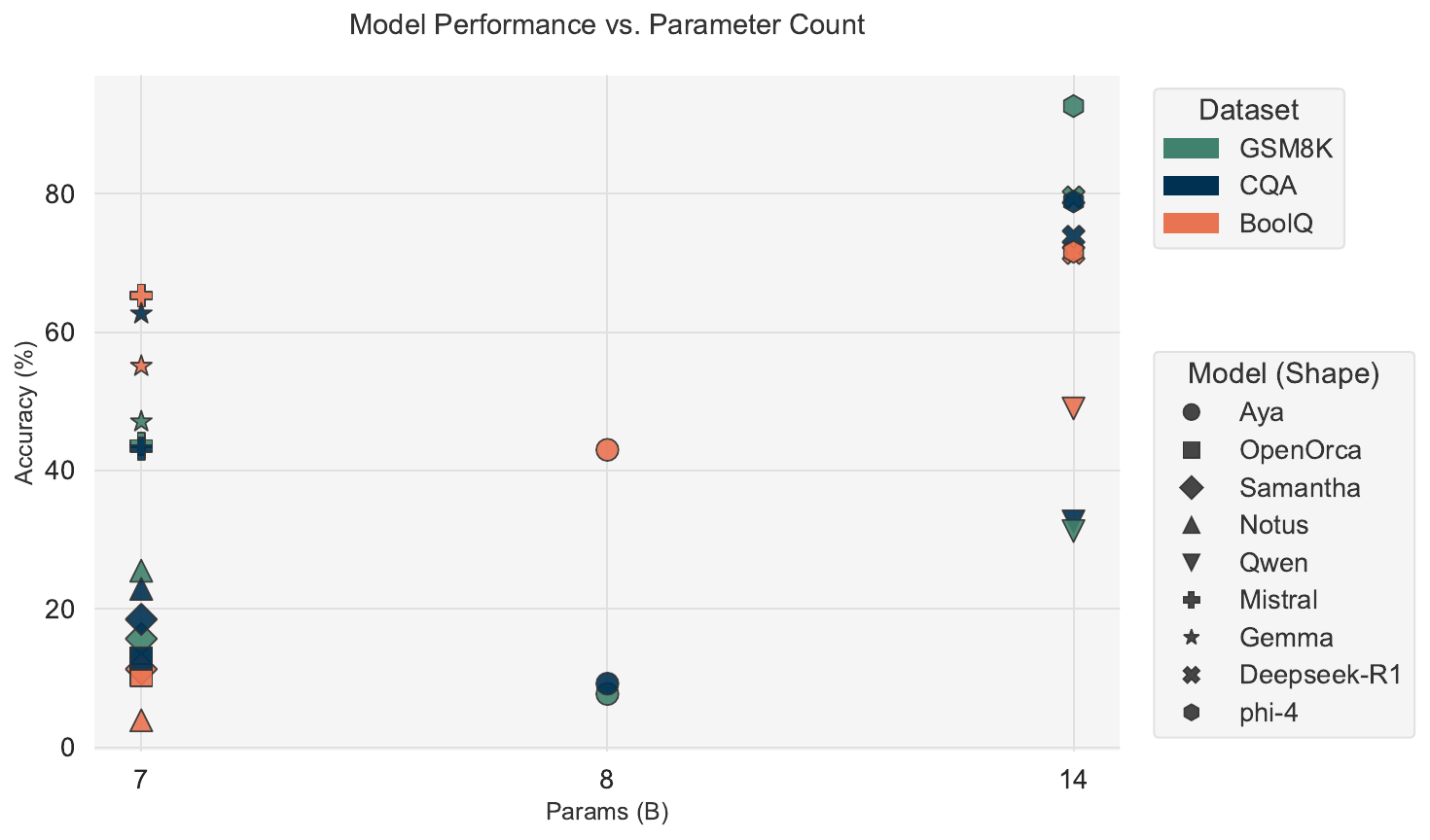}
    \caption{Visualization of the performance comparison of language models across three benchmarks (GSM8K, CQA, BoolQ) as a function of parameter count (7B, 8B, 14B).}
    \label{fig:enter-label}
\end{figure}


\begin{table}[h!]
    \centering
    \small
    \begin{tabular}{lrrrrr}
    \toprule
    Model & GSM8K Acc. (\%) & CQA Acc. (\%) & BoolQ Acc. (\%) & Params. (B) \\
    \midrule
    Aya & 7.71 & 9.20 & 43.00 & 8 \\
    OpenOrca & 12.90 & 12.80 & 10.40 & 7 \\
    Samantha & 15.69 & 18.50 & 11.30 & 7 \\
    Notus & 25.53 & 22.90 & 3.91 & 7 \\
    Qwen & 31.25 & 32.66 & 48.98 & 14 \\
    Mistral & 43.88 & 43.20 & 65.31 & 7 \\
    Gemma & 47.07 & 62.68 & 55.10 & 7 \\
    Deepseek-R1 & 79.52 & 73.83 & 71.43 & 14 \\
    phi-4 & 92.69 & 78.90 & 71.63 & 14 \\
    \bottomrule
    \end{tabular}
    \vspace{0.8em}
    \caption{\textbf{Model performance ranking.} Models listed by their accuracy across various domains in the offline setting, alongside their respective parameter count (billions).}
    \label{tab:model_ranking}
\end{table}

\subsection{Description of Datasets for Task Generation in Online Simulation}

\renewcommand{\arraystretch}{1.2}
\begin{table}[H]
\renewcommand{\arraystretch}{1.2}
\centering
\small

\begin{tabular}{@{}p{2.3cm}p{5.6cm}p{5.6cm}@{}}
\toprule
\textbf{Dataset} & \textbf{Description} & \textbf{Testing Capabilities and Limitations} \\
\midrule
GSM8K  & Collection of 8.5k middle school-level math problems from \citeA{cobbe2021gsm8k} solvable by elementary operations and logical reasoning. & The collections of questions are relatively basic and relevant to mathematics. The expert is allowed to propose any answer they would like (in this case limited to integers). \\
\addlinespace[0.1cm]
CommonsenseQA  & Open-source from benchmark dataset from \citeA{talmor-etal-2019-commonsenseqa} containing 9.26k questions with common sense yes/no questions. & This dataset benchmarks the reasoning capabilities of LLMs, and provides a catalogue of candidate answers to choose from.\\
\addlinespace[0.1cm]
BoolQ & A question-answering dataset from \citeA{clark2019boolq} that focuses on logical reasoning and evaluation of boolean expressions. It is part of the BIG-Bench Hard (BBH) suite. & This dataset examines the logical reasoning capabilities of LLMs and provides a catalogue of possible answers \textit{true} and \textit{false}. \\
\addlinespace[0.1cm]
\bottomrule
\end{tabular}

\vspace{0.3em}
\caption{Overview of selected evaluation datasets.}
\label{tab:eval-datasets}
\end{table}

\subsection{Computing Specifications}

\begin{table}[h!]
    \centering
    \small
\label{tab:specs}
\begin{tabular}{ll}
\toprule
\textbf{Component} & \textbf{Specification} \\
\midrule
CPU & Intel\textsuperscript{\textregistered} Core\texttrademark i9-9900K @ 3.60\,GHz \\
Cores/Threads & 16 \\
GPU & NVIDIA GeForce RTX 2080 \\
VRAM & 8\,GB GDDR6 \\
CUDA Version & 11.8 \\
System Memory & 64\,GB DDR4 \\
Storage & NVMe SSD \\
Python Version & 3.8.10 \\
Gurobi Optimizer & 9.5.2 \\
\bottomrule
\end{tabular}
\vspace{0.3em}
\caption{Computing hardware and software library specifications.}

\end{table}

\clearpage

\subsection{Chain-of-Thought Prompting with Constrained Output Formatting}

\paragraph{Structured Chain-of-Thought Prompting:}
Our approach implements a constrained variant of chain-of-thought (CoT) prompting \cite{wei2022chain} with strict output formatting rules. We provide a prompt template for each of the domains.

\subsubsection{GSM8K Prompt Template}

\begin{verbatim}
### Instructions:
1. Read the question carefully and identify what is being asked.
2. Solve the problem methodically, showing each step clearly.
3. Double-check your calculations before finalizing the answer.
4. Your final output MUST follow EXACTLY this format:

### Reasoning:
[Your step-by-step reasoning here]

### Final Answer: [Numerical Value]

### Required Output Format Rules:
- Only numbers allowed in final answer (e.g., 42, 3.14, 2/3)
- If uncertain, return ### Final Answer: 0
- No additional text after final answer
- Final answer must be the last line

### Question:
{question}

### Choices:
{choices}

### Reasoning:
\end{verbatim}

\subsubsection{CommonsenseQA Prompt Template}

\begin{verbatim}
### Instructions:
1. Read the question carefully and identify what is being asked.
2. Solve the problem methodically, showing each step clearly.
3. Double-check your calculations before finalizing the answer.
4. Your final output MUST follow EXACTLY this format:

### Reasoning:
[Your step-by-step reasoning here]

### Final Answer: One of 5 catagories [A, B, C, D, E]

### Required Output Format Rules:
- Only numbers are allowed in the final answer (e.g., A, B, C, D, E)
- If you cannot determine the answer, you MUST pick a random answer from [A, B, C, D, E].
- No additional text, explanations, or characters after the final answer.
- The final answer line must be the very last line of your response.

### Question:
{question}

### Choices:
{choices}

### Reasoning:
\end{verbatim}

\subsubsection{BoolQ Prompt Template}

\begin{verbatim}
### Instructions:
1. Read the question carefully and identify what is being asked.
2. Solve the problem methodically, showing each step clearly.
3. Double-check your calculations before finalizing the answer.
4. Your final output MUST follow EXACTLY this format:

### Reasoning:
[Your step-by-step reasoning here]

### Final Answer: One of 2 catagories [true, false]

### Required Output Format Rules:
- Only numbers are allowed in the final answer (e.g., true, false)
- If you cannot determine the answer, you MUST pick a random answer from [true, false].
- No additional text, explanations, or characters after the final answer.
- The final answer line must be the very last line of your response.

### Question:
{question}

### Reasoning:
\end{verbatim}




\subsubsection{Output Parsing Algorithm}

The response parsing algorithm enforces strict numerical output constraints through regular expression matching:

\begin{algorithm}[H]
\caption{Response Parser}
\begin{algorithmic}[1]
\Procedure{ParseResponse}{$R$}
    \State $\textit{pattern} \gets \text{``\textbackslash\textbackslash\textbackslash\textbackslash Final Answer:\textbackslash s*(-?\textbackslash d+\textbackslash .?\textbackslash d*|[-+]?\textbackslash d+/\textbackslash d+)''}$
    \State $\textit{match} \gets \text{regex\_search}(R, \textit{pattern}, \text{IGNORECASE})$
    \If{$\textit{match.success}$}
        \State \textbf{return} $\textit{match.group}(1)$
    \Else
        \State \textbf{return} ``0''
    \EndIf
\EndProcedure
\end{algorithmic}
\end{algorithm}

\clearpage

\subsection{Extended Empirical Results} \label{sec:extended_empirical_appendix}

\begin{table}[h!]
    \centering
    \small
    \begin{tabularx}{\linewidth}{@{}
        >{\hsize=1.0\hsize}X  
        >{\hsize=0.3\hsize}X  
        >{\hsize=0.3\hsize}X  
        >{\hsize=0.4\hsize}X  
        >{\hsize=0.4\hsize}X  
        >{\hsize=0.4\hsize}X  
        >{\hsize=0.4\hsize}X  
        >{\hsize=0.4\hsize}X  
        >{\hsize=1.15\hsize}X  
        >{\hsize=0.8\hsize}X@{}  
    }
        \toprule
        \textbf{Algorithm} & $N$ & $|\mathcal{E}^*|$ & $\mathbb{P}^*_{\text{maj}}$ & $\minPEps$ & $R_T$ & $\%R \downarrow$ & $\hat{P}_{\text{maj}}$ & \textbf{Domain} & \textbf{Config. Id.} \\
        \midrule
        \tT{SEE} & 20 & 9 & 0.992 & 0.010 & \textbf{546.2} & 0.921 & \textbf{0.980} & Bernoulli & SE3 \\
        Comb. UCB & 20 & 9 & 0.992 & 0.010 & 6932 & - & 0.370 & Bernoulli & SC3 \\
        \tT{SEE} & 15 & 3 & 0.985 & 0.040 & \textbf{649.8} & 0.867 & \textbf{0.969} & Bernoulli & SE3 \\
        Comb. UCB & 15 & 3 & 0.985 & 0.040 & 4906 & - & 0.549 & Bernoulli & SC2 \\
        \tT{SEE} & 5 & 3 & 0.883 & 0.010 & \textbf{535.3} & 0.719 & \textbf{0.871} & Bernoulli & SE1 \\
        Comb. UCB & 5 & 3 & 0.883 & 0.010 & 1905 & - & 0.759 & Bernoulli & SC1 \\
        \midrule
        \addlinespace[0.2cm]

        $\theta$-WMV & 9 & 9 & 0.644 & 0.004 & \textbf{217.4} & 0.545 & \textbf{0.606} & Bernoulli & WV3 \\
        Zooming & 9 & 9 & 0.644 & 0.004 & 478.4 & - & 0.449 & Bernoulli & WZ3 \\
        $\theta$-WMV & 6 & 6 & 0.911 & 0.004 & \textbf{312.4} & 0.331 & \textbf{0.853} & Bernoulli & WV2 \\
        Zooming & 6 & 6 & 0.911 & 0.004 & 467.0 & - & 0.765 & Bernoulli & WZ2 \\
        $\theta$-WMV & 3 & 3 & 0.881 & 0.106 & \textbf{287.0} & 0.474 & \textbf{0.835} & Bernoulli & WV1 \\
        Zooming & 3 & 3 & 0.881 & 0.106 & 546 & - & 0.681 & Bernoulli & WZ1 \\
        \midrule
        \addlinespace[0.2cm]

        $\theta$-WMV & 9 & 9 & 0.897 & 0.023 & \textbf{200.1} & 0.080 & \textbf{0.828} & Bernoulli & WS4 \\
        \tT{SEE} & 9 & 9 & 0.897 & 0.023 & 216.0 & - & 0.806 & Bernoulli & WE5 \\
        $\theta$-WMV & 5 & 5 & 0.811 & 0.05 & \textbf{203.8} & 0.114 & \textbf{0.732} & Bernoulli & WS5 \\
        \tT{SEE} & 5 & 5 & 0.811 & 0.05 & 230.1 & - & 0.707 & Bernoulli & WE5 \\
        $\theta$-WMV & 3 & 3 & 0.89 & 0.02 & \textbf{202.6} & 0.484 & \textbf{0.817} & Bernoulli & WS6 \\
        \tT{SEE} & 3 & 3 & 0.68 & 0.05 & 392.7 & - & 0.427 & Bernoulli & WE6 \\
        \midrule
        \addlinespace[0.2cm]
        
        $\theta$-WMV & 9 & 9 & 0.927 & 0.027 & \textbf{587.0} & 0.829 & \textbf{0.927} & GSM8K & WG3 \\
        Zooming & 9 & 9 & 0.927 & 0.027 & 3443 & - & 0.604 & GSM8K & ZG3 \\
        $\theta$-WMV & 6 & 6 & 0.927 & 0.132 & \textbf{420.6} & 0.353 & \textbf{0.918} & GSM8K & WG2 \\
        Zooming & 6 & 6 & 0.927 & 0.132 & 650.2 & - & 0.891 & GSM8K & ZG2 \\
        $\theta$-WMV & 3 & 3 & 0.439 & 0.057 & \textbf{302.0} & 0.660 & \textbf{0.433} & GSM8K & WG1 \\
        Zooming & 3 & 3 & 0.439 & 0.057 & 886.0 & - & 0.369 & GSM8K & ZG1 \\
        \midrule
        \addlinespace[0.1cm]
        
        \tT{SEE} & 9 & 4 & 0.807 & 0.036 & \textbf{328.2} & 0.939 & \textbf{0.805} & CommonsenseQA & CS3 \\
        Comb. UCB & 9 & 4 & 0.807 & 0.036 & 5437 & - & 0.293 & CommonsenseQA & CC3 \\
        \tT{SEE} & 5 & 3 & 0.810 & 0.031 & \textbf{676.8} & 0.364 & \textbf{0.810} & CommonsenseQA & CS2 \\
        Comb. UCB & 5 & 3 & 0.810 & 0.031 & 1065 & - & 0.664 & CommonsenseQA & CC2 \\
        \tT{SEE} & 3 & 1 & 0.810 & 0.051 & \textbf{339.0} & 0.65 & \textbf{0.807} & CommonsenseQA & CS1 \\
        Comb. UCB & 3 & 1 & 0.810 & 0.051 & 970 & - & 0.763 & CommonsenseQA & CS1 \\
        \midrule
        \addlinespace[0.1cm]
        
        $\theta$-WMV  & 9 & 9 & 0.763 & 0.0085 & \textbf{210.1} & 0.402 & \textbf{0.675} & BoolQ & WB3 \\
        Zooming & 9 & 9 & 0.763 & 0.0085 & 351.4 & - & 0.483 & BoolQ & ZB3 \\
        $\theta$-WMV  & 5 & 5 & 0.735 & 0.013 & \textbf{360.9} & 0.512 & \textbf{0.735} & BoolQ & WB2 \\
        Zooming & 5 & 5 & 0.735 & 0.013 & 740.3 & - & 0.694 & BoolQ & ZB2 \\
        $\theta$-WMV  & 4 & 4 & 0.735 & 0.04 & \textbf{183.1} & 0.225 & \textbf{0.667} & BoolQ & WB1 \\
        Zooming & 4 & 4 & 0.735 & 0.04 & 236.4 & - & 0.586 & BoolQ & ZB1 \\
        \bottomrule
    \end{tabularx}
    \vspace{0.3em}
    \caption[Bandit experiment results]{Multi-armed bandit experiment results showing consistent performance across environments, with time horizon $T\leq10^4$. Performance metrics show reduced cumulative regret ($R_T$) and increased empirical accuracy ($\hat{P}_{\text{maj}}$) across all baselines.}
    \label{tab:transposed_results}
    \vspace{-2.6em}
\end{table}

\clearpage

\subsubsection{Empirical Results: Successive Expert Elimination with Bernoulli Experts}

\begin{figure}[H]
\minipage{0.47\textwidth}
  \includegraphics[width=\linewidth]{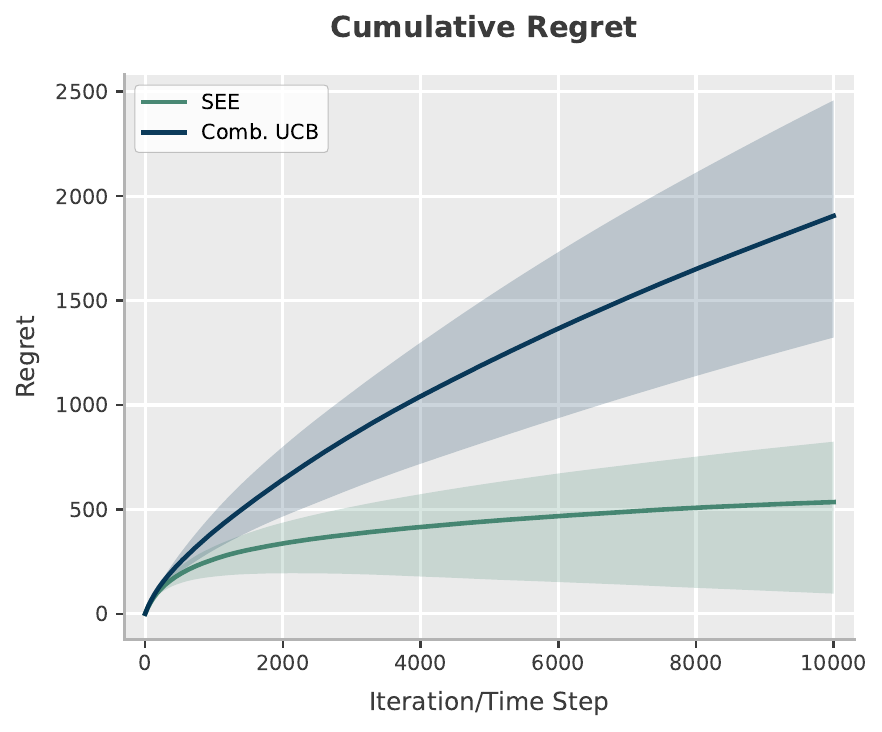}
\endminipage\hfill
\minipage{0.47\textwidth}
  \includegraphics[width=\linewidth]{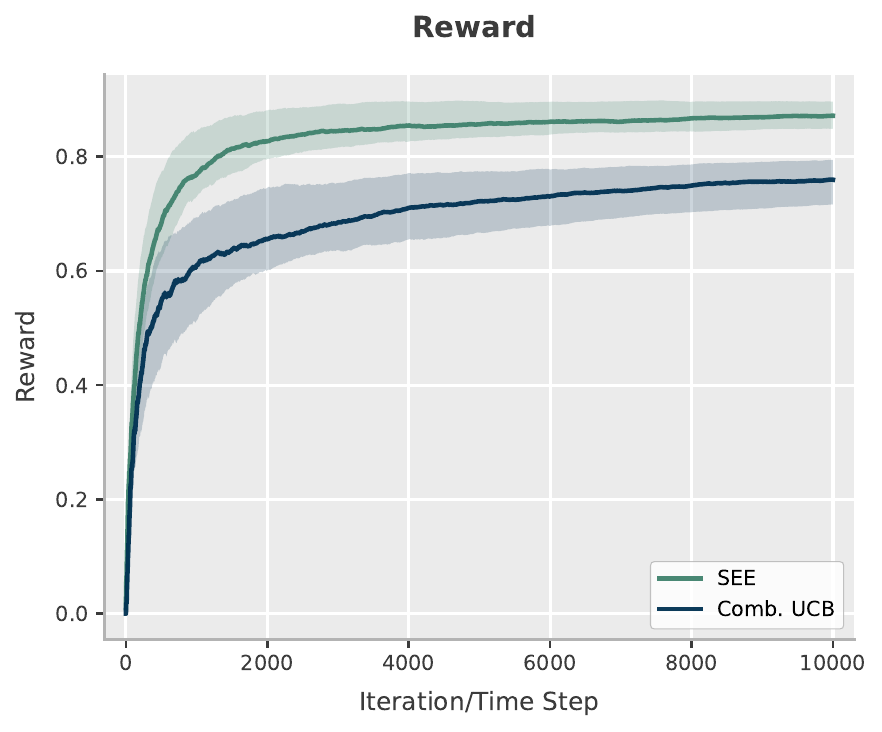}
\endminipage\hfill
\caption*{Config. SE1 \& SC1 - $p_i$: $[0.1, 0.65, 0.77, 0.79, 0.8]$.}
\minipage{0.47\textwidth}
  \includegraphics[width=\linewidth]{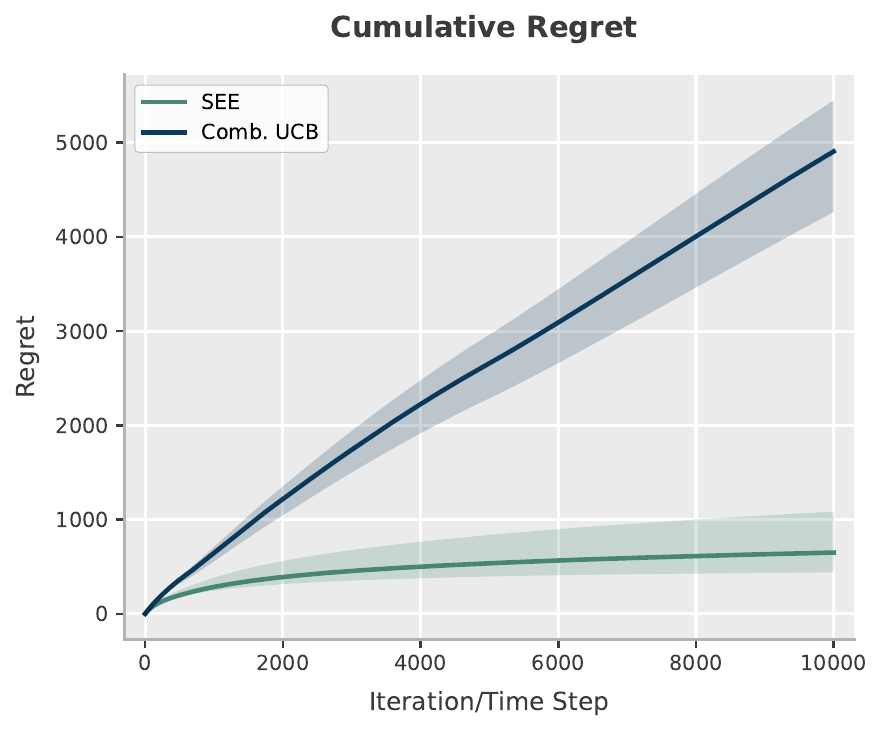}
\endminipage\hfill
\minipage{0.47\textwidth}
  \includegraphics[width=\linewidth]{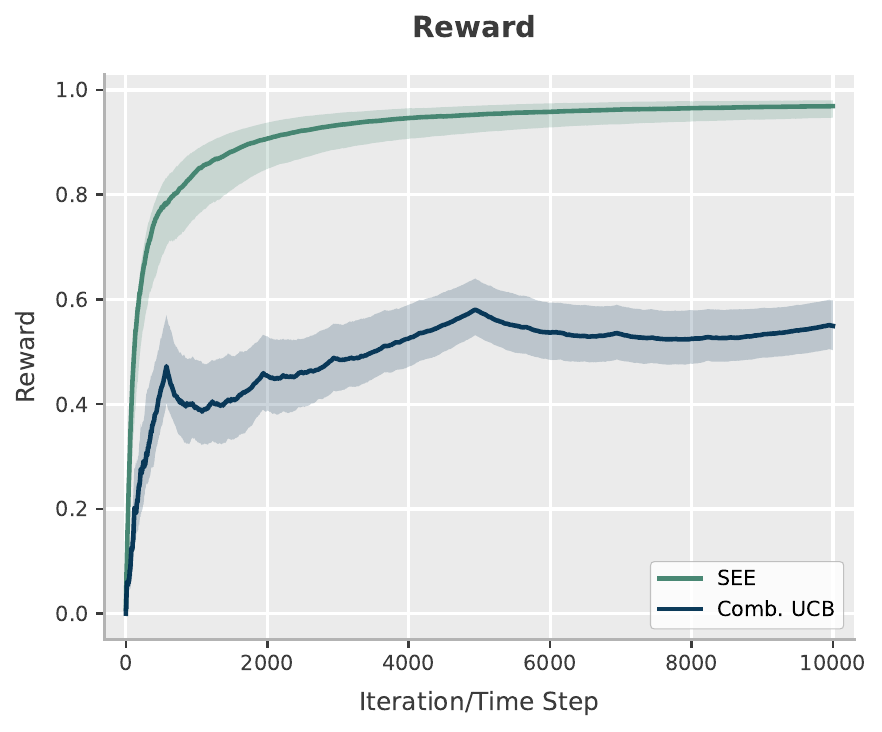}
\endminipage\hfill
\caption*{Config. SE2 \& SC2 - $p_i$: $[0.12, 0.23, 0.31, 0.37, 0.42, 0.49, 0.55, 0.63, 0.68, 0.72, 0.76, 0.81, 0.87, 0.92, 0.98]$.}
\minipage{0.47\textwidth}
  \includegraphics[width=\linewidth]{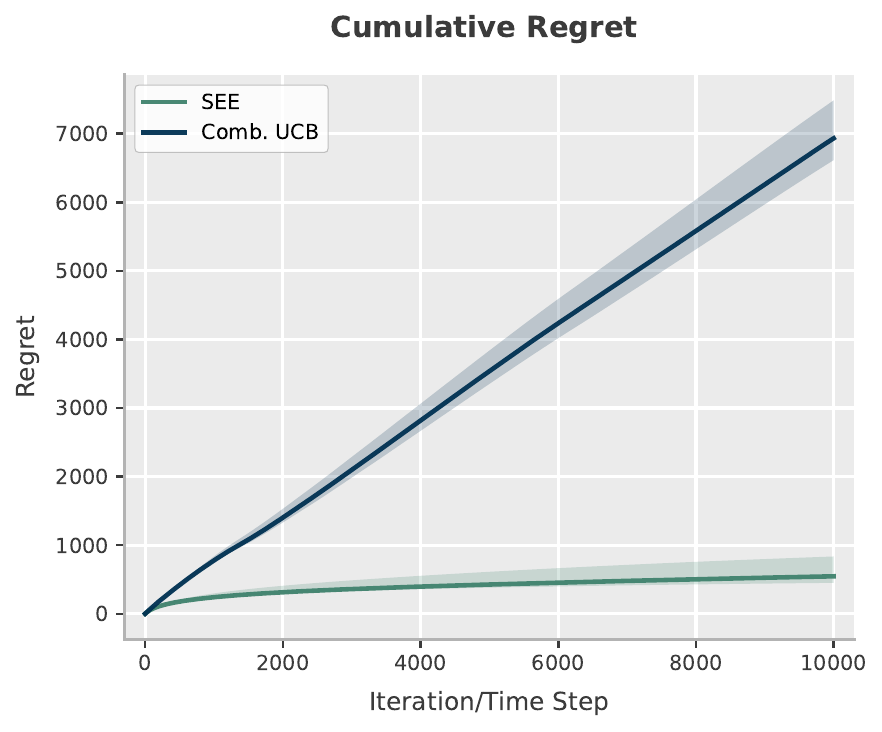}
\endminipage\hfill
\minipage{0.47\textwidth}
  \includegraphics[width=\linewidth]{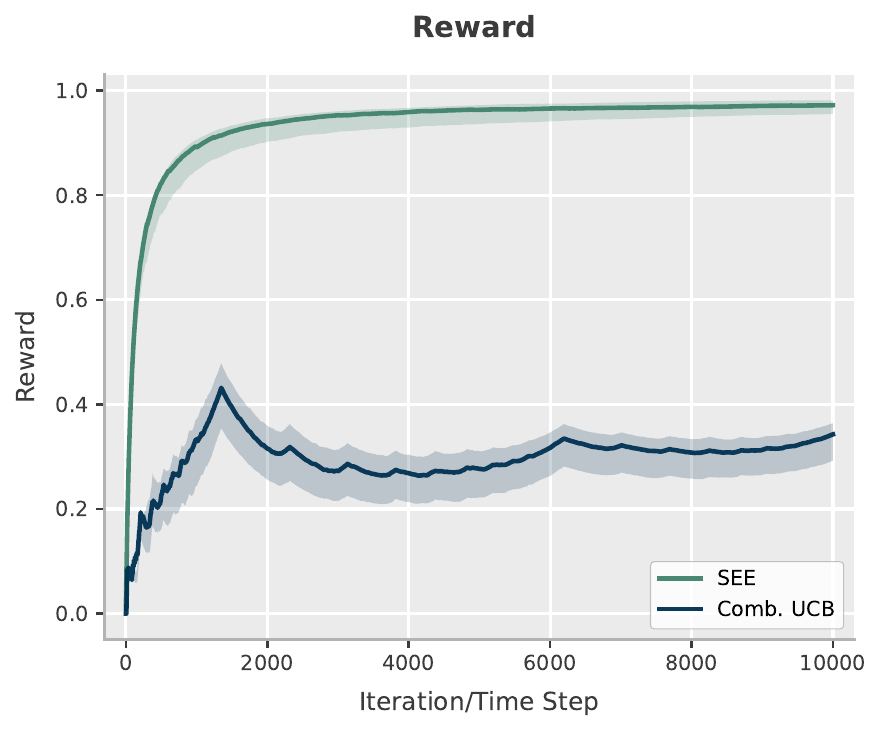}
\endminipage\hfill
\caption*{Config. SE3 \& SC3 - $p_i$: $[
    0.04, 0.07, 0.09, 0.12, 0.15, 0.19, 0.22, 0.26, 0.30, 0.33,
    0.70, 0.73, 0.76, 0.78, 0.81, \\ 0.84, 0.86, 0.89, 0.90, 0.91
]$.}
\caption{Mean values are calculated over 1,000 trials, with finite time horizon $T=10,000$, with shaded regions representing confidence intervals of $\pm \ucB$.} \label{fig:see_plots}
\end{figure}

\clearpage

\subsubsection{Empirical Results: $\theta$-Weighted Majority Voting with Bernoulli Experts}

\begin{figure}[H]
\minipage{0.47\textwidth}
  \includegraphics[width=\linewidth]{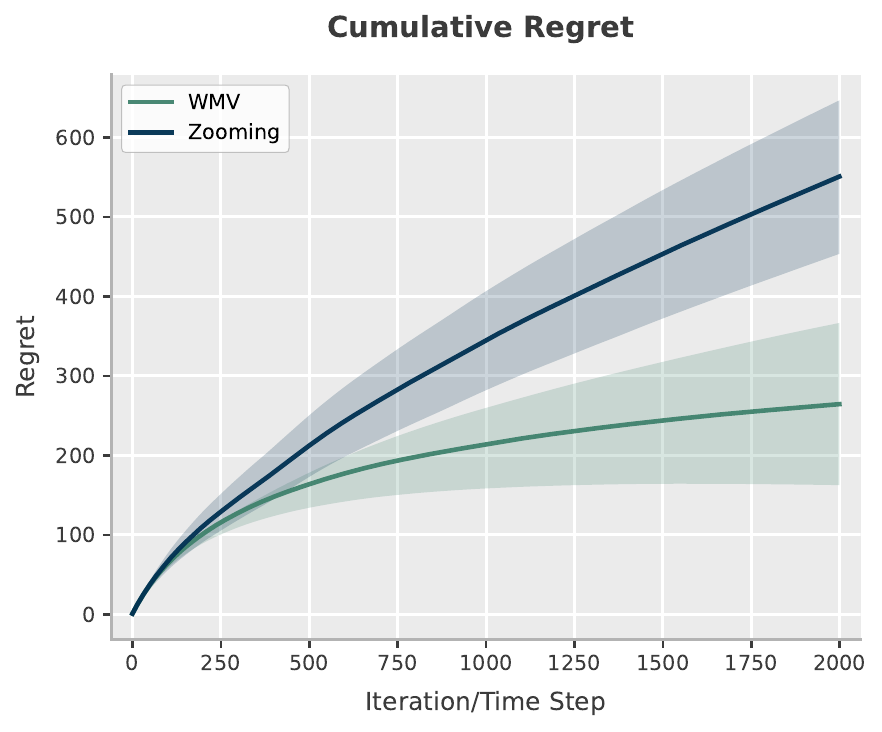}
\endminipage\hfill
\minipage{0.47\textwidth}
  \includegraphics[width=\linewidth]{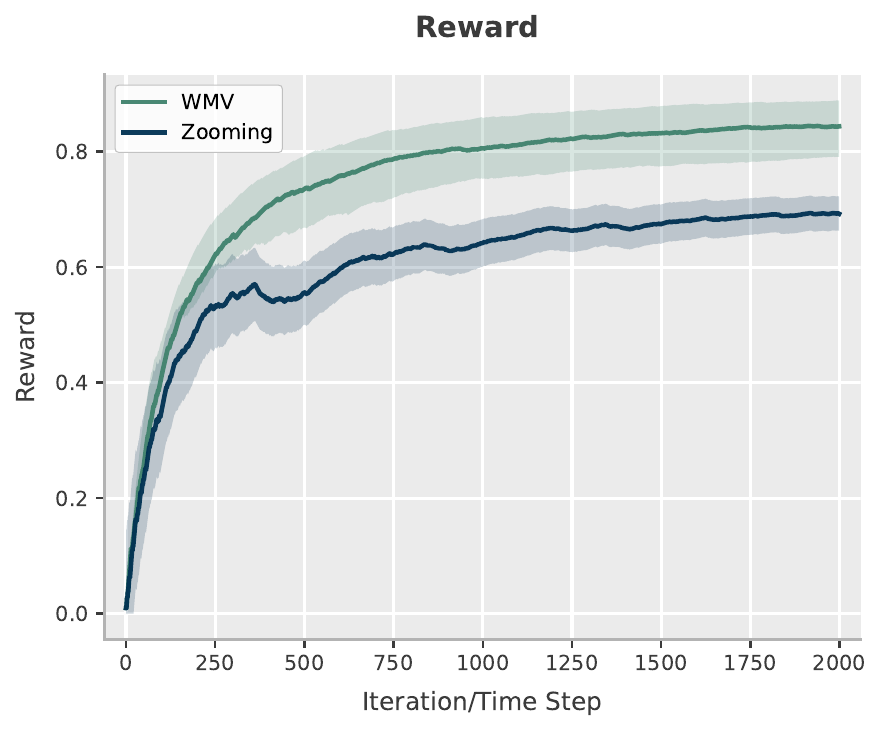}
\endminipage\hfill
\caption*{Config. WV1 \& WZ1 - $p_i$: $[0.332, 0.775, 0.881]$.}
\minipage{0.47\textwidth}
  \includegraphics[width=\linewidth]{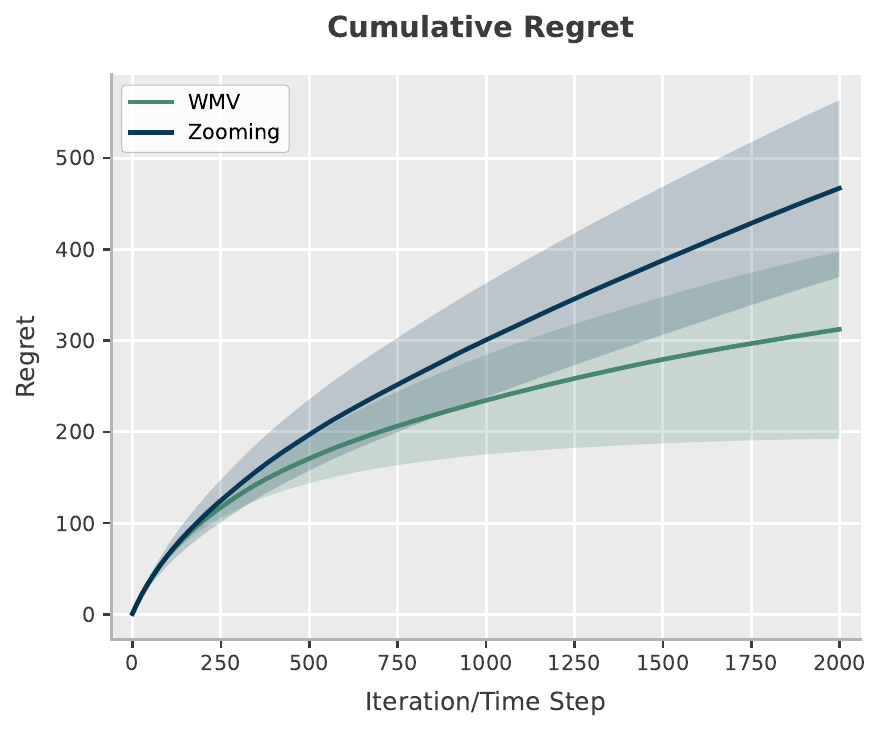}
\endminipage\hfill
\minipage{0.47\textwidth}
  \includegraphics[width=\linewidth]{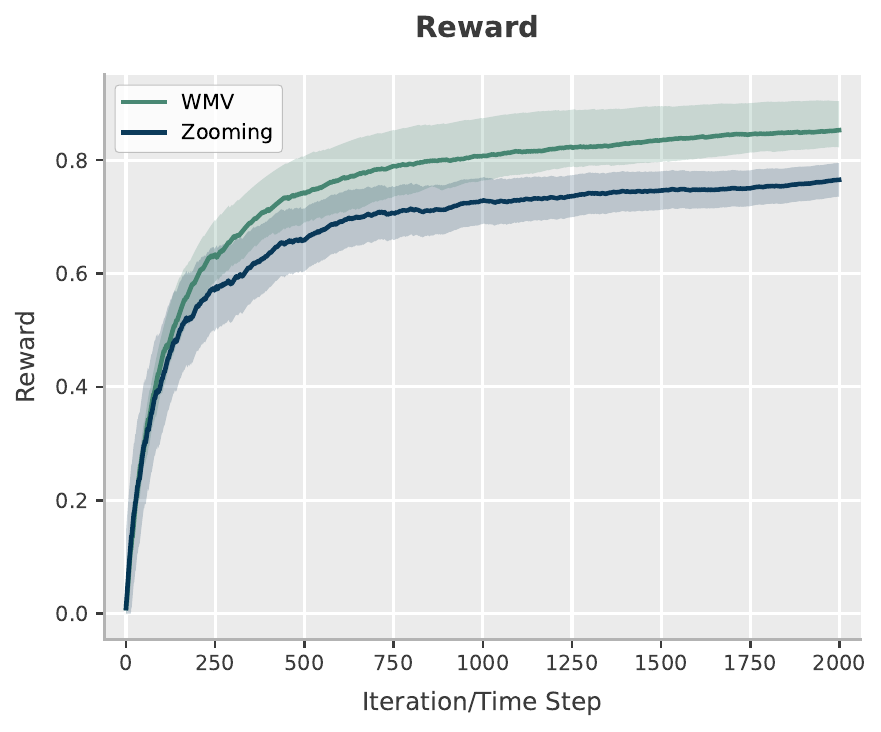}
\endminipage\hfill
\caption*{Config. WV2 \& WZ2 - $p_i$: $[0.388, 0.561, 0.782, 0.799, 0.803, 0.841]$.}
\minipage{0.47\textwidth}
  \includegraphics[width=\linewidth]{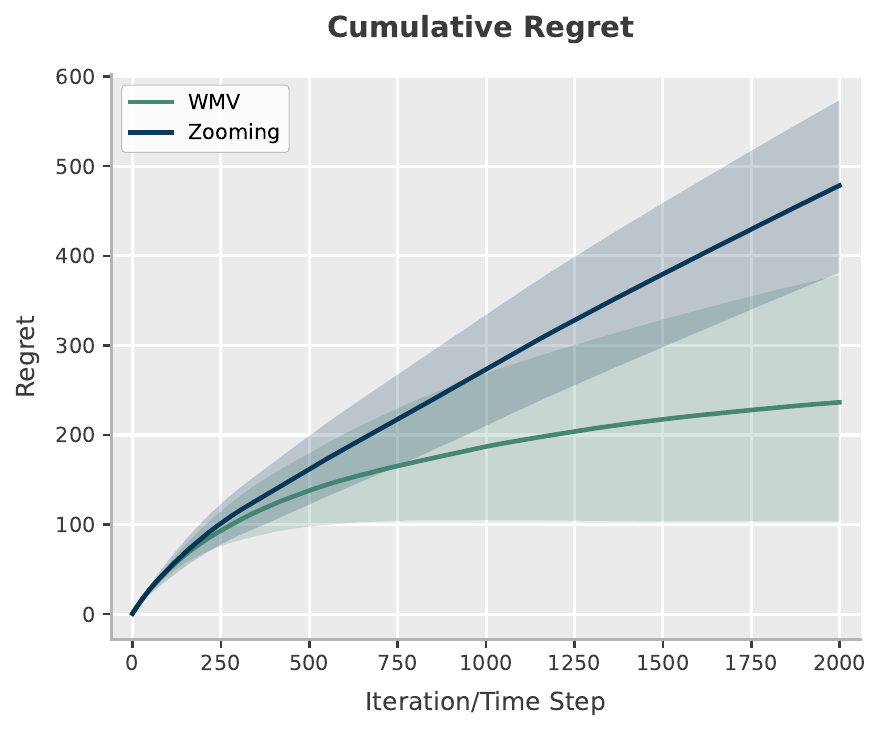}
\endminipage\hfill
\minipage{0.47\textwidth}
  \includegraphics[width=\linewidth]{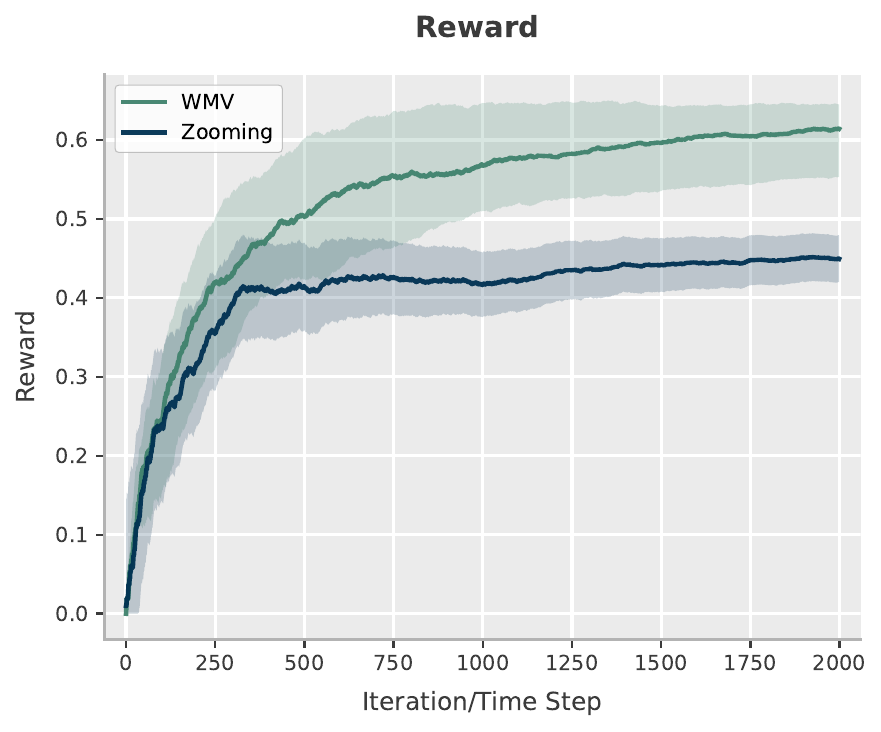}
\endminipage\hfill
\caption*{Config. WV3 \& WZ3 - $p_i$: $[0.261, 0.370, 0.382, 0.499, 0.503, 0.511, 0.542, 0.616, 0.634]$.}
\caption{Mean values are calculated over 1,000 trials, with finite time horizon $T=2000$, with shaded regions representing confidence intervals of $\pm \ucB$.} \label{fig:wmv_plots} 
\end{figure}
\clearpage

\subsubsection{Empirical Results: $\theta$-WMW vs. SEE Voting with Bernoulli Experts}

\begin{figure}[H]
\minipage{0.47\textwidth}
  \includegraphics[width=\linewidth]{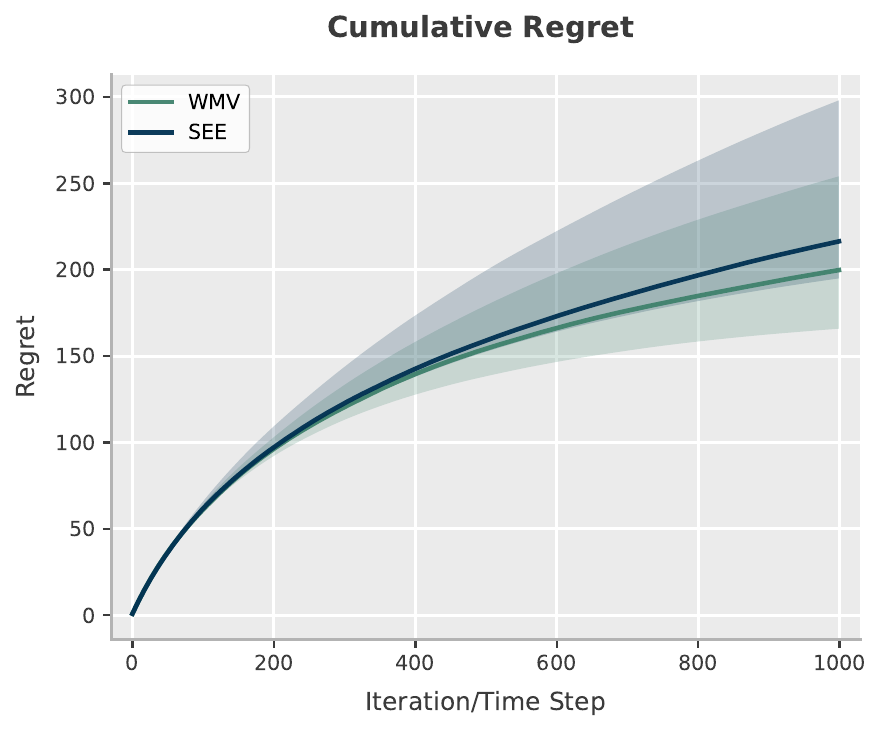}
\endminipage\hfill
\minipage{0.47\textwidth}
  \includegraphics[width=\linewidth]{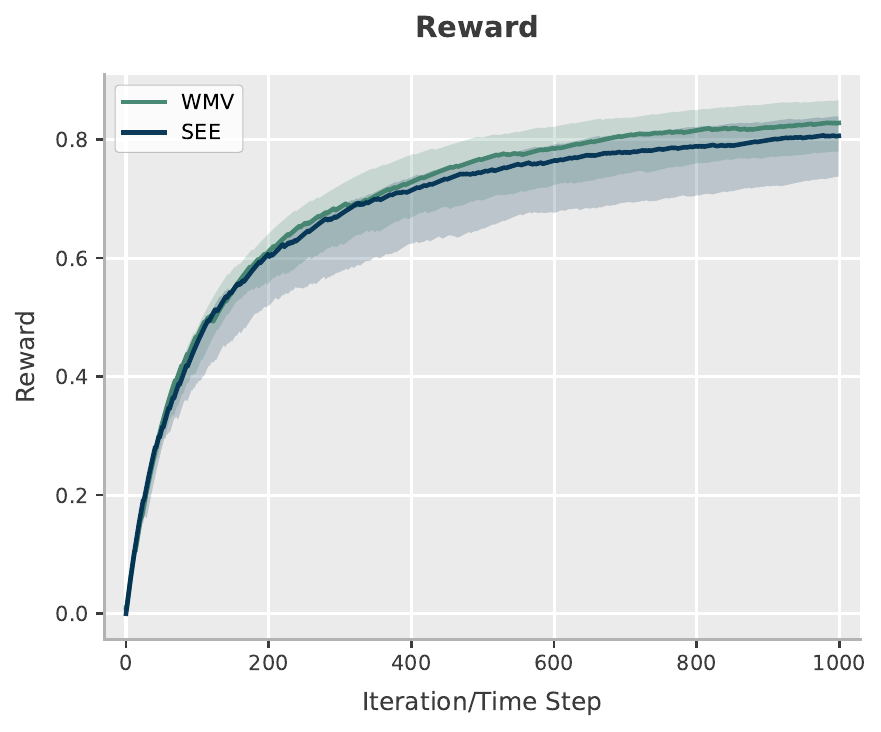}
\endminipage\hfill
\caption*{Config. WS4 \& WE4 - $p_i$: $[0.763, 0.786, 0.8492]$.}
\minipage{0.47\textwidth}
  \includegraphics[width=\linewidth]{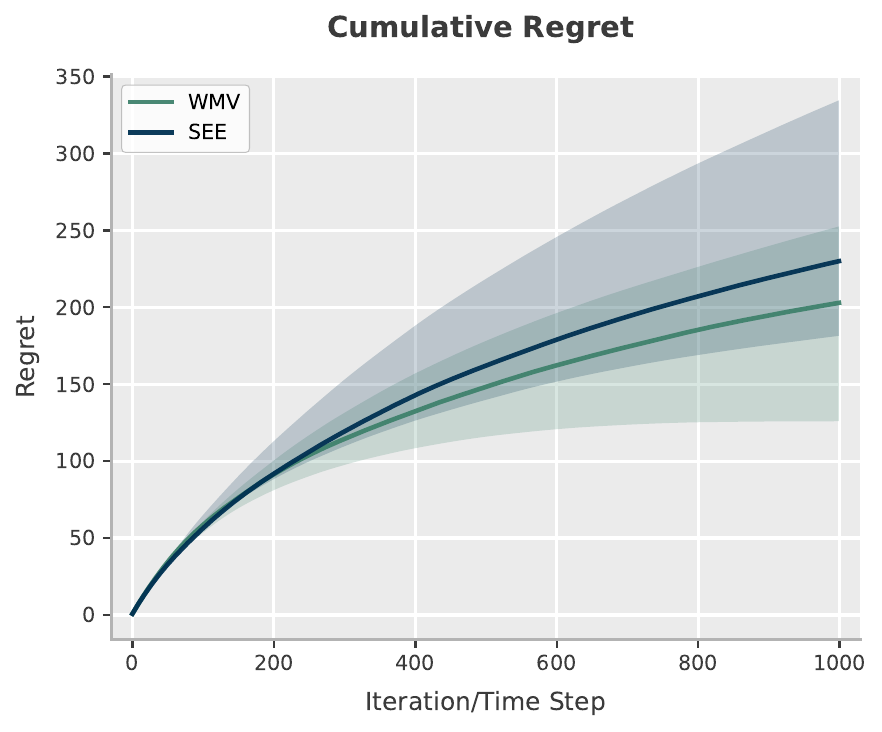}
\endminipage\hfill
\minipage{0.47\textwidth}
  \includegraphics[width=\linewidth]{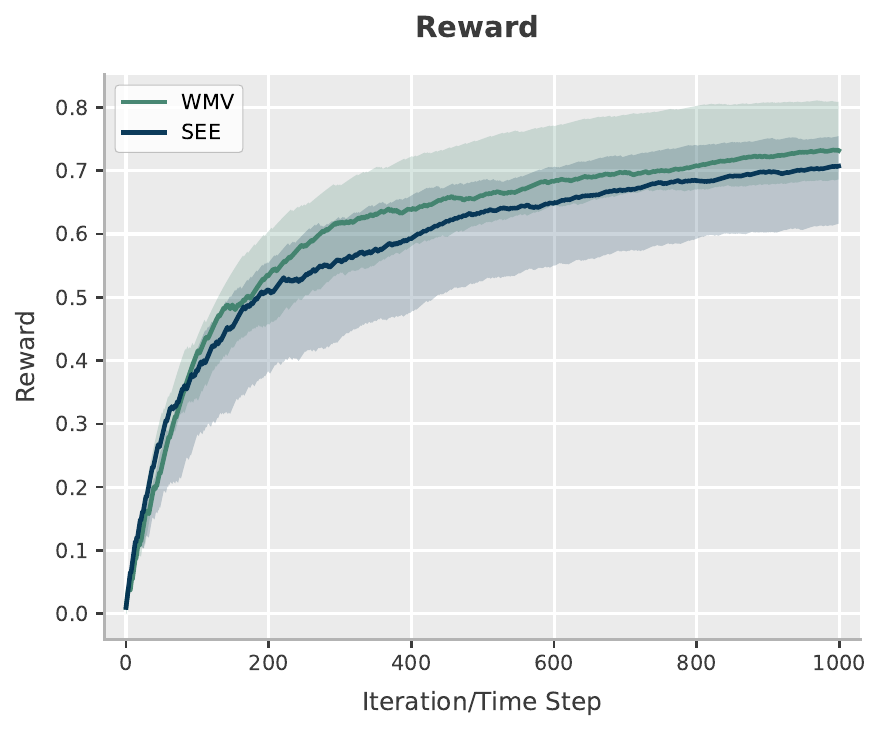}
\endminipage\hfill
\caption*{Config. WS5 \& WE5 - $p_i$: $[0.435, 0.501, 0.667, 0.714, 0.792]$.}
\minipage{0.47\textwidth}
  \includegraphics[width=\linewidth]{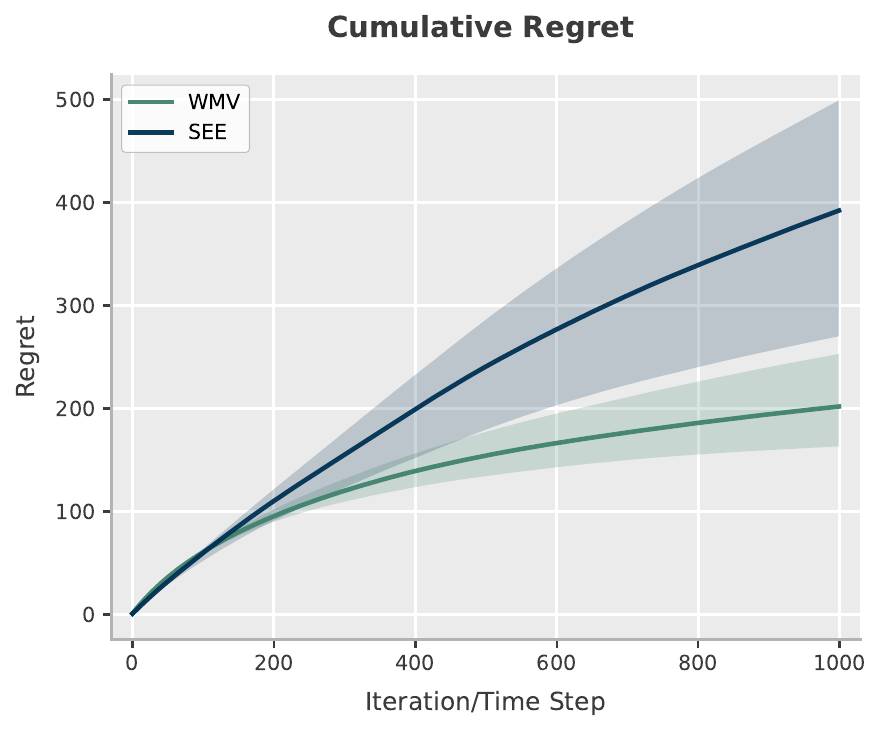}
\endminipage\hfill
\minipage{0.47\textwidth}
  \includegraphics[width=\linewidth]{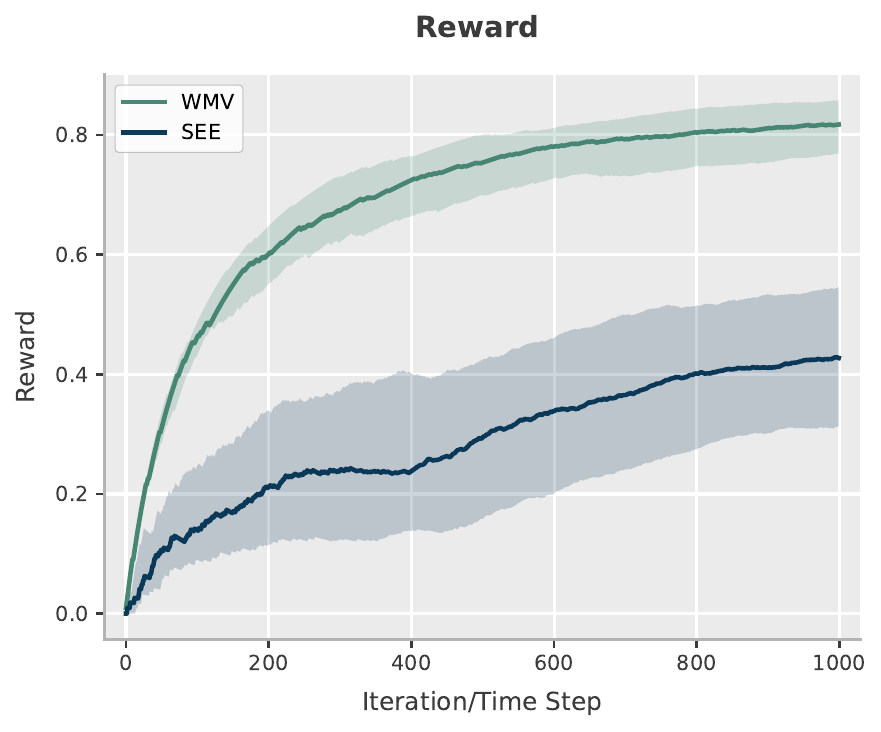}
\endminipage\hfill
\caption*{Config. WS6 \& WE6 - $p_i$: $[0.121, 0.232, 0.319, 0.374, 0.428, 0.498, 0.552, 0.637, 0.681]$.}
\caption{Mean values are calculated over 1,000 trials, with finite time horizon $T=1000$, with shaded regions representing confidence intervals of $\pm \ucB$. To note, for a small batch of experts, the regret performances are near-identical as, the optimal committee consists of a single best expert among a few. The learning complexity is not high, and both algorithms converge to the optimal solution quickly. As the number of experts increases, $\theta$-WMV demonstrates a clear advantage in regret minimization, aligning with our theory from Lemma \ref{lem:maj_vote_dominance}, that the weighted majority voting committee will always yield a superior solution compared to the egalitarian committee.
} \label{fig:wmv_plots} 
\end{figure}
\clearpage

\subsubsection{Empirical Results: $\theta$-Weighted Majority Voting with GSM8K Tasks}

\begin{figure}[H]
\minipage{0.47\textwidth}
  \includegraphics[width=\linewidth]{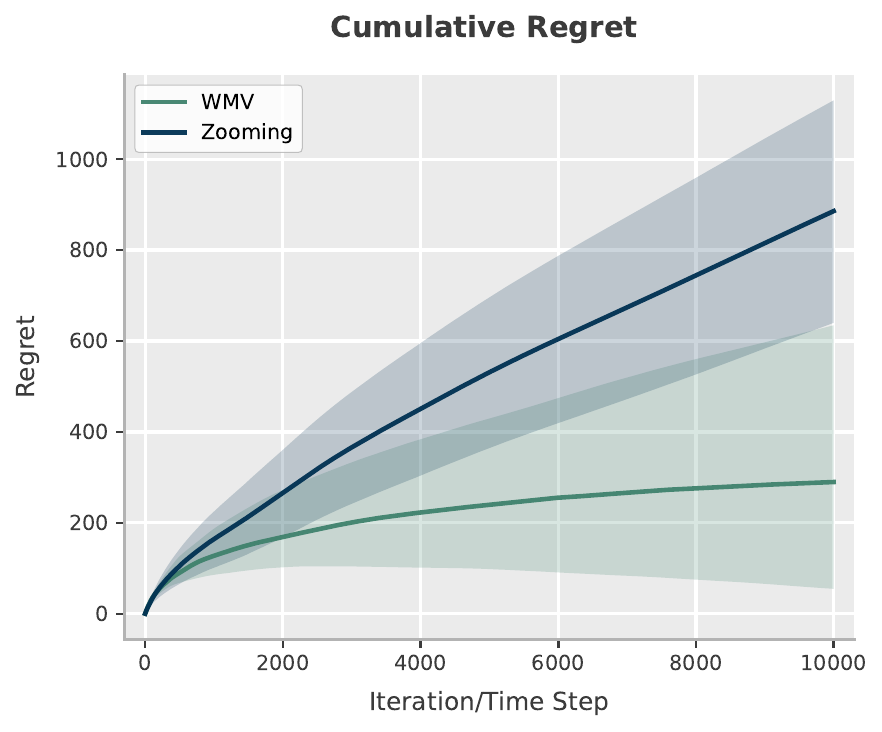}
\endminipage\hfill
\minipage{0.47\textwidth}
  \includegraphics[width=\linewidth]{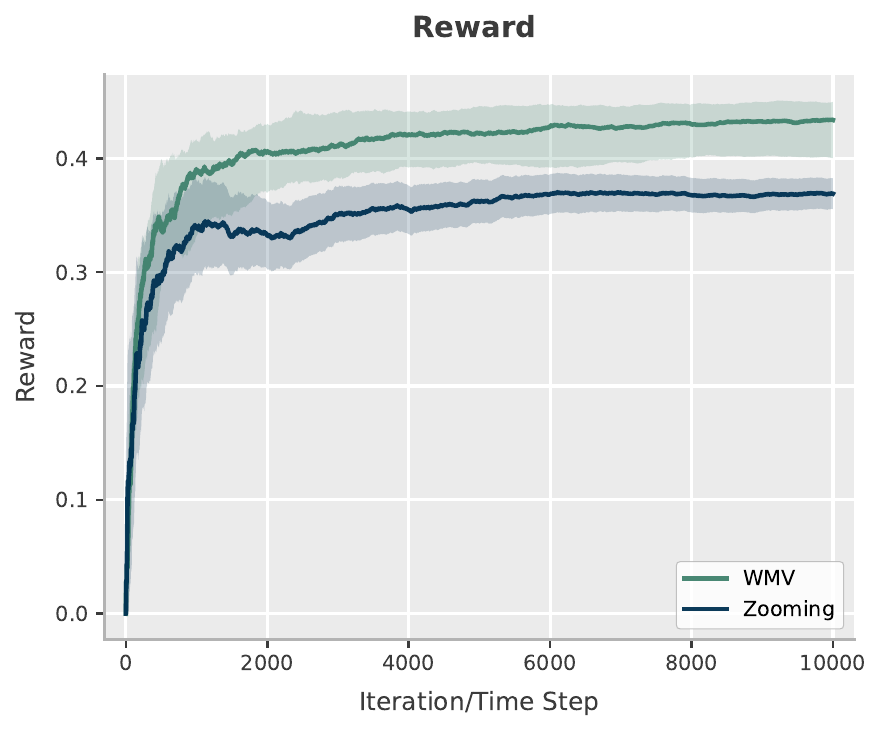}
\endminipage\hfill
\caption*{Config. WG1 \& ZG1 - LLM Expert Set: $\tT{[notus, qwen-14b, mistral]}$.}

\minipage{0.47\textwidth}
  \includegraphics[width=\linewidth]{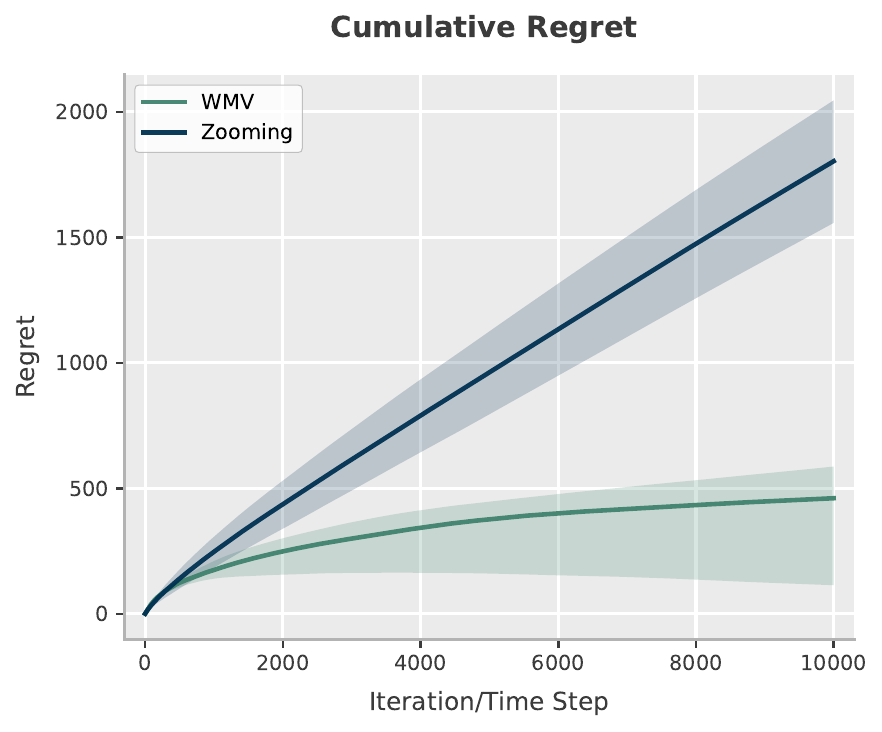}
\endminipage\hfill
\minipage{0.47\textwidth}
  \includegraphics[width=\linewidth]{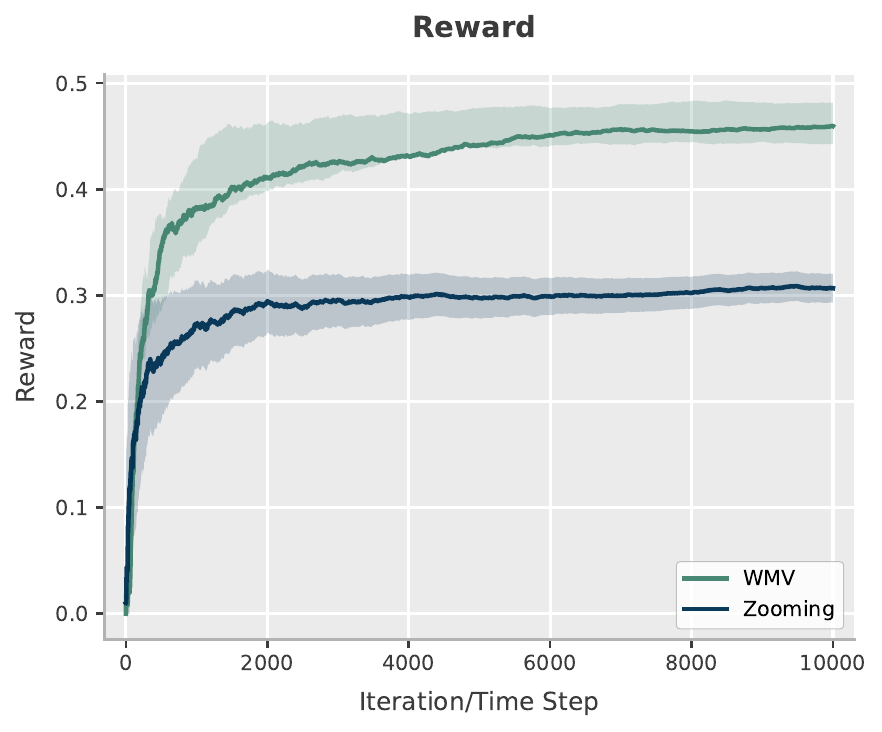}
\endminipage\hfill
\caption*{Config. WG2 \& ZG2 - LLM Expert Set: $\tT{[samantha-mistral, notus, qwen-14b, mistral, } \\ \tT{gemma-7b, deepseek-r1-14b]}$.}

\minipage{0.47\textwidth}
  \includegraphics[width=\linewidth]{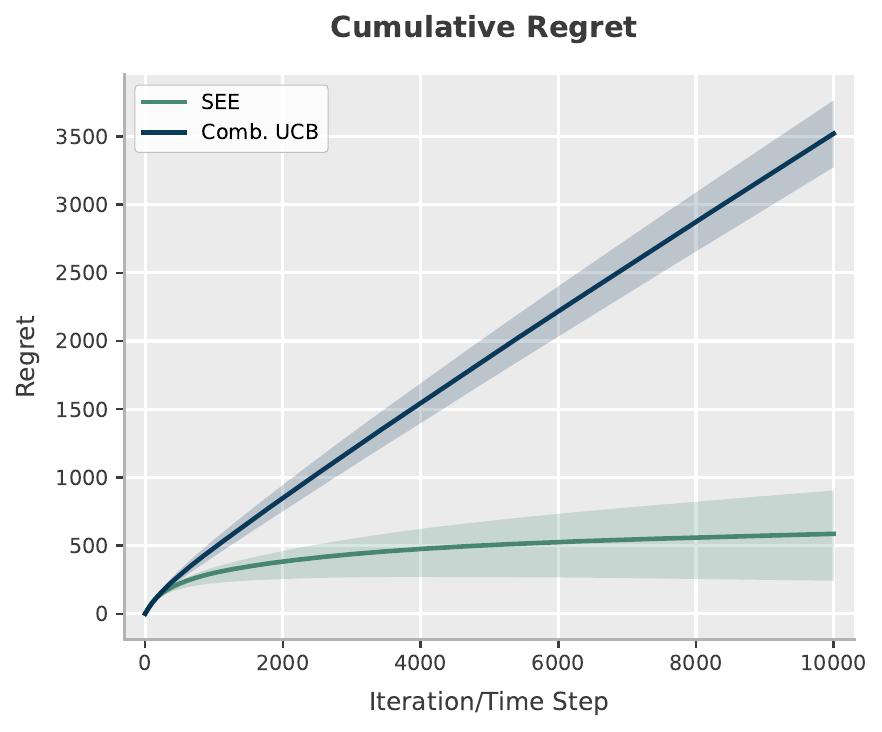}
\endminipage\hfill
\minipage{0.47\textwidth}
  \includegraphics[width=\linewidth]{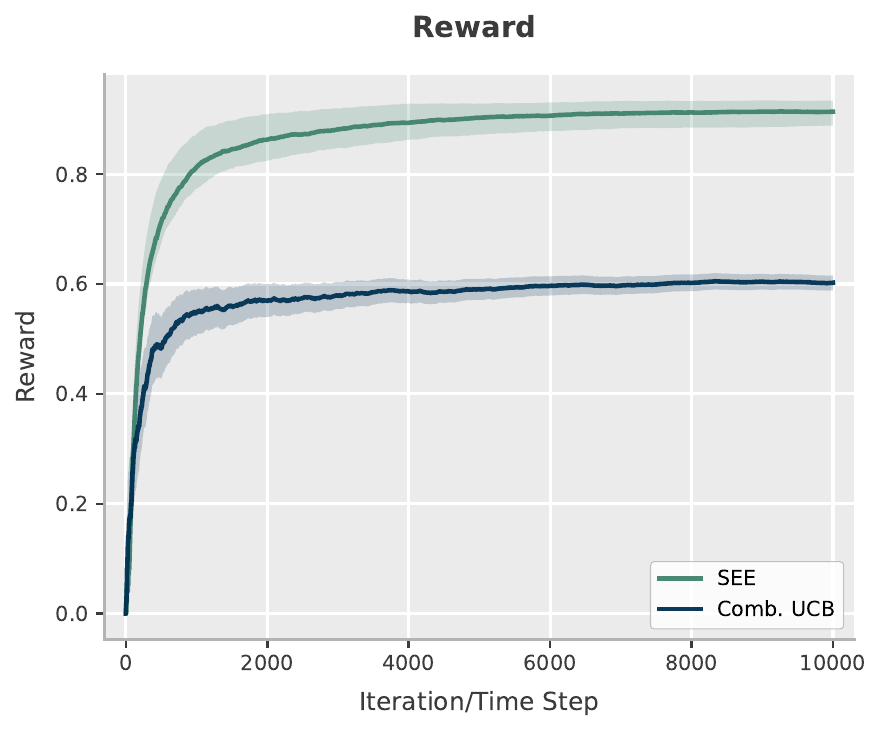}
\endminipage\hfill
\caption*{Config. WG3 \& ZG3 - LLM Expert Set: $\tT{[aya, mistral-openorca, samantha-mistral, notus, } \\ \tT{qwen-14b, mistral, gemma-7b, deepseek-r1-14b, phi4]}$.}

\caption{Questions were sampled from the GSM8K dataset \cite{cobbe2021gsm8k}. Mean values are calculated over 1,000 trials, with finite time horizon $T=10,000$, with shaded regions representing confidence intervals of $\pm \ucB$.} \label{fig:r1_game_summary_plots}
\end{figure}
\clearpage

\subsubsection{Empirical Results: Successive Expert Elimination Voting with CommonsenseQA Tasks}

\begin{figure}[H]
\minipage{0.47\textwidth}
  \includegraphics[width=\linewidth]{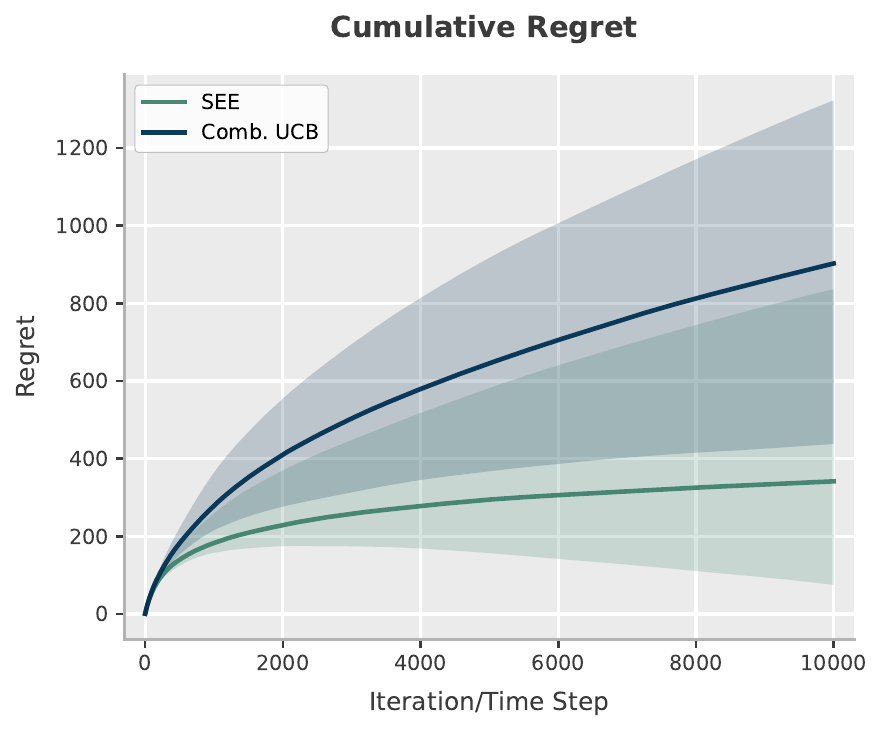}
\endminipage\hfill
\minipage{0.47\textwidth}
  \includegraphics[width=\linewidth]{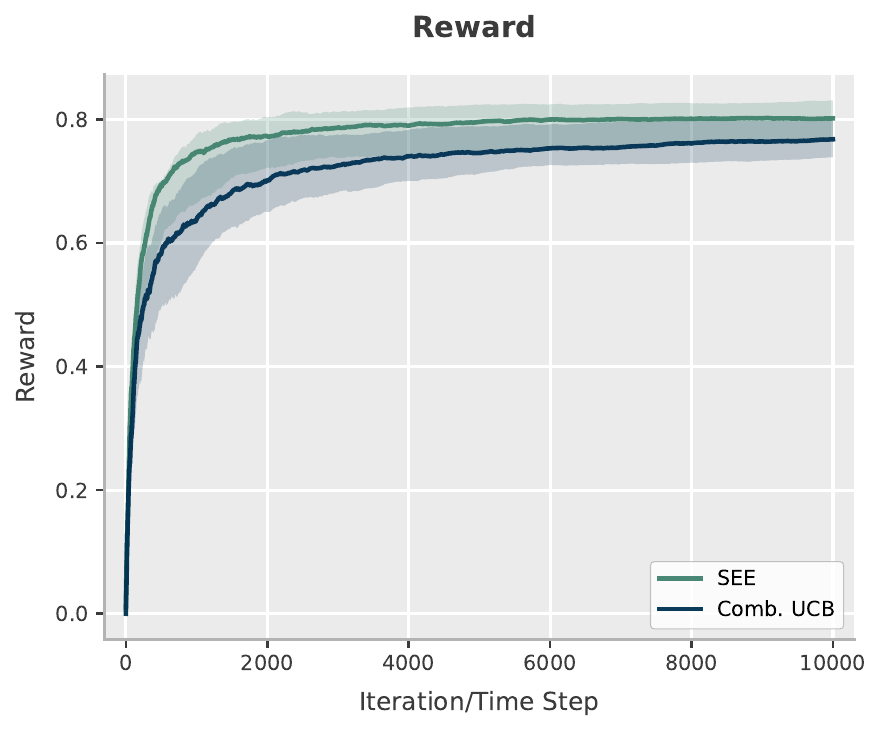}
\endminipage\hfill
\caption*{Config. CS1 \& CC1 - LLM Expert Set: $\tT{[gemma-7b, phi4, deepseek-r1-14b]}$.}
\minipage{0.47\textwidth}
  \includegraphics[width=\linewidth]{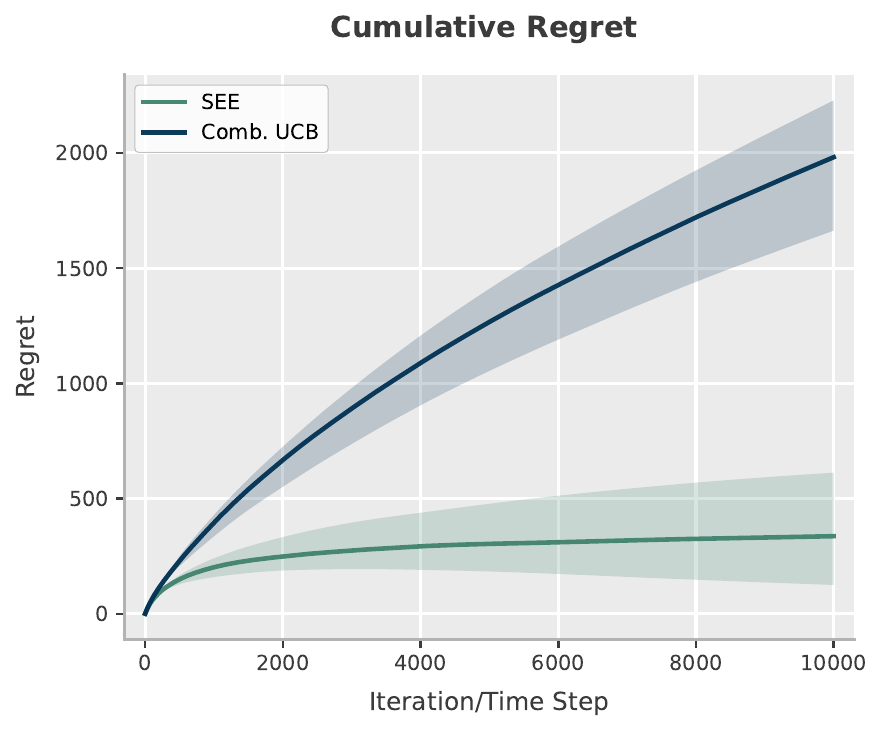}
\endminipage\hfill
\minipage{0.47\textwidth}
  \includegraphics[width=\linewidth]{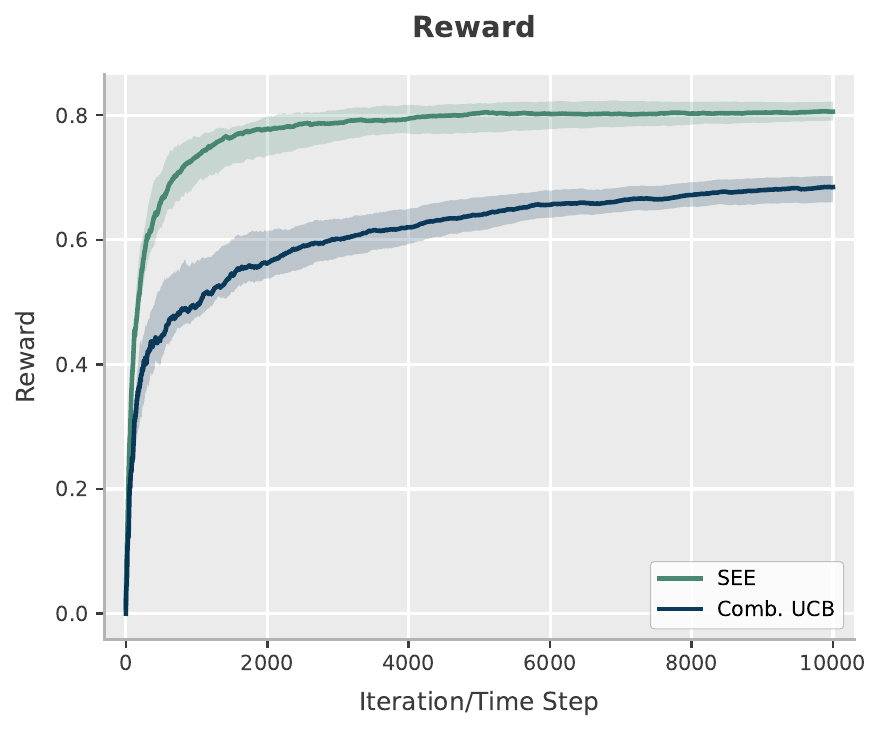}
\endminipage\hfill

\caption*{Config. CS2 \& CC2 - LLM Expert Set: $\tT{[mistral, qwen-14b, gemma-7b, phi4, deepseek-r1-14b]}$.}
\minipage{0.47\textwidth}
  \includegraphics[width=\linewidth]{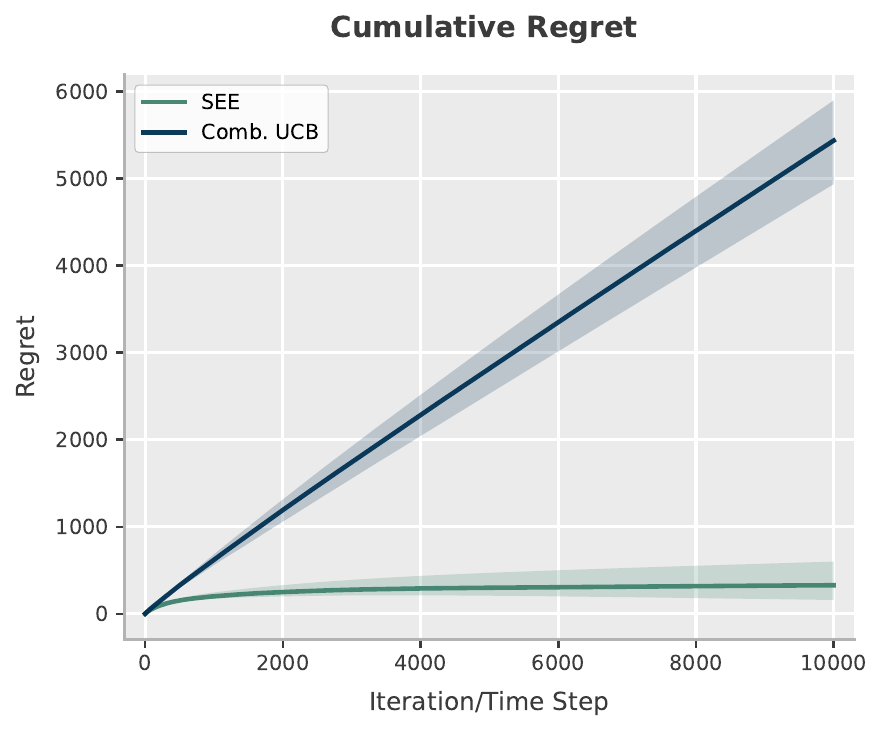}
\endminipage\hfill
\minipage{0.47\textwidth}
  \includegraphics[width=\linewidth]{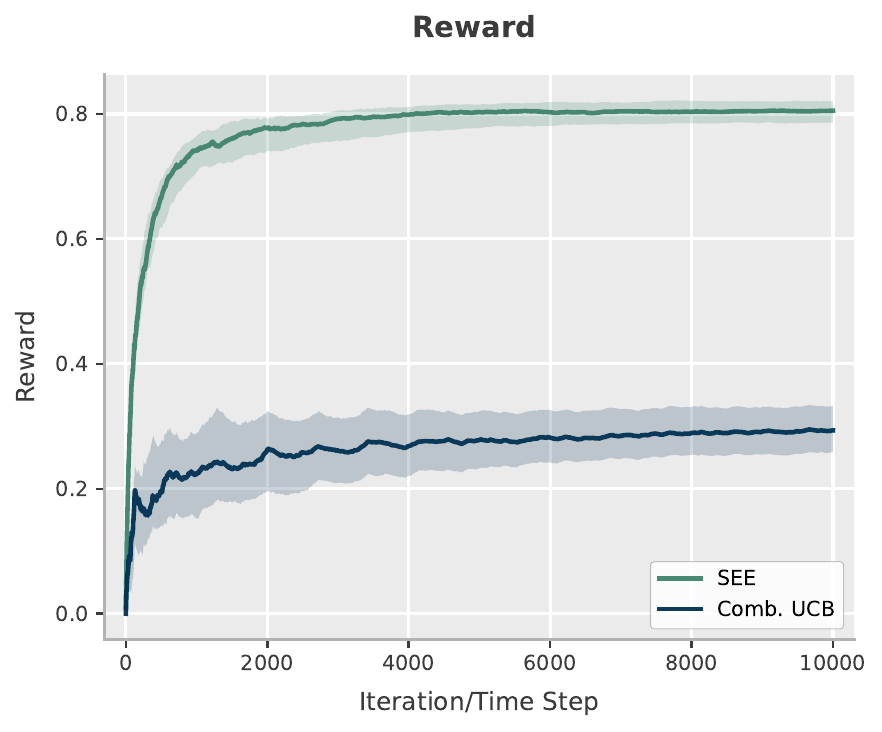}
\endminipage\hfill

\caption*{Config. CS3 \& CC3 - LLM Expert Set: $\tT{[aya, mistral-openorca, samantha-mistral, } \\ \tT{notus, mistral, qwen-14b, gemma-7b, phi4, deepseek-r1-14b]}$.}
\caption{Questions were sampled from the CommonsenseQA dataset \cite{talmor-etal-2019-commonsenseqa}. Mean values are calculated over 1,000 trials, with finite time horizon $T=10,000$, with shaded regions representing confidence intervals of $\pm \ucB$} \label{fig:r1_game_summary_plots}
\end{figure}
\clearpage

\subsubsection{Empirical Results: Successive Expert Elimination Voting with BoolQ Tasks}

\begin{figure}[H]
\minipage{0.47\textwidth}
  \includegraphics[width=\linewidth]{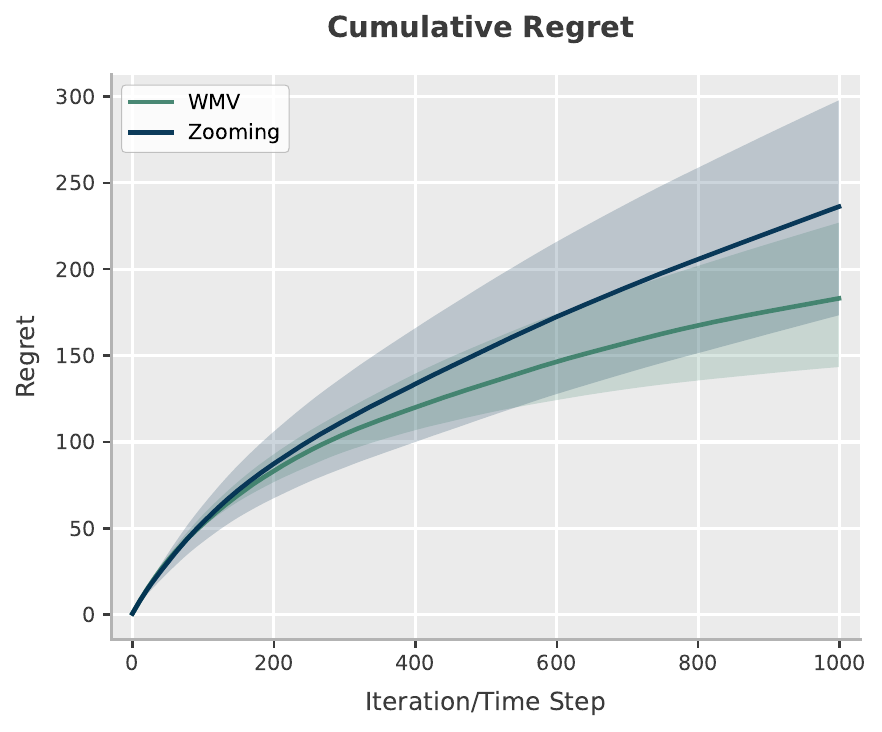}
\endminipage\hfill
\minipage{0.47\textwidth}
  \includegraphics[width=\linewidth]{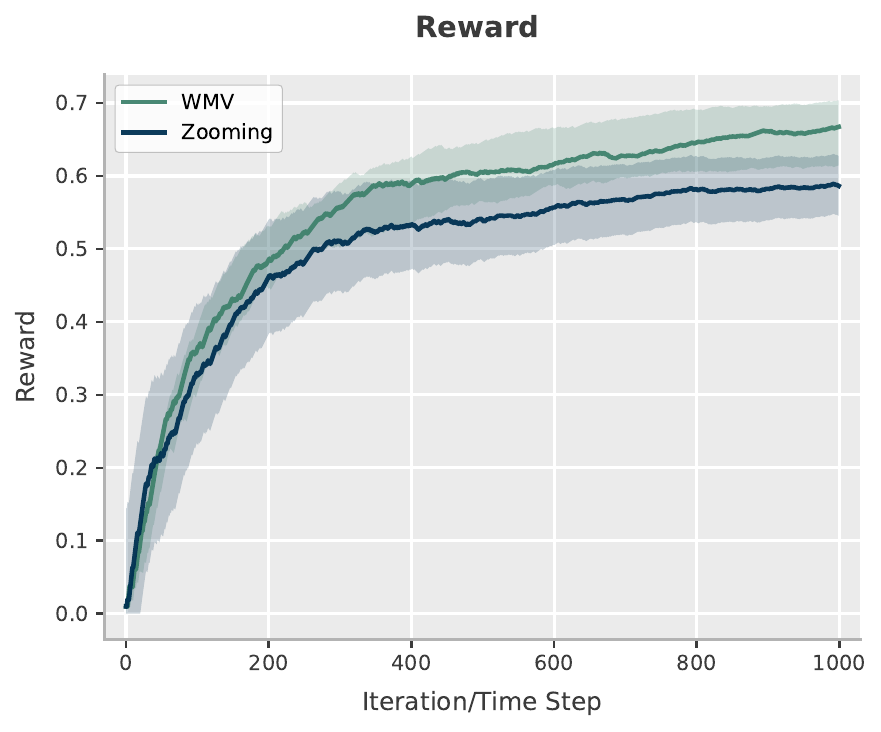}
\endminipage\hfill
\caption*{Config. WB1 \& ZB1 - LLM Expert Set: $\tT{[mistral, gemma-7b, deepseek-r1-14b, phi4]}$.}
\minipage{0.47\textwidth}
  \includegraphics[width=\linewidth]{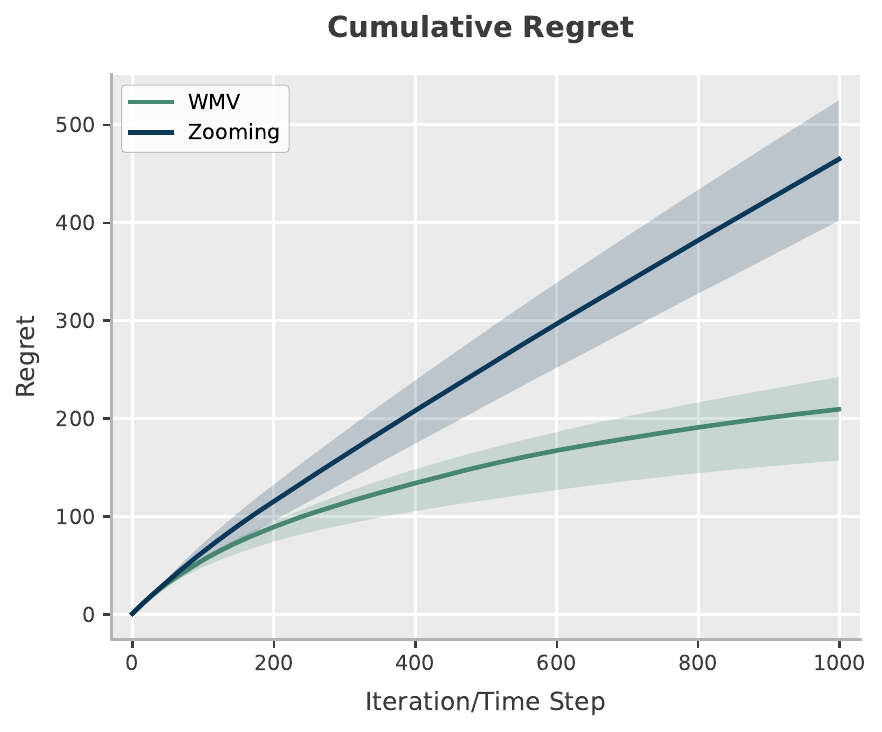}
\endminipage\hfill
\minipage{0.47\textwidth}
  \includegraphics[width=\linewidth]{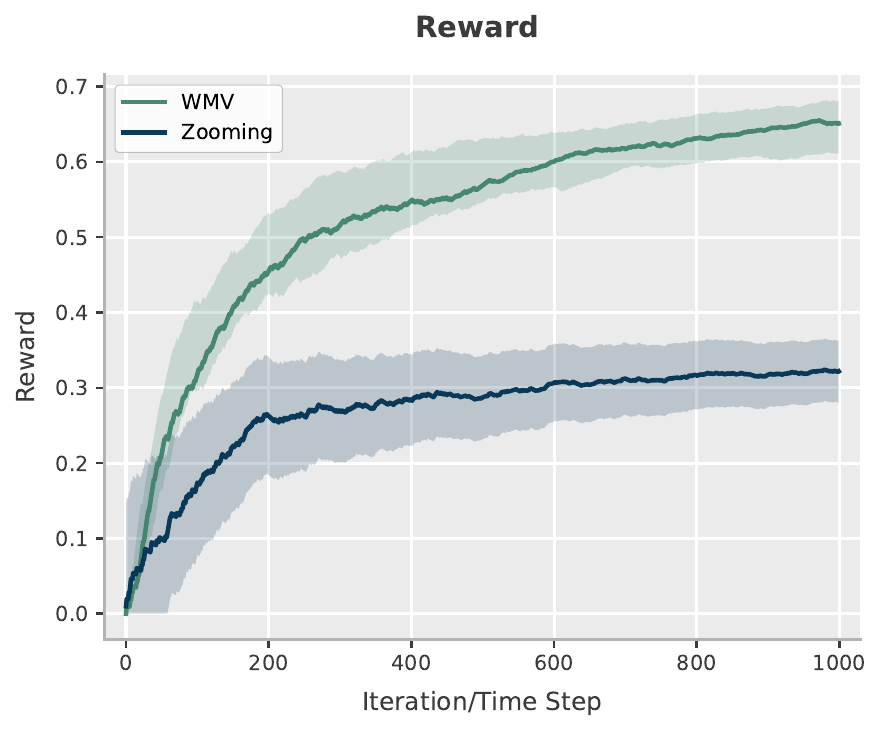}
\endminipage\hfill

\caption*{Config. WB2 \& ZB2 - LLM Expert Set: $\tT{[qwen-14b, mistral, gemma-7b, deepseek-r1-14b, phi4]}$.}
\minipage{0.47\textwidth}
  \includegraphics[width=\linewidth]{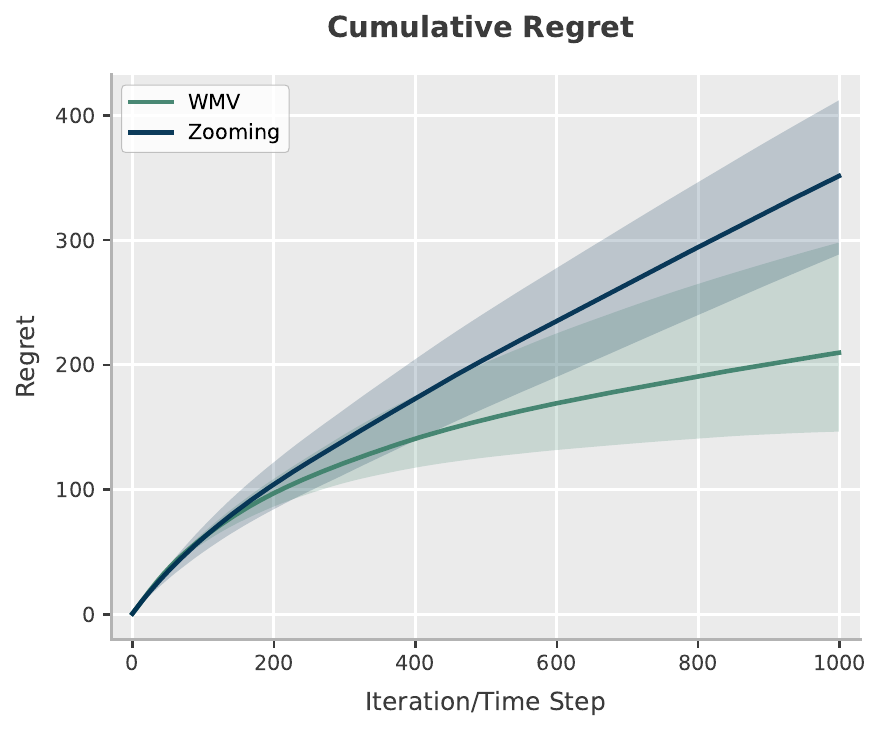}
\endminipage\hfill
\minipage{0.47\textwidth}
  \includegraphics[width=\linewidth]{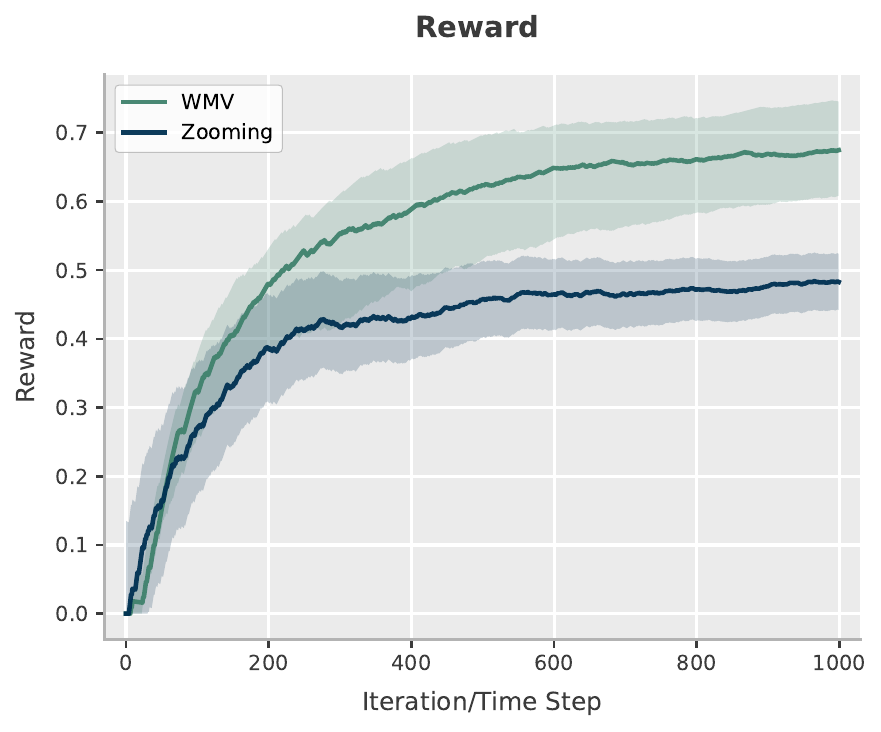}
\endminipage\hfill

\caption*{Config. WB3 \& ZB3 - LLM Expert Set: $\tT{[samantha-mistral, qwen-14b, mistral-openorca,  } \\
\tT{notus, aya, mistral,gemma-7b, deepseek-r1-14b, phi4]}$.}
\caption{Questions were sampled from the BoolQ dataset \cite{clark2019boolq}. Mean values are calculated over 1,000 trials, with finite time horizon $T=1,000$, with shaded regions representing confidence intervals of $\pm \ucB$.}  \label{fig:r1_game_summary_plots}
\end{figure}
\clearpage

\end{document}